\documentclass{article}

\usepackage{iclr2021_conference,times}

\usepackage{amsmath,amsfonts,bm}

\def\eqref#1{equation~\ref{#1}}
\def\1{\bm{1}}

\def\rb{{\textnormal{b}}}

\def\vb{{\bm{b}}}

\def\vv{{\bm{v}}}

\DeclareMathAlphabet{\mathsfit}{\encodingdefault}{\sfdefault}{m}{sl}
\SetMathAlphabet{\mathsfit}{bold}{\encodingdefault}{\sfdefault}{bx}{n}

\newcommand{\E}{\mathbb{E}}

\DeclareMathOperator*{\argmin}{arg\,min}

\let\ab\allowbreak

\usepackage[utf8]{inputenc} % allow utf-8 input
\usepackage[T1]{fontenc}    % use 8-bit T1 fonts
\usepackage{hyperref}       % hyperlinks
\usepackage{url}            % simple URL typesetting
\usepackage{booktabs}       % professional-quality tables
\usepackage{amsfonts}       % blackboard math symbols
\usepackage{nicefrac}       % compact symbols for 1/2, etc.
\usepackage{microtype}      % microtypography

\usepackage{helvet}  %Required
\usepackage{courier}  %Required
\usepackage{adjustbox}
\usepackage{xcolor}
\usepackage{amsmath,amssymb}
\iclrfinalcopy

\usepackage{graphicx} % more modern
\usepackage[hang]{subfigure} 
\graphicspath{{figures/}}
\usepackage{wrapfig}
\usepackage{algorithm}
\usepackage{algorithmic}

\usepackage{multirow}

\usepackage{breakcites}

\allowdisplaybreaks[1]

\ifx\proof\undefined
\newenvironment{proof}{\par\noindent{\bf Proof\ }}{\hfill\BlackBox\\[2mm]}
\fi

\ifx\theorem\undefined
\newtheorem{theorem}{Theorem}
\fi
\ifx\example\undefined
\newtheorem{example}{Example}
\fi
\ifx\lemma\undefined
\newtheorem{lemma}[theorem]{Lemma}
\fi

\ifx\corollary\undefined
\newtheorem{corollary}[theorem]{Corollary}
\fi

\ifx\assumption\undefined
\newtheorem{assumption}{Assumption}
\fi

\ifx\definition\undefined
\newtheorem{definition}{Definition}
\fi

\ifx\proposition\undefined
\newtheorem{proposition}[theorem]{Proposition}
\fi

\ifx\remark\undefined
\newtheorem{remark}{Remark}
\fi

\ifx\conjecture\undefined
\newtheorem{conjecture}[theorem]{Conjecture}
\fi

\ifx\factoid\undefined
\newtheorem{factoid}[theorem]{Fact}
\fi

\ifx\axiom\undefined
\newtheorem{axiom}[theorem]{Axiom}
\fi

\newcommand{\RN}[1]{%
	\textup{\lowercase\expandafter{\it \romannumeral#1}}%
}

\def\Adapt{\textsf{Adapt}}

\newcommand{\new}[1]{{\color{black}{#1}}}

\newcommand{\ROM}[1]{\uppercase\expandafter{\romannumeral #1\relax}}
\newcommand*\samethanks[1][\value{footnote}]{\footnotemark[#1]}
\input{./Definition.tex}
\usepackage{enumitem}
\setlist[itemize]{noitemsep, topsep=-5pt}
\title{Meta-Learning with Neural Tangent Kernels}

\author{%
  Yufan Zhou\thanks{The first two authors contribute equally. Correspondence to Changyou Chen (changyou@buffalo.edu).},\quad  Zhenyi Wang\samethanks[1] ,\quad Jiayi Xian,\quad Changyou Chen,\quad Jinhui Xu \thanks{The research of the first and fifth authors was supported in part by NSF through grants CCF-1716400 and IIS-1910492.}  \\
  Department of Computer Science and Engineering, State University of New York at Buffalo\\
  \texttt{\{yufanzho,zhenyiwa,jxian,changyou,jinhui\}@buffalo.edu} \\
}
\begin{document}

\maketitle

\begin{abstract}
    Model Agnostic Meta-Learning (MAML) has emerged as a standard framework for meta-learning, where a {\em meta-model} is learned with the ability of fast adapting to new tasks. However, as a double-looped optimization problem, MAML needs to differentiate through the whole inner-loop optimization path for every outer-loop training step, which may lead to both computational inefficiency and sub-optimal solutions. In this paper, we generalize MAML to allow meta-learning to be defined in function spaces, and propose the first meta-learning paradigm in the Reproducing Kernel Hilbert Space (RKHS) induced by the meta-model's Neural Tangent Kernel (NTK). Within this paradigm, we introduce two meta-learning algorithms in the RKHS, which no longer need \new{a sub-optimal iterative inner-loop adaptation} as in the MAML framework. We achieve this goal by 1) replacing the adaptation with a fast-adaptive regularizer in the RKHS; and 2) solving the adaptation analytically based on the NTK theory. Extensive experimental studies demonstrate \new{advantages} of our paradigm in both efficiency and quality of solutions compared to related meta-learning algorithms. Another interesting feature of our proposed methods is that they are \new{demonstrated to be more} robust to adversarial attacks and out-of-distribution adaptation than \new{popular baselines}, as demonstrated in our experiments.
\end{abstract}

\section{Introduction}

Meta-learning \citep{schmidhuber1987} has made tremendous progresses in 
%been making rapid progress in 
the last few years. It aims to learn abstract knowledge from many 
%but 
related tasks so that 
%it can be used for 
fast adaption to new and unseen tasks becomes possible. 
%A number of methods have been proposed \cite{}. 
For example, in few-shot learning, meta-learning corresponds to learning a meta-model or meta-parameters so that they can fast adapt to new tasks with a limited number of data samples. 
Among all existing meta-learning methods, Model Agnostic Meta-Learning (MAML) \citep{finn17a} is perhaps one of the most popular  and flexible ones, 
%aiming at learning a meta-model that can fast adapt to new tasks, 
with a number of follow-up works such as \citep{reptile18,PMAML2018,YaoW19,KhodakBT19,adaptive19, learn2learn19, converge20, balance20, provable20}. MAML adopts a double-looped optimization framework, where
%there are two optimization loops, and 
adaptation is achieved by one or several gradient-descent  steps in the inner-loop optimization. 
%problem. 
Such a framework could lead to some undesirable issues related to computational inefficiency and sub-optimal solutions. The main reasons are that   
%We argue that this inner-loop adaptation step could be inefficient and sub-optimal, in the sense that 
1) it is computationally expensive to back-propagate through a stochastic-gradient-descent chain, and 2) it is hard to tune the number of adaptation steps in the inner-loop as it can be different for both training and testing. Several previous works tried to address these issues, but they can only alleviate them to certain extents.
%these methods are not totally satisfactory. 
For example, first order MAML (FOMAML) \citep{finn17a} ignores the high-order terms of the standard MAML, which can speed up the training but may lead to deteriorated performance; MAML with Implicit Gradient (iMAML) \citep{IMAML2019} directly minimizes the objective of the outer-loop without performing the inner-loop optimization. But it still needs an iterative solver to estimate the meta-gradient. 

To better address these issues, we propose two algorithms that generalize meta-learning to the Reproducing Kernel Hilbert Space (RKHS) induced by the meta-model's Neural Tangent Kernel (NTK) \citep{jacot2018neural}. In this RKHS, instead of using parameter adaptation, we propose to perform an implicit function adaptation. To this end, we introduce two algorithms to avoid explicit function adaptation: one replaces the function adaptation step in the inner-loop with a new meta-objective with a fast-adaptive regularizer inspired by MAML; the other solves the adaptation problem analytically based on tools from NTK so that the meta-objective can be directly evaluated on samples in a closed-form. When restricting the function space to be RKHS, the solutions to the proposed two algorithms become conveniently solvable. In addition, we provide theoretical analysis on our proposed algorithms in the cases of using fully-connected neural networks and convolutional neural networks as the meta-model. Our analysis shows close connections between our methods and the existing ones. Particularly, we prove that one of our algorithms is closely related to MAML with some high-order terms ignored in the meta-objective function, thus endowing effective optimization. 
% For the second strategy, we prove that function adaptation in RKHS can be explicitly solved with a closed-form solution. Based on the solution, one can choose how far the meta function is adapted for new tasks. This is particularly useful because adaptation to the optima based on limited training data does not necessarily guarantee the optimally when generalizing to test data. We perform experiments on ...
In summary, our main contributions are:

\begin{itemize}
    \setlength\itemsep{0em}
    \item We re-analyze the meta-learning problem and introduce two new algorithms for meta-learning in RKHS. Different from all existing meta-learning algorithms, our proposed methods can be solved \new{efficiently without cumbersome chain-based adaptations}.
    % a single-looped\footnote{\new{We define single/double loop only for the training stage. An algorithm is double-looped if it needs two computation loops (one contained in the other) in training. Otherwise, it is called single-loop.}} optimization procedure.
    \item We conduct theoretically analysis on the proposed algorithms, which suggests that 
    our proposed algorithms are closely related to the existing MAML methods when fully-connected neural networks and convolutional neural networks are used as the meta-model.
    \item We conduct extensive experiments to validate our algorithms. Experimental results indicate the effectiveness of our proposed methods, through standard few-shot learning, robustness to adversarial attacks and out-of-distribution adaptation.
\end{itemize}

\vspace{-0.2cm}
\section{Preliminaries}
\vspace{-0.1cm}
\subsection{Meta-Learning}
\vspace{-0.2cm}
%Meta Learning \cite{schmidhuber1987} has achieved rapid progress in the last few years. It aims to learn abstract knowledge from many but related tasks, so that it can be used for adapting to new and unseen tasks during training. There are many meta learning models \cite{finn2019} and they 
Meta-learning can be roughly categorized as black-box adaptation methods \citep{learn2learn2016, NTM2014, SNAILICLR18}, optimization-based methods \citep{finn17a},  non-parametric methods \citep{matching16, protonet17, metadata20} and Bayesian meta-learning methods \citep{PMAML2018, BMAML2018, ABML19}. In this paper, we focus on the framework of Model Agnostic Meta-Learning (MAML) \citep{finn17a}, which has two key components, meta initialization and fast adaptation. Specifically, MAML solves the meta-learning problem through a double-looped 
%bi-level 
optimization procedure.
%consisting of two optimization loops. 
In the inner-loop, MAML runs a task-specific adaptation procedure to transform a meta-parameter, $\thetab$, to a task-specific parameter, $\{\phib_m\}_{m=1}^B$, for a total of $B$ different tasks. In the outer-loop, MAML minimizes a total loss   of $\sum_{m=1}^B \mathcal{L}(f_{\phib_m})$ with respect to meta-parameter $\thetab$, where $f_{\phib_m}$ is the model adapted on task $m$ that is typically represented by a deep neural network.
It is worth noting that in MAML, one potential problem is to compute 
the meta-gradient $\nabla_{\thetab} \sum_{m=1}^B \mathcal{L}(f_{\phib_m})$. It requires one to differentiate through the whole inner-loop optimization path, which could be very inefficient. %There are already some methods try to solve this problem \cite{finn17a, reptile18, IMAML2019}. This paper tries to solve the problem from a different point of view, which will be discussed later.

% \begin{algorithm}[ht!]
% 	\caption{Model-Agnostic Meta-Learning (MAML) \cy{This is not necessary}}\label{algo:maml}
% 	\begin{algorithmic}
% 	    \REQUIRE $p(\mathcal{T})$: distribution over tasks, randomly initialize neural network parameters $\thetab$, learning rate $\alpha, \beta$, loss function $\mathcal{L}$.
% 		\WHILE{not done}
% 		    \STATE Sample a batch of tasks $\{{\mathcal{T}}_m\}_{m=1}^B \thicksim p(\mathcal{T})$
%     		\FORALL {${\mathcal{T}}_m $}
%         		\STATE Run several steps of gradient descent with learning rate $\alpha$ on samples of $\mathcal{T}_m$. Denote resulting adapted parameters as $\phib_m$.
%     		\ENDFOR
%         	\STATE Update $\thetab \leftarrow \thetab - \beta \sum_{m=1}^B \nabla_\thetab  \mathcal{L}(f_{\phib_m})$.
% 		\ENDWHILE
% 	\end{algorithmic}
% \end{algorithm}

\vspace{-0.1cm}
\subsection{Gradient Flow}
\vspace{-0.2cm}
Our proposed method relies on the concept of gradient flow. Generally speaking, gradient flow is a continuous-time version of gradient descent. In the finite-dimensional parameter space, a gradient flow is defined by an ordinary differential equation (ODE),  $\mathrm{d}\thetab^t/\mathrm{d}t = - \nabla_{\thetab^t} F(\thetab^t)$, with a starting point $\thetab^0$ and function $F:R^d \rightarrow R$. Gradient flow is also known as steepest descent curve.

One can generalize gradient flows to infinite-dimensional function spaces. Specifically, given a function space $\mathcal{H}$, a functional $\mathcal{F}: \mathcal{H} \rightarrow R$, and a starting point $f^0 \in \mathcal{H}$, a gradient flow is similarly defined as the solution of $\mathrm{d} f^t/\mathrm{d} t = - \nabla_{f^t} \mathcal{F}(f^t)$. This is a curve in the function space $\mathcal{H}$. 
In this paper, 
%to emphasize the gradient in a function space is calculated with respect to function $f^t$, 
we use notation $\nabla_{f^t} \mathcal{F}(f^t)$, instead of $\nabla_{\mathcal{H}} \mathcal{F}(f^t)$, to denote the general function derivative of the energy functional $\mathcal{F}$ with respect to function $f^t$ \citep{OT_book}.

%As the reader may notice, the gradient descent algorithm can actually be regarded as a discrete version of gradient flow implementation.
\vspace{-0.1cm}
\subsection{The Neural Tangent Kernel}
\vspace{-0.2cm}
Neural Tangent Kernel (NTK) is a recently proposed technique for characterizing the dynamics of a neural network under gradient descent \citep{jacot2018neural, arora19, NIPS2019_9063}. NTK allows one to analyze deep neural networks (DNNs) in RKHS induced by NTK. One immediate benefit of this is that the loss functional in the function space is often convex, even when it is highly non-convex in the parameter space \citep{jacot2018neural} \new{\footnote{\new{Let $\mathcal{H}$ be the function space, $F$ be the realization function for neural network defined in Section \ref{subsec:meta-learning-in-RKHS}. Note even if a functional loss ({\it e.g.}, L2 loss) $\mathcal{E}: \mathcal{H} \rightarrow R$ is convex on $\mathcal{H}$, the composition $\mathcal{E} \circ F$ is in general not.}}}. This property allows one to better understand the property of DNNs. Specifically, let $f_{\thetab}$ be a DNN parameterized by $\thetab$. The corresponding NTK $\Thetab$ is defined as: $\Thetab(\xb_1, \xb_2) = \dfrac{\partial f_\thetab(\xb_1)}{\partial \thetab} \dfrac{\partial f_{\thetab}(\xb_2)}{\partial \thetab} ^\intercal$, where $\xb_1,\xb_2$ are two data points. In our paper, we will define meta-learning on an RKHS induced by such a kernel.

\section{Meta-Learning in RKHS}
We first define the meta-learning problem in a general function space, and then restrict the function space to be an RKHS, where two frameworks will be proposed to make meta-learning feasible in RKHS, along with some theoretical analysis. For simplicity, in the following we will hide the superscript time $t$ unless necessary, {\it e.g.}, when the analysis involves time-changing.

\subsection{Meta-Learning in  Function Space}
Given a function space $\mathcal{H}$, a distribution of tasks $P(\mathcal{T})$, and a loss function $\mathcal{L}$, the goal of meta-learning is to find a {\em meta function} $f^* \in \mathcal{H}$, so that it performs well after simple adaptation on a specific task. Let $\mathcal{D} _m^{tr}$ and $\mathcal{D} _m^{test}$ be the training and testing sets, respectively, sampled from a data distribution of task $\mathcal{T}_m$. The meta-learning problem on function space $\mathcal{H}$ is defined as:
\begin{align}\label{eq:goal_function_space}
    f^* = \argmin_{f \in \mathcal{H}} \mathcal{E}(f), \text{ with } \mathcal{E}(f) = \mathbb{E}_{\mathcal{T}_m}\left[\mathcal{L}\Big(\Adapt(f, \mathcal{D}_m^{tr}), \mathcal{D} _m^{test}\Big)\right]
\end{align}
where $\Adapt$ denotes some adaptation algorithms, {\it e.g.}, several steps of gradient descent; $\mathcal{E}:\mathcal{H} \rightarrow R$ is called energy functional, which is used to evaluate the model represented by the function $f$.

In theory, solving \eqref{eq:goal_function_space} is equivalent to solving the gradient flow equation $\textup{d} f^t/\textup{d} t = -\nabla_{f^t} \mathcal{E}(f^t)$. 
However, solving the gradient flow equation is generally infeasible, since $\RN{1})$ it is hard to directly apply optimization methods in function space and $\RN{2})$ the energy functional $\mathcal{E}$ contains an adaptation algorithm $\Adapt$, making the functional gradient infeasible. %As a result, 
Thus, a better way is to design a special energy functional so that it can be directly optimized without running the specific adaptation algorithm. In the following, we first specify the functional meta-learning problem in RKHS, and then propose two methods to derive efficient solutions for the problem.

\subsection{Meta-Learning in  RKHS}\label{subsec:meta-learning-in-RKHS}

We consider a function $f$ that is parameterized by $\thetab \in \mathbb{R}^P$, denoted as $f_{\thetab}$, with $P$ being the number of parameters. Define a realization function $F: \mathbb{R}^P \rightarrow \mathcal{H}$ that maps parameters to a function. With these, we can then define an energy function in the parameter space as $E \triangleq \mathcal{E} \circ F: R^P \rightarrow R$ with $\circ$ being the composition operator. Consequently, with an initialized $\thetab^0$, we can define the gradient flow of $E(\thetab^t)$ in parameter space as: $\textup{d}\thetab^t/\textup{d} t = -\nabla_{\thetab^t} E(\thetab^t)$. 
In the following, we first establish an equivalence between the gradient flow in RKHS and the gradient flow in the parameter space. We then propose two algorithms for meta-learning in the RKHS induced by NTK.

\begin{theorem}\label{thm:equal_gf}
    Let $\mathcal{H}$ be an RKHS induced by the NTK $\Thetab$ of $f_{\thetab}$.
    With $f^0 = f_{\thetab^0}$, the gradient flow of $\mathcal{E}(f^t)$ coincides with the function evolution of $f_{\thetab^t}$ driven by the gradient flow of $E(\thetab^t)$. 
\end{theorem}
The proof of Theorem \ref{thm:equal_gf} relies on the property of NTK \citep{jacot2018neural}, and is provided in the Appendix. Theorem \ref{thm:equal_gf} serves as a foundation of our proposed methods, which indicates that solving the meta-learning problem in RKHS can be done by some appropriate manipulations. In the following, we describe two different approaches termed Meta-RKHS-\ROM{1} and Meta-RKHS-\ROM{2}, respectively.

% \begin{theorem}\label{thm:equa_gradient}
%     If $f_\thetab$ is a neural network with parameter $\thetab \in R^P$, $\mathcal{H}$ is the Reproducing Kernel Hilbert Space (RKHS) induced by $\Thetab$, where $\Thetab$ is the Neural Tangent Kernel (NTK) defined as
%     \[
%     \Thetab(\xb_1, \xb_2) = \dfrac{\partial f_\thetab(\xb_1)}{\partial \thetab} \dfrac{\partial f_{\thetab}(\xb_2)}{\partial \thetab} ^\intercal.
%     \] 
%     Then we have
%     \[
%     \Vert \nabla _{f_\thetab} \mathcal{L}(f_\thetab) \Vert_{\mathcal{H}}^2 = \Vert \nabla _{\thetab} \mathcal{L}(f_\thetab) \Vert^2 
%     \]
% \end{theorem}

\subsection{Meta-RKHS-\ROM{1}: Meta-Learning in RKHS without Adaptation}

Our goal is to design an energy functional that has no adaptation component, but is capable of achieving fast adaptation. For this purpose, we first introduce two definitions: empirical loss function $\mathcal{L}(f_\thetab, \mathcal{D}_m)$ and expected loss function $\mathcal{L}(f_\thetab)$. 
%For example, 
Let $\mathcal{D}_m = \{\xb_{m,i}, \yb_{m,i}\}_{i=1}^n$ be a set containing the data of a regression task $\mathcal{T}_m$. The empirical loss function $\mathcal{L}(f_\thetab, \mathcal{D}_m)$ and the expected loss function $\mathcal{L}_m(f_\thetab)$ can be defined as:
\begin{align*}
    \mathcal{L}(f_\thetab, \mathcal{D}_m) = \dfrac{1}{2n} \sum_{i=1}^n\big\Vert f(\xb_{m,i}) - \yb_{m,i}\big\Vert^2,~~~~
    \mathcal{L}_m(f_\thetab) = \mathbb{E}_{\xb_{m}, \yb_{m}}\left[\dfrac{1}{2}\big\Vert f(\xb_{m}) - \yb_{m}\big\Vert^2\right].
\end{align*}

Our idea is to define a regularized functional such that it endows the ability of fast adaptation in RKHS. Our solution is based on some property of the standard MAML. 
%To that end, 
We start from analyzing the meta-objective of MAML with a $k$-step gradient-descent adaptation, {\it i.e.}, applying $k$ gradient-descent steps in the inner-loop. The objective can be formulated as
\vspace{-0.2cm}
\[
\thetab^* = \argmin_{\thetab} \mathbb{E}_{\mathcal{T}_m}\left[\mathcal{L}(f_{\phib}, \mathcal{D}_m^{test})\right], \text{ with } \phib = \thetab - \alpha \sum_{i=0}^{k-1}\nabla_{\thetab_i} \mathcal{L}(f_{\thetab_i}, \mathcal{D}_m^{tr})~,
\]
where $\alpha$ is the learning rate of the inner-loop, $\thetab_0 = \thetab$, and $\thetab_{i+1} = \thetab_{i} - \alpha \nabla_{\thetab_i} \mathcal{L}(f_{\thetab_i}, \mathcal{D}_m^{tr})$\footnote{\new{For ease of  our later notation, we write the gradient $\nabla_{\thetab_i} \mathcal{L}$ (thus the parameter as well) as a row vector.}}. 
By Taylor expansion, we have
\vspace{-0.2cm}
\begin{align}\label{eq:taylor_maml}
     \mathbb{E}_{\mathcal{T}_m}\left[\mathcal{L}(f_{\phib}, \mathcal{D}_m^{test})\right] \approx  \mathbb{E}_{\mathcal{T}_m}\left[\mathcal{L}(f_{\thetab}, \mathcal{D}_m^{test}) -\alpha \sum_{i=0}^{k-1}\nabla_{\thetab_i} \mathcal{L}(f_{\thetab_i}, \mathcal{D}_m^{tr}) \nabla_{\thetab} \mathcal{L}(f_{\thetab}, \mathcal{D}_m^{test})^\intercal \right].
\end{align}
Since $\mathcal{D}_m^{tr}$ and $\mathcal{D}_m^{test}$ come from the same distribution, \eqref{eq:taylor_maml} is an unbiased estimator of 
\vspace{-0.2cm}
\begin{align}\label{eq:exp_taylor_maml}
    \mathcal{M}_k = \mathbb{E}_{\mathcal{T}_m}[\mathcal{L}_m(f_\thetab) - \sum_{i=0}^{k-1} \beta_i ], \text{ where } \beta_i = \alpha \nabla_{\thetab_i} \mathcal{L}_m(f_{\thetab_i}) \nabla_{\thetab}\mathcal{L}_m(f_\thetab)^\intercal.
\end{align}
We focus on the case of $k=1$, which is 
$\mathcal{M}_1 = \mathbb{E}_{\mathcal{T}_m}\left[\mathcal{L}_m(f_{\thetab})\right] - \alpha \mathbb{E}_{\mathcal{T}_m}\left[\Vert \nabla_{\thetab}\mathcal{L}_m(f_\thetab)\Vert ^2\right]$. 
The first term on the RHS is the traditional multi-task loss evaluated at $\thetab$ for all tasks. The second term corresponds to the negative gradient norm; minimizing it means
%corresponds to 
choosing a $\thetab$ with the maximum gradient norm. Intuitively, when $\thetab$ is not a stationary point of a task, one should choose the steepest descent direction to reduce the loss maximally for a specific task, thus leading to {\em fast adaptation}.

The above understanding suggests us to
%Inspired by the above analysis, 
%we 
propose  the following regularized energy functional, $\widetilde{\mathcal{E}}_\alpha$, for meta-learning \new{in the RKHS induced with the NTK} for {\em fast function adaptation}:
\begin{align}\label{eq:energy_functional_1}
    \widetilde{\mathcal{E}}(\alpha, f_\thetab) = \mathbb{E}_{\mathcal{T}_m}\left[\mathcal{L}_m(f_\thetab) - \alpha \Vert \nabla _{f_\thetab}\mathcal{L}_m(f_\thetab) \Vert^2_{\mathcal{H}} \right],
\end{align}
where $\Vert \cdot\Vert_{\mathcal{H}}$ denotes the functional norm in $\mathcal{H}$, and $\alpha$ is a hyper-parameter. \new{The above objective is inspired by the Taylor expansion of the MAML objective, but is defined in the RKHS induced by the NTK. Its connection with MAML and some functional-space properties will be discussed later.}

\paragraph{Solving the Function Optimization Problem}
To \new{minimize} \eqref{eq:energy_functional_1}, we first derive  Theorem~\ref{thm:equal_gradient} to reduce the function optimization problem to a parameter optimization problem.  %with the function parameterized
%the function with 
%by $\thetab$.
\vspace{-0.2cm}
\begin{theorem}\label{thm:equal_gradient}
    Let $f_\thetab$ be a neural network with parameter $\thetab$ and $\mathcal{H}$ be the RKHS induced by the NTK $\Thetab$ of $f_{\thetab}$. Then, the following are equivalent 
    \[
    \widetilde{\mathcal{E}}(\alpha, f_\thetab) = \mathcal{M}_1, \text{ and } \Vert \nabla _{f_\thetab}\mathcal{L}_m(f_\thetab) \Vert^2_{\mathcal{H}} = \Vert \nabla_{\thetab} \mathcal{L}_m(f_\thetab)\Vert ^2~.
    \]
\end{theorem}

Theorem~\ref{thm:equal_gradient} is crucial to our approach
%remarkable, 
as it indicates that solving problem \eqref{eq:energy_functional_1} is no more difficult than the original parameter-based MAML, although it only considers one-step adaptation case. Next, we will show that multi-step adaptation in the parameter space can also be well-approximated by our objective \eqref{eq:energy_functional_1} but with a scaled regularized parameter $\alpha$. In the following, we consider the squared loss $\mathcal{L}$. 
%to be the squared loss. 
The case with the cross-entropy loss is discussed in the Appendix. We assume that  $f_{\thetab}$ is parameterized by either fully-connected or convolutional neural networks, and only consider the impact of number of hidden layers $L$ in our theoretical results.

\vspace{-0.2cm}
\begin{theorem}\label{thm:delta_bound_nn}
    Let $f_{\thetab}$ be a fully-connected neural network with $L$ hidden layers and ReLU activation function,  $s_1, ..., s_{L+1}$ be the spectral norm of the weight matrices,  $s=\max_h s_h$, and $\alpha$ be the learning rate of gradient descent. If $\alpha \leq O(qr)$ with $q=\min(1/(Ls^L), L^{-1/(L+1)})$ and $r = \min(s^{-L}, s)$, then the following holds
    \vspace{-0.3cm}
    \[
    \vert  \mathcal{M}_k -\widetilde{\mathcal{E}}(k\alpha, f_\thetab) \vert \leq O\Big(\dfrac{1}{L}\Big).
    \]
\end{theorem}

\vspace{-0.2cm}
\begin{theorem}\label{thm:delta_bound_cnn}
    Let $f_{\thetab}$ be a convolutional neural network with $L-l$ convolutional layers and $l$ fully-connected layers and with ReLU activation function, and $d_x$ be the input dimension. 
    Denote  by $W^h$   the parameter \textbf{vector} of the convolutional layer for $h\leq L-l$, and 
    %$W^{h}$ with $L-l+1<h\leq L+1$ be 
    the weight \textbf{matrices} of the fully connected layers for $L-l+1\leq h\leq L+1$. $\Vert \cdot \Vert_2$ means both the spectral norm of a matrix and the Euclidean norm of a vector.  
    Define $s_h = \sqrt{d_x}\Vert W^h \Vert_2$ if $h=1,...,L-l$, and $\Vert W^h \Vert_2$ if $L-l+1 \leq h \leq L+1$. 
    % \[
    % s_h = \left\{
    %             \begin{array}{ll}
    %              \sqrt{d_x}\Vert W^h \Vert_2,\quad  &\text{if $h=1,...,L-l$} \\
    %              \Vert W^h \Vert_2,\quad &\text{if $L-l+1< h \leq L+1$}\\
    %             \end{array}
    %           \right.
    % \]
    Let $s=\max_h s_h$ and $\alpha$ be the learning rate of gradient descent. If $\alpha \leq O(qr)$ with $q=\min(1/(Ls^L), L^{-1/(L+1)})$ and $r = \min(s^{-L}, s)$, the following holds
    \[
    \vert \mathcal{M}_k - \widetilde{\mathcal{E}}(k\alpha, f_\thetab) \vert \leq O\Big(\dfrac{1}{L}\Big).
    \]
    % for all $0 \leq i\leq k$:
    % \[
    %  \big\vert \beta_i - \alpha \Vert \nabla _{f_\thetab}\mathcal{L}_m(f_\thetab) \Vert^2_{\mathcal{H}} \big\vert \leq O(qs^L) \leq O\Big(\dfrac{1}{L}\Big).
    % \]
\end{theorem}
\new{The above Theorems indicate that, for a meta-model with fully-connected and convolutional layers, the proposed Meta-RKHS-\ROM{1} can be an efficient approximation of MAML with a bounded error.}
%We show the relationship between some previous works and our proposed method.

\vspace{-0.3cm}
\paragraph{Comparisons with Reptile and MAML}
\new{Similar to Reptile and MAML, the testing stage of Meta-RKHS-\ROM{1} also requires gradient-based adaptation on meta-test tasks.} 
By Theorem \ref{thm:equal_gf}, we known that gradient flow of an energy functional can be approximated by gradient descent in a parameter space. Reptile with 1-step adaptation \citep{reptile18} is equivalent to the approximation of the gradient flow of $\widetilde{\mathcal{E}}(\alpha, f_\thetab)$ with $\alpha=0$, which does not include the fast-adaptation regularization as in our method. \new{For a fairer comparison on the efficiency, we will discuss the computational complexity later.}

% \paragraph{Some Theoretical Results}
% Although our proposed method is inspired by MAML with 1-step adaptation, the theoretical relationship is still unclear. Let k be an arbitrary positive integer, $\alpha$ be the inner-loop learning rate of MAML, we now show that the training process of MAML with k-step adaptation is actually governed by minimizing the proposed energy functional $\widetilde{\mathcal{E}}(k\alpha, f_\thetab)$. 

% In the case of $k \geq 2$ and $\mathcal{L}$ is the squared loss, \eqref{eq:exp_taylor_maml} is governed by $\widetilde{\mathcal{E}}(k\alpha, f_\thetab)$ according to the following theorems. We provide the proofs in Appendix. There are similar results when $\mathcal{L}$ is cross-entropy loss, which are also discussed in the Appendix due to the limitation of space.

From the equivalent parameter-optimization form indicated in Theorem~\ref{thm:equal_gradient}, we know that our energy functional $\widetilde{\mathcal{E}}(\alpha, f_\thetab)$ is closely related to MAML. However, with this form, our method does not need the explicit adaptation steps in \new{training} ({\it i.e.}, the inner-loop of MAML), leading to a \new{simpler} optimization problem. We will show that our proposed method leads to better results.

\subsection{Meta-RKHS-\ROM{2}: Meta-Learning in RKHS with a Closed-form Adaptation}
In this section, we present our second solution for meta-learning in RKHS by deriving a closed-form adaptation function, {\it i.e.}, we focus on a case where $\Adapt(f, \mathcal{D}_m^{tr})$ is  analytically solvable using the theory of NTK. 
Specifically, we are given a loss function $\mathcal{L}$, 
%to be minimized, 
tasks $\mathcal{T}_m$ with randomly split training set $\mathcal{D}_m^{tr} = \{\xb^{tr}_{m,i}, \yb^{tr}_{m,i}\}_{i=1}^n$, and testing set $\mathcal{D}_m^{test}$. Let  $\thetab^t_m$ and $f^t_{m, \thetab}$ denote the parameters and the corresponding function at time $t$ adapted by task $\mathcal{T}_m$ from the meta parameter $\thetab$ and meta function $f_{\thetab}$, respectively. From the NTK theory \citep{jacot2018neural,arora19, NIPS2019_9063}, we can write the function/parameter evolution as:
\[
\dfrac{\textup{d}\thetab^t_m}{\textup{d}t} = -\nabla_{\thetab_m^t} \mathcal{L}(f^t_{m, \thetab}, \mathcal{D}_m^{tr}),\quad\text{ and }\quad   \dfrac{\textup{d} f^t_{m, \thetab}}{\textup{d} t}= \dfrac{\textup{d} \thetab_m^t}{\textup{d} t} \dfrac{\partial f^t_{m, \thetab}}{\partial \thetab_m^t} ^\intercal= \sum_{i=1}^n\dfrac{\partial \mathcal{L}(f^t_{m, \thetab}, \mathcal{D}_m^{tr})}{\partial f^t_{m,\thetab}(\xb^{tr}_{m,i})} \Thetab(\xb^{tr}_{m,i}, \cdot). 
\]
%We want to emphasize 
The above differential equation corresponds to the adaptation step, {\it i.e.}, how to adapt the meta parameter/function for task $m$. 
%Fortunately, based on 
By the NTK theory, we can show that this admits closed-form solutions. In our meta-learning settings, this indicates that no explicit adaptation steps are necessary.

To see why this is the case, 
%To this end, 
we first investigate the regression case, where the loss function $\mathcal{L}$ is 
%designed to be 
the squared loss. Let $\xb \in \mathcal{D}_m^{test}$ be a test data point. As shown in \cite{arora19, NIPS2019_9063}, with a large enough neural network we can safely assume that NTK will not change too much during the training. In this case, we can have a closed-form solution for $f_{m,\thetab}^t$ as 
\begin{align} \label{eq:ntk_time_t}
        f_{m, \thetab}^t(\xb) = f_{\thetab}(\xb) + H(\xb, \Xb^{tr}_m)H^{-1}(\Xb^{tr}_m, \Xb^{tr}_m)\left(e^{-tH(\Xb^{tr}_m, \Xb^{tr}_m)} - \Ib\right)\left(f_{\thetab}(\Xb^{tr}_m) - Y^{tr}\right)~,
\end{align}
where $e$ is the matrix exponential map, which can be approximated by $Pad\acute{e}$ approximation \citep{pade96}. $H(\Xb^{tr}_m, \Xb^{tr}_m)$ is an $n \times n$ kernel matrix with its $(i,j)$ element being $\Thetab(\xb_{m,i}, \xb_{m,j})$, $H(\xb, \Xb^{tr}_m)$ is a $1 \times n$ vector with its $i$-th element being $\Thetab(\xb, \xb_{m,i})$,  $f_{\thetab}(\Xb^{tr}_m) \in R^n$ is the predictions of all training data at the initialization, and $Y^{tr} \in R^n$ is the target value of the training data. 
Specifically, at time $t = \infty$, we have 
\begin{align} \label{eq:ntk_time_infinity}
    f_{m, \thetab}^{\infty}(\xb) = f_{\thetab}(\xb) + H(x, \Xb^{tr}_m)H^{-1}(\Xb^{tr}_m, \Xb^{tr}_m)\left(Y^{tr} - f_{\thetab}(\Xb^{tr}_m)\right)~.
\end{align}
The above results allow us to directly define an energy functional by substituting $\Adapt(f, \mathcal{D}^{tr}_m)$ in \eqref{eq:goal_function_space} with its closed-form solution $f_{m, \thetab}^t$. In other words, our new energy functional is
\begin{align}\label{eq:energy_functional_2}
    \overline{\mathcal{E}}(t, f_{\thetab}) = \mathbb{E}_{\mathcal{T}_m}\left[\mathcal{L}_m(f_{m, \thetab}^t)\right]~,
\end{align}
where $f_{m, \thetab}^t$ is defined in \eqref{eq:ntk_time_t}, and $\mathcal{L}_m(f_{m, \thetab}^t)$ is the expectation of $\mathcal{L}\Big(f_{m, \thetab}^t, \mathcal{D} _m^{test}\Big)$. 
For classification problems, we follow the same strategy as in \cite{arora19} to extend regression to classification. Mores details can be found in the Appendix, including the algorithm in Appendix~\ref{app:alg}.

\paragraph{On Potential Robustness of Meta-RKHS-\ROM{2}}
% It is known that the NTK evolution on the NTK regime is essentially a linear model \citep{jacot2018neural,NIPS2019_8559}. In our setting, this indicates that conditioned on the meta parameter, each adaptive task model is a linear model. Intuitively, linear models should have fewer adversarial samples and thus can be deem more robust. Combining with the expressive meta-model, our Meta-RKHS-\ROM{2} is thus expected to perform much better than a simple linear model while endowing better robustness than standard neural networks. A theoretical understanding of the robustness would be much involved and complicated, which we leave for interesting future research. However, we will demonstrate through extensive empirical verification, and show that Meta-RKHS-\ROM{2} is indeed a more robust model. 

\new{
Our extensive empirical studies show that Meta-RKHS-\ROM{2} is a more robust model than related baselines. We provide an intuitive explanation on the potential robustness of Meta-RKHS-\ROM{2}, as we find current theories of both robustness machine learning and NTK are insufficient for a formal explanation. Our explanation is based on some properties of both the meta-learning framework and NTK: 1) Strong initialization (meta model): For NTK to generalize well, we argue that it is necessary to start the model with a good initialization. This is automatically achieved in our meta-learning setting, where the meta model serves as the initialization for NTK predictions. Actually, this has been supported by recent research \citep{Fortetal_nips20}, which shows that there is a chaotic stage in the NTK prediction with finite neural networks, and the NTK regime can be reachable with a good initialization. 2) Low complex classification boundary: It is known that NTK is a linear model in the NTK regime. Intuitively, generating adversarial samples with a lower complex model should be relatively harder because there is less data in the vicinity of the decision boundary compared to a more complex model, making the probability of the model being attacked smaller. Thus we argue that our model can be more robust than standard meta learning models.
3) Our NTK-based model is robust enough to adapt with different time steps. And these finite time steps can be more robust to adversarial attacks than that of the infinite-time limit partly due to the complexity of back-propagating gradients. We note each of the individual factors might not be enough to ensure robustness. Instead, we argue it is the combination effect of these factors that lead to robustness of our model. Formal analysis is out of the scope of this paper and  left for future work.
}

\vspace{-0.1cm}
\paragraph{Connection with Meta-RKHS-\ROM{1}}
The proposed two methods choose different strategies %perspectives 
to avoid explicit adaptation in meta-learning, which seem to be two very different algorithms.
We prove below theorem, which indicates that the difference of the underlying gradient flows of the two algorithms indeed increases w.r.t.\! both $T$ and the depth $L$ of a DNN (we only consider impacts of $T$ and $L$). 
% However, as we can see in the following theorem,  the difference between the corresponding two energy functional can actually be bounded.
%, stated in Theorem~\ref{thm:energy_functional_relation}.
\vspace{-0.2cm}
\begin{theorem}\label{thm:energy_functional_relation}
    Let $f_{\thetab}$ be a neural network with $L$ hidden layers, with each layer being either fully-connected or convolutional. Assume that $\|\mathcal{L}\|_{\infty} < \infty$. Then, $error(T)  = \vert \widetilde{\mathcal{E}}(T, f_{\thetab})  - \overline{\mathcal{E}}(T, f_{\thetab}) \vert$ is a non-decreasing function of $T$. Furthermore, for arbitrary $T>0$ we have $error(T) \leq O\big( T^{2L+3}\big)$.
\end{theorem}
\vspace{-0.2cm}

\new{Actually, Meta-RKHS-\ROM{2} implicitly contains a term of functional gradient norm because $\overline{\mathcal{E}}(T, f_{\thetab}) = \mathbb{E}_{\mathcal{T}_m}\left[\mathcal{L}_m(f_{\thetab}) - \int_0 ^T \Big\Vert \nabla_{\thetab^t}\mathcal{L}_m(f_{m, \thetab}^t)  \Big\Vert ^2  \text{d} t\right]$. The difference compared to Meta-RKHS-\ROM{1} mainly comes from the fact that Meta-RKHS-\ROM{1} can be regarded as an approximation of time-discrete adaptation, while Meta-RKHS-\ROM{2} is based on time-continuous adaptation.} In our experiments, we observe that Meta-RKHS-\ROM{1} is as fast as FOMAML, which means that it is more computationally efficient than the standard MAML. Meanwhile Meta-RKHS-\ROM{2} is the more robust model in  tasks of adversarial attack and out-of-distribution adaptation.

% On the other hand, $\widetilde{\mathcal{E}}$ can be taken as a first order approximation to $\overline{\mathcal{E}}$ and their difference could be measured by the difference between $\nabla_{\theta^t}
%   \mathcal{L}_m( f_{m, \theta}^t)$ and $\nabla_{\theta} \mathcal{L}_m(f_{\theta} )$, which grows as $T$ increases (see proof in appendix).
% \par
\vspace{-0.2cm}
\paragraph{Connection with iMAML} 
Our proposed method is similar to the iMAML algorithm \citep{finn2019} in the sense that both methods try to solve meta-learning without executing  the optimization path. Different from iMAML, which still relies on an iterative solver, our method only needs to solve a \new{simpler} optimization problem \new{due to the closed-form adaptation}.

\begin{table}[t!]
    \vspace{-0.1in}
    \centering
    \caption{Running time comparison per iteration with $C_1 = d_xp+Lp^2$ and $C_2 =d_xp+Ld_xp^2$.}
    \label{tab:time_complexity}
    \begin{tabular}{ccccc}
        \toprule
        & FOMAML & Reptile & Meta-RKHS-I & Meta-RKHS-II\\
        \midrule
        Fully-connected & $O(n(k+1)C_1)$ & $O(nkC_1)$ & $O(nC_1)$& $O(nC_1+n^3)$\\
        Convolutional &  $O(n(k+1)C_2)$ & $O(nkC_2)$ & $O(nC_2)$& $O(nC_2+n^3)$\\
        \bottomrule
    \end{tabular}
    \vspace{-0.6cm}
\end{table}
\subsection{Time Complexity Analysis}
We compare the time complexity of our proposed methods with other first-order meta-learning methods. Without loss of generality, we analyze the complexity in the case of a $L$-layer MLP or $L$-layer convolutional neural networks. Recall that $d_x$ is the input dimension. Assume each layer has width (filter number) $O(p)$. Let $n$ be the data batch size, $k$ the adaptation steps of inner-loop optimization. We summarize the time complexity in Table \ref{tab:time_complexity}, where we simply assume the complexity of multiplying matrices with sizes $a\times b$ and $b\times c$ to be $O(abc)$. Note in the meta-learning setting, $n$ is typically small, indicating the efficiency of our proposed methods. %For simplicity, we assumed complexity of multiplying matrices with shape $a\times b$ and $b\times c$ to be $O(abc)$, which can be further improved. However, this improvement will not influence our conclusion because all the methods will be improved consequently.

\section{Experiments}
\vspace{-0.2cm}
We conduct a set of experiments to evaluate the effectiveness of our proposed methods, including a sine wave regression toy experiment, few-shot classification, robustness to adversarial attacks, out-of-distribution generalization and ablation study. Due  to space limit, more results are provided in the Appendix. We compare our models with related baselines including MAML \citep{finn17a}, the first order MAML (FOMAML) \citep{finn17a}, Reptile \citep{reptile18} and iMAML \citep{IMAML2019}. Results are reported as mean and variance over three independent runs. 

\vspace{-0.2cm}
\subsection{Regression}
\vspace{-0.2cm}
Following \cite{finn17a, reptile18}, we first test our proposed methods on the 1-dimensional sine wave regression problem. This problem is instructive, where a model is trained on many different sine waves with different amplitudes and phases, and tested by adapting the trained model to new sine waves with only a few data points using a fixed number of gradient-descent steps. %By design, only if the method learns some abstract knowledge from all the sine waves, it will adapt to new sine waves. 
Following \cite{finn17a, reptile18}, we use a fully-connected neural network with 2 hidden layers and the ReLU activation function. The results are shown in Figure \ref{fig:regression}. 

\begin{figure}[tb!]
    \centering
        \subfigure[Random Initialized]{\includegraphics[width=0.27\linewidth]{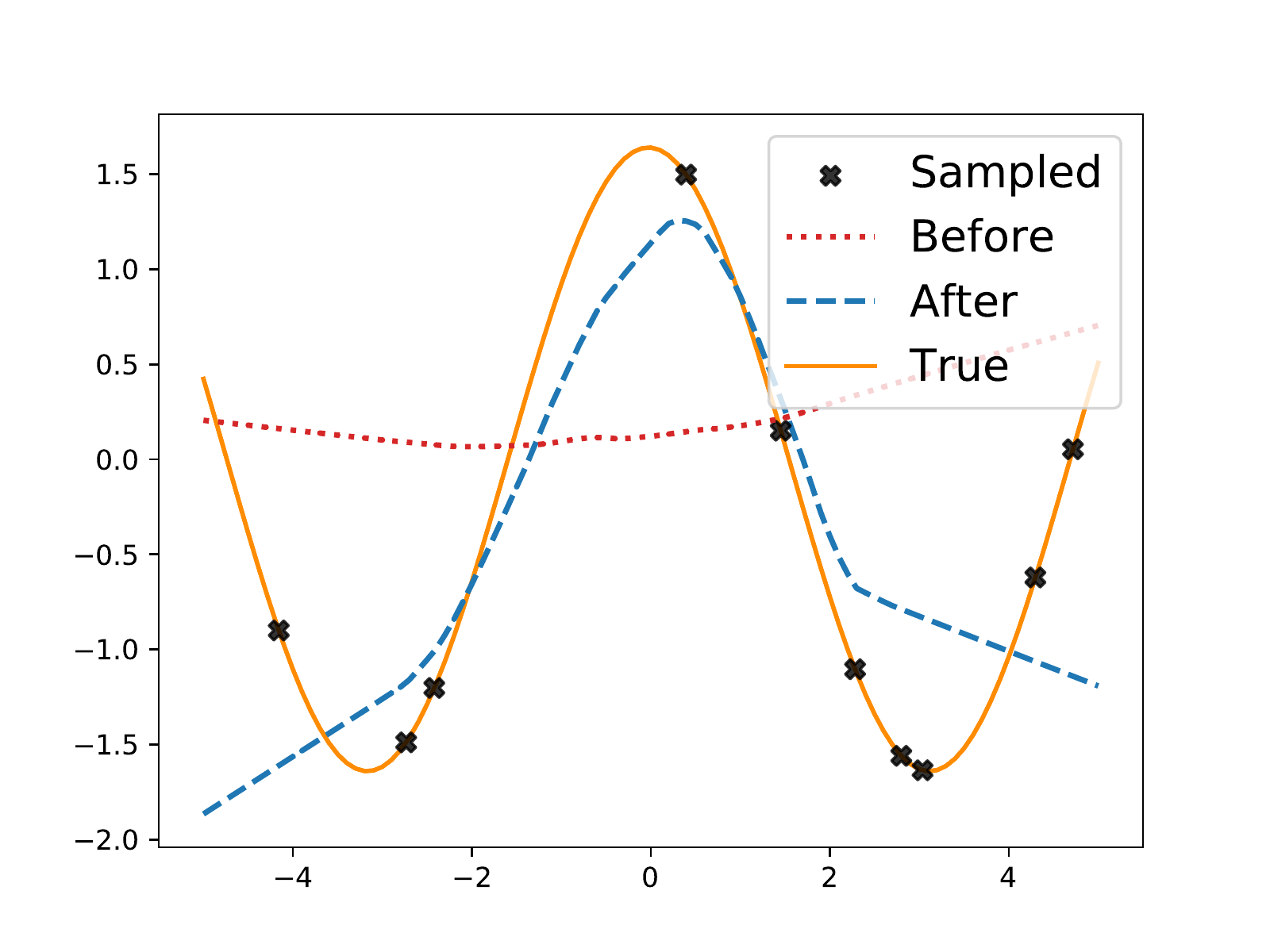}}
        \subfigure[ Meta-RKHS-\ROM{1}]{\includegraphics[width=0.27\linewidth]{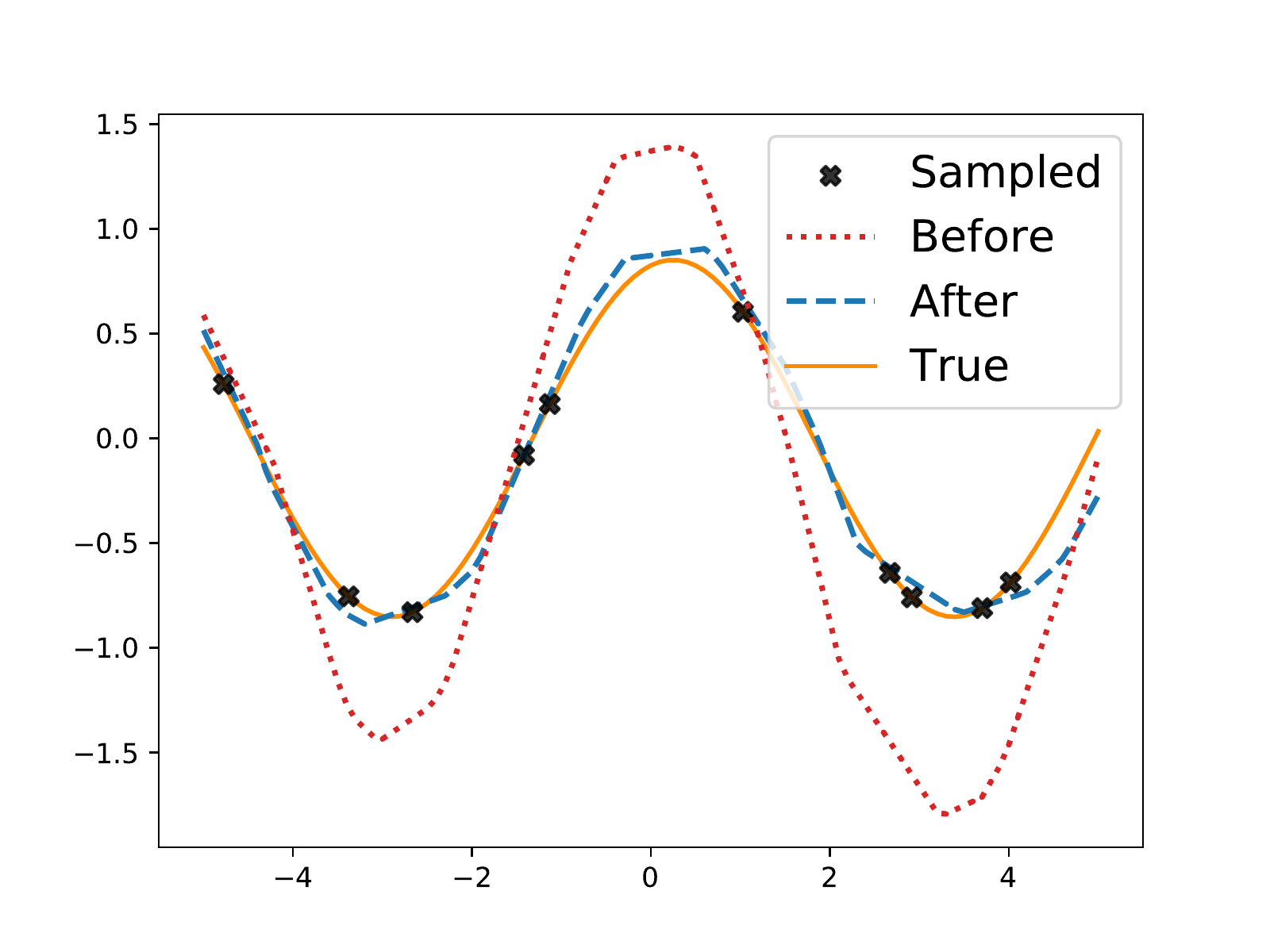}}
        \subfigure[ Meta-RKHS-\ROM{2}]{\includegraphics[width=0.27\linewidth]{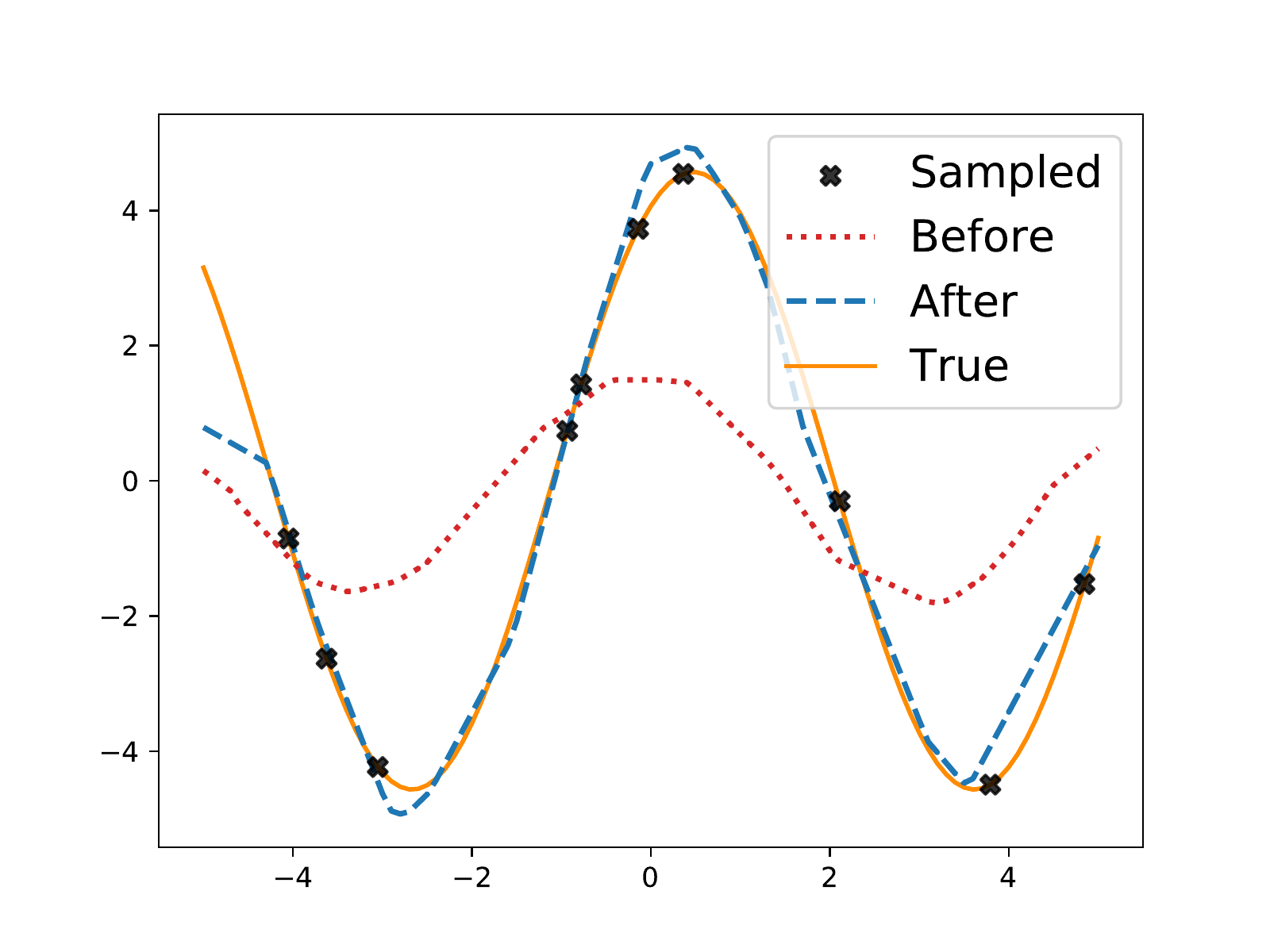}}
        \vspace{-0.1in}
     \caption{Performance of random initialized network and our methods. The models before/after adaptation are shown in dotted/dashed lines, samples used for adaptation are also shown in the figure.}
    \label{fig:regression}
    \vspace{-0.2in}
\end{figure}

% \begin{table}[ht!]
%     \caption{Few-shot classification results on Mini-ImageNet and FC-100 \cy{Use wraptable to wrap the text arond table.}}
%     \vspace{-0.1in}
%     \label{tab:results_mini_imagenet_fc_100}
%     \begin{center}
%     \begin{sc}
%     \begin{adjustbox}{scale=0.7}
%         \begin{tabular}{lrrrrrr}
%             \toprule
%             &\multicolumn{2}{c}{Mini-ImageNet} & \multicolumn{2}{c}{FC-100}\\
%             Algorithm& 5 Way 1 Shot & 5 Way 5 Shots&5 Way 1 Shot & 5 Way 5 Shots \\
%             \midrule
%              MAML& $48.70 \pm 1.84 \%$ & $63.11 \pm 0.93 \%$ & $38.00 \pm 1.95 \%$ & $49.34 \pm 0.97 \%$&\\
%              FOMAML& $48.07 \pm 1.75 \%$ & $63.15 \pm 0.91 \%$& $37.73 \pm 1.93 \%$ & $49.05 \pm 0.99 \%$&\\
%              iMAML& $49.30 \pm 1.88 \%$ & $64.89 \pm 0.95 \%$& $38.38 \pm 1.70 \%$ & $49.41 \pm 0.80 \%$&\\
%             %  Ours &$48.23 \pm 1.96 \%$ & $65.35 \pm 1.02 \%$ & $38.92 \pm 1.65 \%$ & $50.22 \pm 1.82 \%$\\
%             %  Ours-GR &$48.56 \pm 1.89 \%$ & & $\mathbf{39.23 \pm 1.72 \%}$ \\
%              Reptile & $49.70 \pm 1.83 \%$ & $65.91 \pm 0.84 \%$ & $38.40 \pm 1.94 \%$ & $50.50 \pm 0.87 \%$&\\
%              \midrule
%              Meta-RKHS-\ROM{1} (Ours)& $\mathbf{51.10 \pm 1.82} \%$ & $\mathbf{66.19 \pm 0.80 \%}$ & $38.90 \pm 1.90 \%$ & $\mathbf{51.47 \pm 0.86\%}$&\\
%              Meta-RKHS-\ROM{2} (Ours)& $50.40 \pm 2.86\%$& $65.40\pm 0.91 \%$& $\mathbf{41.20 \pm 2.83\%}$& $51.36 \pm 0.96$\\
%              \bottomrule
%         \end{tabular}
%     \end{adjustbox}
%     \end{sc}
%     \end{center}
% \end{table}

\vspace{-0.1cm}
\subsection{Few-shot Image Classification}
\vspace{-0.2cm}
For this experiment, we choose two popular datasets adopted for meta-learning: Mini-ImageNet and FC-100 \citep{TADAM18}. The cross-entropy loss is adopted for Meta-RKHS-\ROM{1}; while the squared loss is used for Meta-RKHS-\ROM{2} following \cite{arora19, CNNGP19}. 
Similar to \cite{finn17a}, the model architecture is set to be a four-layer convolutional neural network with ReLU activation. The filter number is set to be 32. The Adam optimizer \citep {adam15} is used to minimize the energy functional. Meta batch size is set to be 16 and learning rates are set to be 0.01 for Meta-RKHS-\ROM{2}. 

\begin{wraptable}{R}{0.6\linewidth}
    \vspace{-0.2in}
    \caption{Few-shot classification results on Mini-ImageNet and FC-100.}
    \vspace{-0.05in}
    \label{tab:results_mini_imagenet_fc_100}
    \begin{center}
    \begin{sc}
    \begin{adjustbox}{scale=0.57}
        \begin{tabular}{lrrrrrr}
            \toprule
            &\multicolumn{2}{c}{Mini-ImageNet} & \multicolumn{2}{c}{FC-100}\\
            Algorithm& 5 Way 1 Shot & 5 Way 5 Shots&5 Way 1 Shot & 5 Way 5 Shots \\
            \midrule
             MAML& $48.70 \pm 1.84 \%$ & $63.11 \pm 0.93 \%$ & $38.00 \pm 1.95 \%$ & $49.34 \pm 0.97 \%$&\\
             FOMAML& $48.07 \pm 1.75 \%$ & $63.15 \pm 0.91 \%$& $37.73 \pm 1.93 \%$ & $49.05 \pm 0.99 \%$&\\
             iMAML& $49.30 \pm 1.88 \%$ & $64.89 \pm 0.95 \%$& $38.38 \pm 1.70 \%$ & $49.41 \pm 0.80 \%$&\\
            %  Ours &$48.23 \pm 1.96 \%$ & $65.35 \pm 1.02 \%$ & $38.92 \pm 1.65 \%$ & $50.22 \pm 1.82 \%$\\
            %  Ours-GR &$48.56 \pm 1.89 \%$ & & $\mathbf{39.23 \pm 1.72 \%}$ \\
             Reptile & $49.70 \pm 1.83 \%$ & $65.91 \pm 0.84 \%$ & $38.40 \pm 1.94 \%$ & $50.50 \pm 0.87 \%$&\\
             \midrule
             Meta-RKHS-\ROM{1}& $\mathbf{51.10 \pm 1.82} \%$ & $\mathbf{66.19 \pm 0.80 \%}$ & $38.90 \pm 1.90 \%$ & $\mathbf{51.47 \pm 0.86\%}$&\\
             Meta-RKHS-\ROM{2}& $50.53 \pm 2.09\%$& $65.40\pm 0.91 \%$& $\mathbf{41.20 \pm 2.17\%}$& $51.36 \pm 0.96$\\
             \bottomrule
        \end{tabular}
    \end{adjustbox}
    \end{sc}
    \end{center}
    \vspace{-0.1in}
\end{wraptable}
The results are shown in Table \ref{tab:results_mini_imagenet_fc_100}. Note the results of Reptile is different from those in \cite{reptile18}, because we re-evaluate it under the same setting as \cite{finn17a}, {\it i.e.}, 10 steps of adaptation is applied during testing. Our results of iMAML is based on the implementation of \cite{spigler2019}. It is observed that our proposed methods achieve better accuracy than different baselines. Interestingly, our Meta-RKHS-\ROM{1} performs better than FOMAML (this is also the case in other experiments), although they share a similar objective. We conjecture the reason is because our Meta-RKHS-\ROM{1} restricts the function to be in an RKHS, making the functional space smaller thus easier to optimize compared to the unrestricted version of FOMAML. In terms of our two algorithms, there is not always a winner on all the tasks. We note that Meta-RKHS-\ROM{1} 
is more efficient in training. However, we show below that Meta-RKHS-\ROM{2} is better in terms of robustness to adversarial attacks and out-of-distribution generalization.

\vspace{-0.1cm}
\subsection{Robustness to Adversarial Attacks}
\vspace{-0.2cm}
We now compare the adversarial robustness of our methods with other popular baselines. 
% For a classifier $f(\bm{x})$ which correctly predicts the label of $\bm{x}$,  an adversarial example $\bm{x}_{adv}$ is the one with $\mathcal{D}(\bm{x}, x_{adv}) <\epsilon$ (where $\mathcal{D}$ represents the distance between two input images, e.g., $\ell_{1}, \ell_{2}, \ell_{\infty}$ distance), but $f(\bm{x}) \neq f(\bm{x}_{adv})$. 
We adopt both white-box and black-box attacks in this experiment. For the white-box attacks, we adopt strong attacks including the PGD Attack \citep{Madry2017TowardsDL}, \new{BPDA attack \citep{athalye2018obfuscated} and SPSA attack \citep{uesato2018adversarial}}. For PGD attack, we use $\ell_{\infty}$ norm and compare the results on Mini-imagenet and FC-100. We compare the robust accuracy with different magnitude with 20-step attack with a step size of $2/255$. \new{For BPDA attack, we apply median smoothing, JPEGFilter and BitSqueezing as input transformation adapted from \citep{guo2018countering} as defense strategies. For SPSA attack, we follow \citep{uesato2018adversarial} and set the Adam learning rate 0.01, perturbation size $\delta$ = 0.01}.  For Black-box attack, we adopt the strong query efficient attack method \citep{blackbox19}. Follow the setting of \cite{blackbox19}, we use a fixed step size of 0.2.

%Due to space limitation, we show some results in Figure \ref{fig:blackbox_att} and \ref{fig:whitebox_att}, leaving some other results in the Appendix. 
We consider both finite-time and infinite-time adaptation in this experiment. For finite-time adaptation, the Pad\'{e} approximation with $P=Q=1$ and $P=Q=2$ to approximate the matrix exponential are considered \citep{ButcherC:1992}. We use Meta-RKHS-II$\_$t100$\_$PQ1 and Meta-RKHS-II$\_$t100$\_$PQ2 to denote methods using finite time $t=100$, $P=Q=1$ or $P=Q=2$, respectively. We observe other finite time $t$ makes similar predictions, thus we only consider $t=100$. The results from the black-box attack in Figure~\ref{fig:blackbox_att} indicate the robustness of our Meta-RKHS-\ROM{2}. In fact, the gaps are significantly large, making it the only useful robust model in the adversarial-attack setting. Our Meta-RKHS-\ROM{1} is not as robust as Meta-RKHS-\ROM{2}, but still slightly outperforms other baselines. \new{Regarding the white-box attack, results in Figure~\ref{fig:bpda_att}, \ref{fig:spsa_att} and \ref{fig:whitebox_att} again show that our proposed Meta-RKHS-\ROM{2} is significantly more robust than baselines under the three strong attacks.} It is also interesting to see that our Meta-RKHS-\ROM{1} performs slightly better than Meta-RKHS-\ROM{2} in some rare cases, {\it e.g.}, in the Mini-ImageNet 5-way 1-shot case when the attack magnitude is not too small. More results are presented in the Appendix.

% \begin{figure}[htb!]
%      \centering
%      \begin{minipage}{0.33\textwidth}
%          \centering
%          \includegraphics[width=\textwidth]{imagenet_shots5.png}
%      \end{minipage}
%      \hfill
%      \begin{minipage}{0.32\textwidth}
%          \centering
%          \includegraphics[width=\textwidth]{imagenet_shots1.png}
%      \end{minipage}
%      \hfill
%       \begin{minipage}{0.33\textwidth}
%          \centering
%          \includegraphics[width=\textwidth]{cifar100_shots5.png}
%      \end{minipage}
     
%      \caption{Miniimagenet 5-way 5-shot (left) and 5-way 1-shot (right) PGD infinity norm attacks.}
% \end{figure}

% \begin{figure}[t!]
% \includegraphics[width=0.46\linewidth]{figures/BPDA_cifar.png}
% \centering
% \vspace{-0.1in}
% \caption{BPDA attack on FC-100 5-way 5-shot. }\label{fig:blackbox_att}
% \vspace{-0.1in}
% \end{figure}

% \begin{figure}[t!]
% \includegraphics[width=0.46\linewidth]{figures/SPSA_cifar.png}
% \centering
% \vspace{-0.1in}
% \caption{SPSA attack on FC-100 5-way 5-shot. }\label{fig:blackbox_att}
% \vspace{-0.1in}
% \end{figure}

\begin{figure}[t!]
\includegraphics[width=0.85\linewidth]{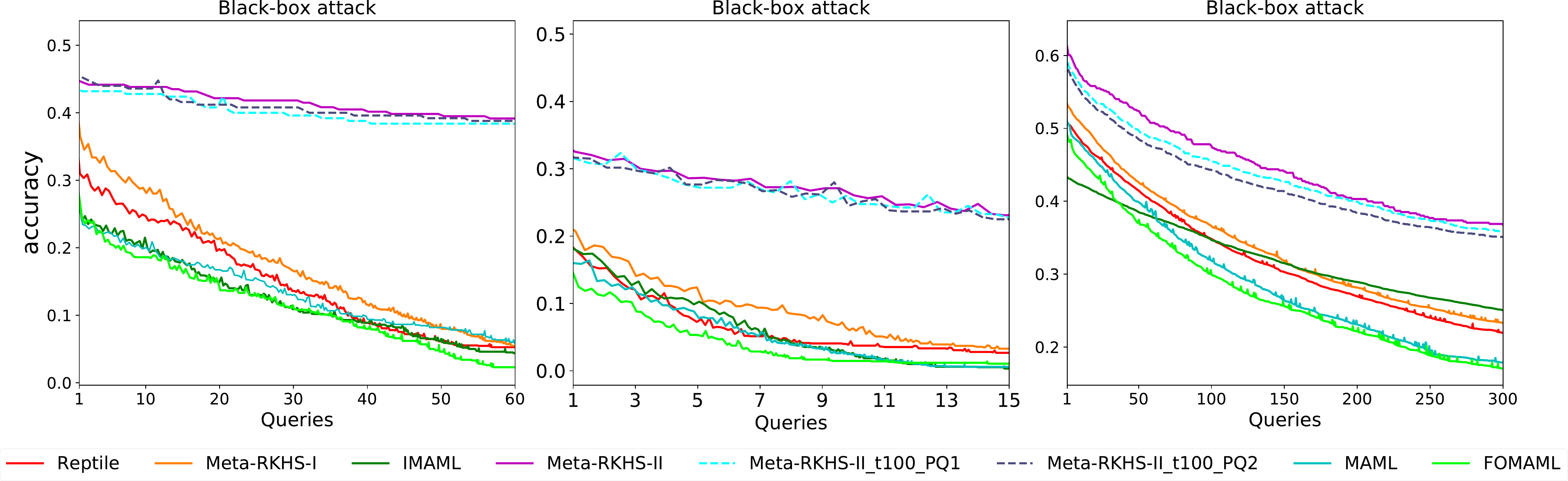}
\centering
\vspace{-0.1in}
\caption{Black-box attack on Mini-ImageNet and FC-100. Mini-ImageNet 5-way 1-shot (left), FC-100 5-way 1-shot (middle) and Mini-ImageNet 5-way 5-shot (right). }\label{fig:blackbox_att}
\vspace{-0.1in}
\end{figure}

\begin{figure}[t!]
\includegraphics[width=0.4\linewidth]{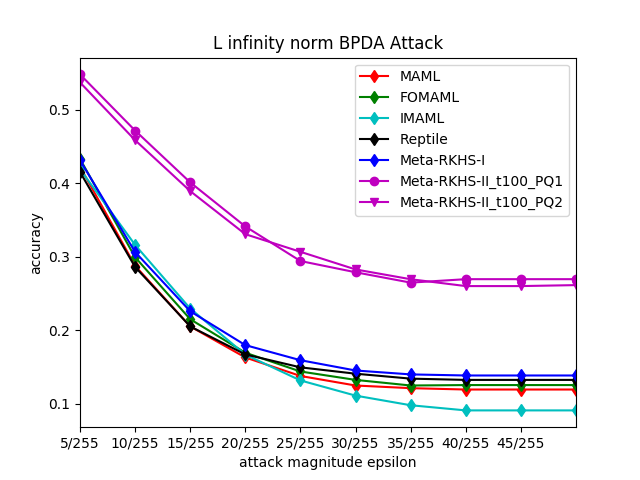}
\includegraphics[width=0.4\linewidth]{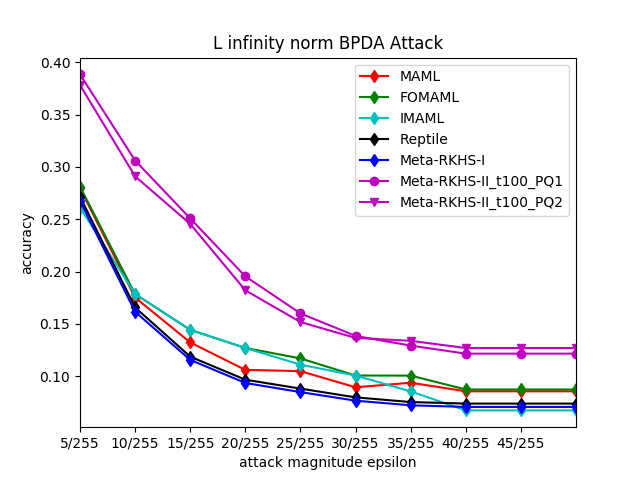}
\centering
\vspace{-0.1in}
\caption{BPDA attack on Mini-ImageNet 5-way 5-shot (left) and FC-100 5-way 5-shot (right). }\label{fig:bpda_att}
\vspace{-0.1in}
\end{figure}

\begin{figure}[t!]
\includegraphics[width=0.4\linewidth]{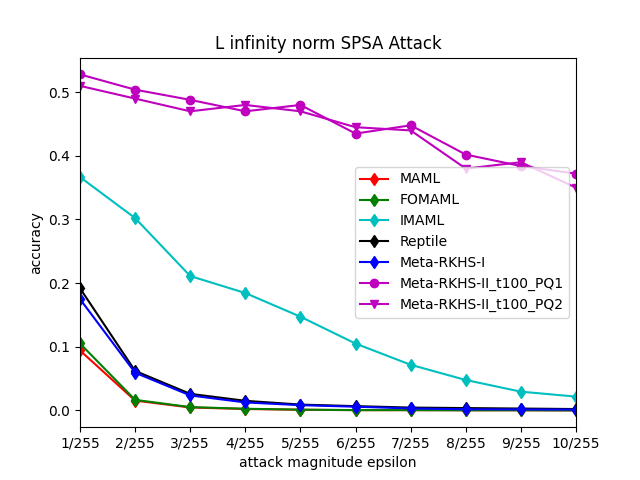}
\includegraphics[width=0.4\linewidth]{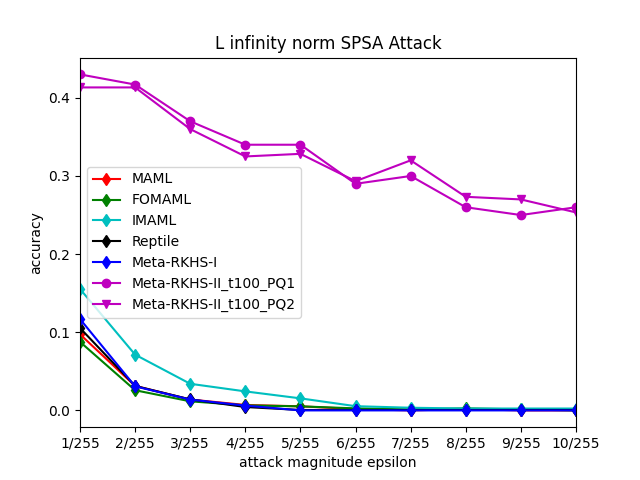}
\centering
\vspace{-0.1in}
\caption{SPSA attack on Mini-ImageNet 5-way 5-shot (left) and FC-100 5-way 5-shot (right). }\label{fig:spsa_att}
\vspace{-0.1in}
\end{figure}

\begin{figure}[t!]
\includegraphics[width=\linewidth]{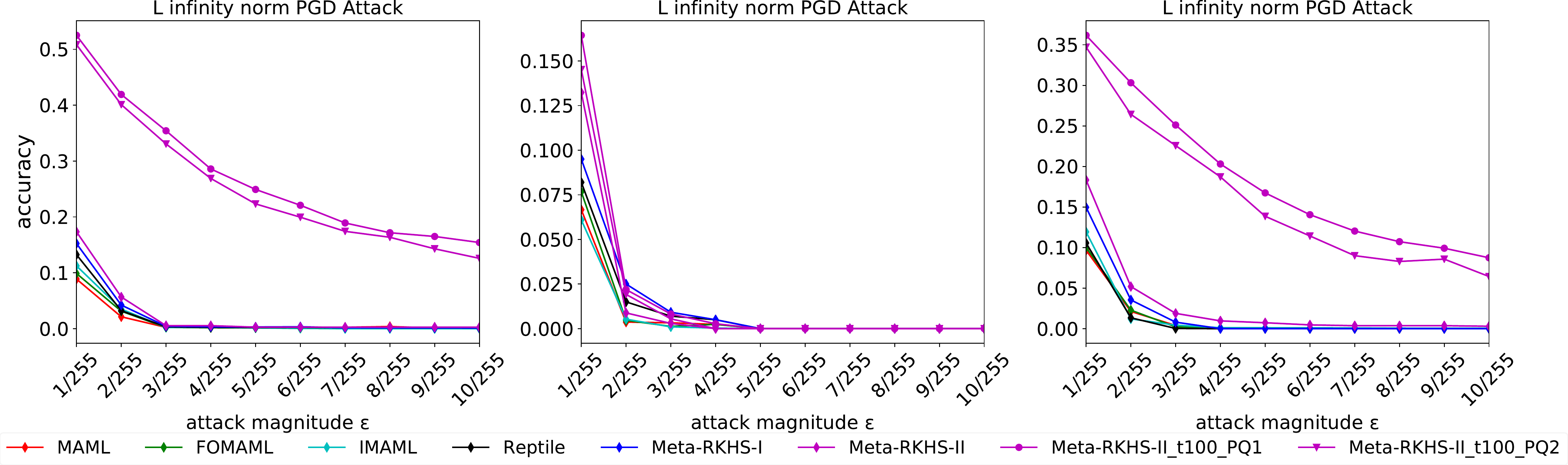}
\centering
\vspace{-0.1in}
\caption{$ \ell_{\infty}$ norm PGD attack on Mini-ImageNet and FC-100. Mini-ImageNet 5-way 5-shot (left), Mini-ImageNet 5-way 1-shot (middle) and FC-100 5-way 5-shot (right). }\label{fig:whitebox_att}
\vspace{-0.1in}
\end{figure}

% \begin{figure}[htb!]
%      \centering
%      \begin{minipage}{0.32\textwidth}
%          \centering
%          \includegraphics[width=\textwidth]{imagenet_black_box_shots5.png}
%      \end{minipage}
%      \hfill
%      \begin{minipage}{0.32\textwidth}
%          \centering
%          \includegraphics[width=\textwidth]{imagenet_black_box_shots1.png}
%      \end{minipage}
%      \hfill
%           \begin{minipage}{0.33\textwidth}
%          \centering
%          \includegraphics[width=\textwidth]{cifar100_black_box_shots5.png}
%      \end{minipage}
%      \caption{Miniimagenet 5-way 5-shot (left) and 5-way 1-shot (right) Black-box attacks.}
%     %  \label{blackbox:imagenet}
% \end{figure}

\subsection{Out-of-distribution Generalization}
We adopt similar strategy in \citep{balance20} to test a model's ability of generalizing to out-of-distribution datasets. \new{In this setting, the state of arts are achieved by Bayesian TAML \citep{balance20}. Different from their setting that considers any-shot learning with maximum number of examples for each class being as large as 50, we only focus on the standard 1 or 5 shot learning. We thus modify their code to accommodate our standard setting.} The CUB \citep{CUB11} and VGG Flower \cite{vggflower08} are fine-grained datasets used in this experiment, where all images are resized to $84 \times 84$.  We follow \cite{balance20} to split these datasets into meta training/validation/testing sets. We first train all the methods on Mini-ImageNet or FC-100 datasets, then conduct meta-testing on CUB and VGG Flower datasets. The results are shown in Table \ref{fig:OOD-Mini-ImageNet}. Again, our methods achieve the best results, \new{outperforming the state-of-art method with our Meta-RKHS-\ROM{2}}, indicating the robustness of our proposed methods. More results are presented in the Appendix.

\begin{wraptable}{R}{0.6\linewidth}
    \vspace{-1.5cm}
    \caption{Meta testing on different out-of-distribution datasets with model trained on Mini-ImageNet.}
    \vspace{-0.1in}
    \label{fig:OOD-Mini-ImageNet}
    \begin{center}
    \begin{sc}
    \begin{adjustbox}{scale=0.57}
        \begin{tabular}{lrrrrrrrrrr}
            \toprule
        & \multicolumn{2}{c}{5 way 1 shot}& &\multicolumn{2}{c}{5 way 5 shot}&\\
    %   \cmidrule{2-3}  \cmidrule{5-7}\\
            Algorithm & CUB & VGG Flower & & CUB & VGG Flower &\\
            \midrule
         MAML &$34.23\pm 1.52 \%$ &$52.98\pm 1.76 \%$&& $52.36 \pm 0.94 \%$&  $67.52 \pm 1.30 \%$ &\\
         FOMAML& $35.32\pm 1.69 \%$ & $53.86\pm 1.64 \%$ && $52.02\pm 0.71 \%$ & $68.83\pm 1.16 \%$\\
         Reptile& $35.61\pm 1.38 \%$ & $53.57\pm 1.58 \%$ &&$51.93\pm 0.89 \%$&$71.62\pm 1.25 \%$& \\
         iMAML& $40.55 \pm 0.61 \%$ & $54.97 \pm 0.80 \%$ & &$46.31 \pm 2.03 \%$ &$60.67 \pm 1.91 \%$\\
         Bayesian TAML(SOTA)&$41.57\pm 0.60\%$ &$58.56\pm 0.66 \%$&&$61.78 \pm 0.56\%$ & $77.95\pm0.46 \%$\\
         \midrule
         Meta-RKHS-\ROM{1}&$36.73\pm 1.26 \%$&  $54.79\pm 1.61 \%$ & &$54.19\pm 0.73 \%$ & $72.76\pm 1.08 \%$\\
         Meta-RKHS-\ROM{2}&$\mathbf{45.36\pm 0.87 \%}$&$\mathbf{60.80\pm 1.02 \%}$ && $\mathbf{65.21 \pm 0.64 \%}$& $\mathbf{78.25 \pm 0.49 \%}$ & \\
        \bottomrule
        \end{tabular}
    \end{adjustbox}
    \end{sc}
    \end{center}
    \vspace{-0.1in}
\end{wraptable}

\subsection{Ablation Study}
% \begin{wraptable}{R}{0.6\linewidth}
%     \vspace{-0.52in}
%     \caption{Meta testing on different out-of-distribution datasets with model trained on Mini-ImageNet.}
%     \vspace{-0.1in}
%     \label{fig:OOD-Mini-ImageNet}
%     \begin{center}
%     \begin{sc}
%     \begin{adjustbox}{scale=0.58}
%         \begin{tabular}{lrrrrrrrrrr}
%             \toprule
%         & \multicolumn{2}{c}{5 way 1 shot}& &\multicolumn{2}{c}{5 way 5 shot}&\\
%     %   \cmidrule{2-3}  \cmidrule{5-7}\\
%             Algorithm & CUB & VGG Flower & & CUB & VGG Flower &\\
%             \midrule
%          MAML &$34.23\pm 1.52 \%$ &$52.98\pm 1.76 \%$&& $52.36 \pm 0.94 \%$&  $67.52 \pm 1.30 \%$ &\\
%          FOMAML& $35.32\pm 1.69 \%$ & $53.86\pm 1.64 \%$ && $52.02\pm 0.71 \%$ & $68.83\pm 1.16 \%$\\
%          Reptile& $35.61\pm 1.38 \%$ & $53.57\pm 1.58 \%$ &&$51.93\pm 0.89 \%$&$71.62\pm 1.25 \%$& \\
%          iMAML& $40.55 \pm 1.61 \%$ & $54.97 \pm 1.80 \%$ & &$49.31 \pm 1.03 \%$ &$64.67 \pm 1.41 \%$\\
%          \midrule
%          Meta-RKHS-\ROM{1}&$37.85\pm 1.26 \%$&  $54.79\pm 1.61 \%$ & &$54.19\pm 0.73 \%$ & $72.76\pm 1.08 \%$\\
%          Meta-RKHS-\ROM{2}&$\mathbf{45.36\pm 0.87 \%}$&$\mathbf{60.80\pm 1.02 \%}$ && $\mathbf{65.21 \pm 0.64 \%}$& $\mathbf{78.25 \pm 0.49 \%}$ & \\
%         \bottomrule
%         \end{tabular}
%     \end{adjustbox}
%     \end{sc}
%     \end{center}
%     \vspace{-0.2in}
% \end{wraptable}

\begin{table}[t!]
    \caption{Meta-RKHS-\ROM{2} with different time $t$.}
    % \vspace{-0.1in}
    \label{tab:different_t}
    \begin{center}
    \begin{sc}
    \begin{adjustbox}{scale=0.7}
        \begin{tabular}{lrrrrrr}
            \toprule
            &Time $t$ & $t=0.1$ & $t=1$ & $t=10$ & $t=100$ &$t=\infty$ \\
            \midrule
            \multirow{2}{*}{Mini-ImageNet} & 5 Way 1 Shot& $49.67 \pm 2.23 \%$ & $48.27 \pm 2.23 \%$ & $\mathbf{50.53 \pm 2.09} \%$ & $49.13 \pm 2.19 \%$ & $48.70 \pm 2.28 \%$ \\
             &5 Way 5 Shots& $64.51 \pm 0.93 \%$ & $64.28\pm 0.98 \%$ & $\mathbf{65.40\pm 0.91 \%}$ & $64.24 \pm 1.06 \%$ & $64.95\pm 0.96 \%$ \\
            \midrule
            \multirow{2}{*}{FC-100}   &  5 Way 1 Shot& $36.50 \pm 2.10 \%$ & $38.80 \pm 2.32 \%$ & $\mathbf{41.20 \pm 2.17 \%}$ & $38.80 \pm 2.21 \%$ & $37.60 \pm 2.13 \%$ \\
            &   5 Way 5 Shots& $48.35 \pm 1.02 \%$ & $49.79 \pm 1.04 \%$ & $\mathbf{51.36 \pm 0.96} \%$ & $48.59 \pm 1.09 \%$ & $49.48 \pm 0.98 \%$ \\
            \bottomrule
        \end{tabular}
    \end{adjustbox}
    \end{sc}
    \end{center}
    \vspace{-0.3in}
\end{table}
We conduct several ablation studies, including: comparing Reptile with Meta-RKHS-\ROM{1} under different adaptation steps (results shown in the Appendix),  
testing the impact of choosing different time $t$ in Meta-RKHS-\ROM{2} (results shown in Table \ref{tab:different_t}) and the impact of network architecture with different number of CNN feature channels (results shown in the Appendix). It is interesting to see that a finite-time (around $t = 10$) achieves the best accuracy, although the infinite-time case guarantees a stationary point. This indicates that a stationary point achieved by limited training data in the adaptation step is not always the best choice, because the limited training data might easily overfit the model, thus achieving worse test results.

\vspace{-0.1cm}
\section{Conclusion}
\vspace{-0.2cm}
We develop meta-learning in RKHS, and propose two practical algorithms allowing efficient adaptation in the function space by avoiding some complicated adaptations as in traditional methods.  
We show connections between our proposed methods and existing ones. Extensive experiments suggest that our methods are more effective, achieve better generalization and are 
%effective and 
more robust against adversarial attacks and out-of-distribution generalization, \new{compared to popular strong baselines.} 
% \section*{Acknowledgements}

%by experimental results.
\clearpage
% \section*{Broader Impact}
% In this paper, we re-analyze the meta-learning problem, and propose two well-motivated algorithms to solve meta learning in an RKHS. Our proposed methods obtain better results in different experiments, and are more robust in terms of out-of-distribution generalization and defending adversarial attacks. Adversarial attacks is a serious issue in real-world applications. Our proposed method may inspire related research topics that would have positive social impact. Furthermore, our methods are shown to be highly robust in few-shot learning, indicating it might have less over-fitting issue than other models. One potential topic we would like to investigate is that whether our methods could decrease the biases contained in datasets.
\bibliography{reference,ref}

\begin{thebibliography}{41}
\providecommand{\natexlab}[1]{#1}
\providecommand{\url}[1]{\texttt{#1}}
\expandafter\ifx\csname urlstyle\endcsname\relax
  \providecommand{\doi}[1]{doi: #1}\else
  \providecommand{\doi}{doi: \begingroup \urlstyle{rm}\Url}\fi

\bibitem[Allen-Zhu et~al.(2019)Allen-Zhu, Li, and Song]{pmlr-v97-allen-zhu19a}
Zeyuan Allen-Zhu, Yuanzhi Li, and Zhao Song.
\newblock A convergence theory for deep learning via over-parameterization.
\newblock volume~97 of \emph{Proceedings of Machine Learning Research}, pp.\
  242--252, Long Beach, California, USA, 09--15 Jun 2019. PMLR.
\newblock URL \url{http://proceedings.mlr.press/v97/allen-zhu19a.html}.

\bibitem[Andrychowicz et~al.(2016)Andrychowicz, Denil, G\'{o}mez, Hoffman,
  Pfau, Schaul, Shillingford, and de~Freitas]{learn2learn2016}
Marcin Andrychowicz, Misha Denil, Sergio G\'{o}mez, Matthew~W Hoffman, David
  Pfau, Tom Schaul, Brendan Shillingford, and Nando de~Freitas.
\newblock Learning to learn by gradient descent by gradient descent.
\newblock In D.~D. Lee, M.~Sugiyama, U.~V. Luxburg, I.~Guyon, and R.~Garnett
  (eds.), \emph{Advances in Neural Information Processing Systems}, pp.\
  3981--3989. 2016.

\bibitem[Arora et~al.(2019)Arora, Du, Hu, Li, Salakhutdinov, and Wang]{arora19}
Sanjeev Arora, Simon~S. Du, Wei Hu, Zhiyuan Li, Ruslan Salakhutdinov, and
  Ruosong Wang.
\newblock On exact computation with an infinitely wide neural net.
\newblock In \emph{Advances in Neural Information Processing Systems}, 2019.

\bibitem[Athalye et~al.(2018)Athalye, Carlini, and
  Wagner]{athalye2018obfuscated}
Anish Athalye, Nicholas Carlini, and David Wagner.
\newblock Obfuscated gradients give a false sense of security: Circumventing
  defenses to adversarial examples.
\newblock In \emph{International Conference on Machine Learning}, 2018.

\bibitem[Butcher \& Chipman(1992)Butcher and Chipman]{ButcherC:1992}
J.~C. Butcher and F.~H. Chipman.
\newblock Generalized pad\'{e} approximations to the exponential function.
\newblock \emph{BIT Numerical Mathematics}, 32:\penalty0 118--130, 1992.

\bibitem[Denevi et~al.(2019)Denevi, Ciliberto, Grazzi, and
  Pontil]{learn2learn19}
Giulia Denevi, Carlo Ciliberto, Riccardo Grazzi, and Massimiliano Pontil.
\newblock Learning-to-learn stochastic gradient descent with biased
  regularization.
\newblock In \emph{https://arxiv.org/abs/1903.10399}, 2019.

\bibitem[Fallah et~al.(2020)Fallah, Mokhtari, and Ozdaglar]{converge20}
Alireza Fallah, Aryan Mokhtari, and Asuman Ozdaglar.
\newblock On the convergence theory of gradient-based model-agnostic
  meta-learning algorithms.
\newblock In \emph{International Conference on Artificial Intelligence and
  Statistics}, 2020.

\bibitem[Finn \& Levine(2019)Finn and Levine]{finn2019}
Chelsea Finn and Sergey Levine.
\newblock Meta-learning: from few-shot learning to rapid reinforcement
  learning.
\newblock In \emph{ICML 2019 Meta-Learning Tutorial}, 2019.

\bibitem[Finn et~al.(2017)Finn, Abbeel, and Levine]{finn17a}
Chelsea Finn, Pieter Abbeel, and Sergey Levine.
\newblock Model-agnostic meta-learning for fast adaptation of deep networks.
\newblock In \emph{International Conference on Machine Learning}, 2017.

\bibitem[Finn et~al.(2018)Finn, Xu, and Levine]{PMAML2018}
Chelsea Finn, Kelvin Xu, and Sergey Levine.
\newblock Probabilistic model-agnostic meta-learning.
\newblock In \emph{Advances in Neural Information Processing Systems}. 2018.

\bibitem[Fort et~al.(2020)Fort, Dziugaite, Paul, Kharaghani, Roy, and
  Ganguli]{Fortetal_nips20}
Stanislav Fort, Gintare~Karolina Dziugaite, Mansheej Paul, Sepideh Kharaghani,
  Daniel~M. Roy, and Surya Ganguli.
\newblock Deep learning versus kernel learning: an empirical study of loss
  landscape geometry and the time evolution of the neural tangent kernel.
\newblock In \emph{Advances in Neural Information Processing Systems}, 2020.

\bibitem[Graves et~al.(2014)Graves, Wayne, and Danihelka]{NTM2014}
Alex Graves, Greg Wayne, and Ivo Danihelka.
\newblock Neural turing machines.
\newblock In \emph{https://arxiv.org/abs/1410.5401}, 2014.

\bibitem[Guo et~al.(2018)Guo, Rana, Cisse, and van~der
  Maaten]{guo2018countering}
Chuan Guo, Mayank Rana, Moustapha Cisse, and Laurens van~der Maaten.
\newblock Countering adversarial images using input transformations.
\newblock In \emph{International Conference on Learning Representations}, 2018.

\bibitem[Guo et~al.(2019)Guo, Gardner, You, Wilson, and Weinberger]{blackbox19}
Chuan Guo, Jacob~R. Gardner, Yurong You, Andrew~Gordon Wilson, and Kilian~Q.
  Weinberger.
\newblock Simple black-box adversarial attacks.
\newblock In \emph{International Conference on Machine Learning}, 2019.

\bibitem[Jacot et~al.(2018)Jacot, Gabriel, and Hongler]{jacot2018neural}
Arthur Jacot, Franck Gabriel, and Cl{\'e}ment Hongler.
\newblock Neural tangent kernel: Convergence and generalization in neural
  networks.
\newblock In \emph{Advances in neural information processing systems}, pp.\
  8571--8580, 2018.

\bibitem[Khodak et~al.(2019{\natexlab{a}})Khodak, Balcan, and
  Talwalkar]{KhodakBT19}
Mikhail Khodak, Maria-Florina Balcan, and Ameet Talwalkar.
\newblock Provable guarantees for gradient-based meta-learning.
\newblock In \emph{International Conference on Machine Learning},
  2019{\natexlab{a}}.

\bibitem[Khodak et~al.(2019{\natexlab{b}})Khodak, Balcan, and
  Talwalkar]{adaptive19}
Mikhail Khodak, Maria-Florina Balcan, and Ameet Talwalkar.
\newblock Adaptive gradient-based meta-learning methods.
\newblock In \emph{Advances in Neural Information Processing Systems},
  2019{\natexlab{b}}.

\bibitem[Kingma \& Ba(2015)Kingma and Ba]{adam15}
Diederik~P. Kingma and Jimmy Ba.
\newblock Adam: A method for stochastic optimization.
\newblock In \emph{International Conference on Learning Representations}, 2015.

\bibitem[Lee et~al.(2020)Lee, Lee, Na, Kim, Park, Yang, and Hwang]{balance20}
Hae~Beom Lee, Hayeon Lee, Donghyun Na, Saehoon Kim, Minseop Park, Eunho Yang,
  and Sung~Ju Hwang.
\newblock Learning to balance: Bayesian meta-learning for imbalanced and
  out-of-distribution tasks.
\newblock In \emph{International Conference on Learning Representations}, 2020.

\bibitem[Lee et~al.(2019)Lee, Xiao, Schoenholz, Bahri, Novak, Sohl-Dickstein,
  and Pennington]{NIPS2019_9063}
Jaehoon Lee, Lechao Xiao, Samuel Schoenholz, Yasaman Bahri, Roman Novak, Jascha
  Sohl-Dickstein, and Jeffrey Pennington.
\newblock Wide neural networks of any depth evolve as linear models under
  gradient descent.
\newblock In \emph{Advances in Neural Information Processing Systems 32}, pp.\
  8572--8583. Curran Associates, Inc., 2019.

\bibitem[Madry et~al.(2017)Madry, Makelov, Schmidt, Tsipras, and
  Vladu]{Madry2017TowardsDL}
Aleksander Madry, Aleksandar Makelov, Ludwig Schmidt, Dimitris Tsipras, and
  Adrian Vladu.
\newblock Towards deep learning models resistant to adversarial attacks.
\newblock \emph{ArXiv}, abs/1706.06083, 2017.

\bibitem[M.Arioli et~al.(1996)M.Arioli, B.Codenotti, and C.Fassino]{pade96}
M.Arioli, B.Codenotti, and C.Fassino.
\newblock {The $Pad\acute{e}$ method for computing the matrix exponential}.
\newblock \emph{Linear Algebra and its Applications}, June 1996.

\bibitem[Mishra et~al.(2018)Mishra, Rohaninejad, Chen, and Abbeel]{SNAILICLR18}
Nikhil Mishra, Mostafa Rohaninejad, Xi~Chen, and Pieter Abbeel.
\newblock A simple neural attentive meta-learner.
\newblock In \emph{International Conference on Learning Representations}, 2018.

\bibitem[Nichol et~al.(2018)Nichol, Achiam, and Schulman]{reptile18}
Alex Nichol, Joshua Achiam, and John Schulman.
\newblock On first-order meta-learning algorithms.
\newblock In \emph{https://arxiv.org/abs/1803.02999}, 2018.

\bibitem[Nilsback \& Zisserman(2008)Nilsback and Zisserman]{vggflower08}
Maria-Elena Nilsback and Andrew Zisserman.
\newblock Automated flower classification over a large number of classes.
\newblock In \emph{Sixth Indian Conference on Computer Vision, Graphics and
  Image Processing}, 2008.

\bibitem[Novak et~al.(2019)Novak, Xiao, Bahri, Lee, Yang, Hron, Abolafia,
  Pennington, and Sohl-dickstein]{CNNGP19}
Roman Novak, Lechao Xiao, Yasaman Bahri, Jaehoon Lee, Greg Yang, Jiri Hron,
  Daniel~A. Abolafia, Jeffrey Pennington, and Jascha Sohl-dickstein.
\newblock Bayesian deep convolutional networks with many channels are gaussian
  processes.
\newblock In \emph{International Conference on Learning Representations}, 2019.

\bibitem[Oreshkin et~al.(2018)Oreshkin, Rodriguez, and Lacoste]{TADAM18}
Boris~N. Oreshkin, Pau Rodriguez, and Alexandre Lacoste.
\newblock Tadam: Task dependent adaptive metric for improved few-shot learning.
\newblock In \emph{Advances in Neural Information Processing Systems}, 2018.

\bibitem[Rajeswaran et~al.(2019)Rajeswaran, Finn, Kakade, and
  Levine]{IMAML2019}
Aravind Rajeswaran, Chelsea Finn, Sham Kakade, and Sergey Levine.
\newblock Meta-learning with implicit gradients.
\newblock In \emph{Advances in Neural Information Processing Systems}. 2019.

\bibitem[Ravi \& Beatson(2019)Ravi and Beatson]{ABML19}
Sachin Ravi and Alex Beatson.
\newblock Amortized bayesian meta-learning.
\newblock In \emph{International Conference on Learning Representations}, 2019.

\bibitem[Santambrogio(2016)]{santambrogio2016}
Filippo Santambrogio.
\newblock { Euclidean, Metric, and Wasserstein } gradient flows: an overview,
  2016.

\bibitem[Schmidhuber(1987)]{schmidhuber1987}
Jurgen Schmidhuber.
\newblock Evolutionary principles in self-referential learning.
\newblock Diploma thesis, Technische Universitat Munchen, Germany, 14 May 1987.

\bibitem[Snell et~al.(2017)Snell, Swersky, and Zemel]{protonet17}
Jake Snell, Kevin Swersky, and Richard~S. Zemel.
\newblock Prototypical networks for few-shot learning.
\newblock In \emph{Advances in Neural Information Processing Systems}, 2017.

\bibitem[{Spigler}(2019)]{spigler2019}
Giacomo {Spigler}.
\newblock {Meta-learnt priors slow down catastrophic forgetting in neural
  networks}.
\newblock \emph{arXiv e-prints}, art. arXiv:1909.04170, Sep 2019.

\bibitem[Triantafillou et~al.(2020)Triantafillou, Zhu, Dumoulin, Lamblin, Evci,
  Xu, Goroshin, Gelada, Swersky, Manzagol, and Larochelle]{metadata20}
Eleni Triantafillou, Tyler Zhu, Vincent Dumoulin, Pascal Lamblin, Utku Evci,
  Kelvin Xu, Ross Goroshin, Carles Gelada, Kevin Swersky, Pierre-Antoine
  Manzagol, and Hugo Larochelle.
\newblock Meta-dataset: A dataset of datasets for learning to learn from few
  examples.
\newblock In \emph{International Conference on Learning Representations}, 2020.

\bibitem[Tripuraneni et~al.(2020)Tripuraneni, Jin, and Jordan]{provable20}
Nilesh Tripuraneni, Chi Jin, and Michael~I. Jordan.
\newblock Provable meta-learning of linear representations.
\newblock In \emph{https://arxiv.org/abs/2002.11684}, 2020.

\bibitem[Uesato et~al.(2018)Uesato, O'Donoghue, van~den Oord, and
  Kohli]{uesato2018adversarial}
Jonathan Uesato, Brendan O'Donoghue, Aaron van~den Oord, and Pushmeet Kohli.
\newblock Adversarial risk and the dangers of evaluating against weak attacks,
  2018.

\bibitem[Villani(2008)]{OT_book}
C~Villani.
\newblock \emph{Optimal transport -- Old and new}, volume 338, pp.\  xxii+973.
\newblock 01 2008.
\newblock \doi{10.1007/978-3-540-71050-9}.

\bibitem[Vinyals et~al.(2016)Vinyals, Blundell, Lillicrap, Kavukcuoglu, and
  Wierstra]{matching16}
Oriol Vinyals, Charles Blundell, Timothy Lillicrap, Koray Kavukcuoglu, and Daan
  Wierstra.
\newblock Matching networks for one shot learning.
\newblock In \emph{https://arxiv.org/pdf/1606.04080.pdf}, 2016.

\bibitem[Wah et~al.(2011)Wah, Branson, Welinder, Perona, and Belongie]{CUB11}
Catherine Wah, Steve Branson, Peter Welinder, Pietro Perona, and Serge
  Belongie.
\newblock The caltech-ucsd birds-200-2011 dataset.
\newblock In \emph{Technical Report CNS-TR-2011-001, California Institute of
  Technology}, 2011.

\bibitem[Yao et~al.(2019)Yao, Wei, Huang, and Li]{YaoW19}
Huaxiu Yao, Ying Wei, Junzhou Huang, and Zhenhui Li.
\newblock Hierarchically structured meta-learning.
\newblock In \emph{International Conference on Machine Learning}, 2019.

\bibitem[Yoon et~al.(2018)Yoon, Kim, Dia, Kim, Bengio, and Ahn]{BMAML2018}
Jaesik Yoon, Taesup Kim, Ousmane Dia, Sungwoong Kim, Yoshua Bengio, and Sungjin
  Ahn.
\newblock Bayesian model-agnostic meta-learning.
\newblock In \emph{Advances in Neural Information Processing Systems}. 2018.

\end{thebibliography}
\bibliographystyle{iclr2021_conference}

\clearpage
\appendix

\section{Algorithms}\label{app:alg}
Our proposed algorithms for meta-learning in the RKHS are summarized in Algorithm~\ref{algo:mgfl}.

\begin{algorithm}[h!]
	\caption{Meta-Learning in RKHS}\label{algo:mgfl}
	\begin{algorithmic}
	    \REQUIRE $p(\mathcal{T})$: distribution over tasks, randomly initialized neural network parameters $\thetab$.
		\WHILE{not done}
		    \STATE Sample a batch of tasks $\{{\mathcal{T}}_m\}_{m=1}^B \thicksim p(\mathcal{T})$
    		\FORALL {${\mathcal{T}}_m $}
        		\STATE Sample a batch of data points $\mathcal{D}_m$ \OR Sample two batches of data points $\mathcal{D}^{tr}_m$, $\mathcal{D}^{test}_m$.\\
    		\ENDFOR
        	\STATE Evaluate the energy functional by \eqref{eq:energy_functional_1} with $\{\mathcal{D}_m\}_{m=1}^B $ \OR Evaluate the energy functional by \eqref{eq:energy_functional_2} with $\{\mathcal{D}^{tr}_m, \mathcal{D}^{test}_m\}_{m=1}^B $. Minimize the energy functional w.r.t $\thetab$.
		\ENDWHILE
	\end{algorithmic}
\end{algorithm}

\section{Proof of Theorem \ref{thm:equal_gf}}

\textbf{Theorem \ref{thm:equal_gf}} 
\textit{
    If $f_\thetab$ is a neural network with parameter $\thetab \in R^P$ and $\mathcal{H}$ is the Reproducing Kernel Hilbert Space (RKHS) induced by $\Thetab$, where $\Thetab$ is the Neural Tangent Kernel (NTK) of $f_{\thetab}$, 
    then with initialization $f^0 = f_{\thetab^0}$, the gradient flow of $\mathcal{E}(f^t)$ coincides with the function evolution of $f_{\thetab^t}$ induced by the gradient flow of $E(\thetab^t)$. 
}
\begin{proof}
Without loss of generality, we can rewrite $\mathcal{E}(f) = \mathbb{E}_{\mathcal{T}_m}\{\mathbb{E}_{(\xb_m, \yb_m)}\left[ C(f(\xb_m), \yb_m)\right]\}$ with some function $C(\cdot, \cdot)$.

For a neural network $f_{\thetab}$ with parameter $\thetab \in R^P$, the gradient flow of $E$ in $R^P$ is
\[
    \dfrac{\textup{d} \thetab^t}{\textup{d} t} = -\nabla_{\thetab^t} E(\thetab^t).
\]
We have
\begin{align*}
    \dfrac{\textup{d} \thetab^t}{\textup{d} t} &= -\nabla_{\thetab^t} (\mathcal{E} \circ F)(\thetab^t)\\ 
    &= -\mathbb{E}_{\mathcal{T}_m}\{\mathbb{E}_{(\xb_m, \yb_m)}\left[ \nabla_{\thetab^t} C(f_{\thetab^t}(\xb_m), \yb_m)\right]\}\\
    &= -\mathbb{E}_{\mathcal{T}_m}\bigg\{\mathbb{E}_{(\xb_m, \yb_m)}\left[\dfrac{\partial C(f_\thetab^t(\xb_m), \yb_m)}{\partial f_\thetab^t(\xb_m)} \dfrac{\partial f_\thetab^t(\xb_m)}{\partial \thetab^t}\right]\bigg\}.
\end{align*}
We know that the dynamics of $f_{\thetab^t}$ is
\begin{align}\label{eq:proof_equa_2}
    \dfrac{\textup{d} f_{\thetab^t}}{\textup{d} t} 
    &= \dfrac{\textup{d} \thetab^t}{\textup{d} t} \dfrac{\partial f_{\thetab^t}}{\partial \thetab^t} ^\intercal \nonumber\\
    &=-\mathbb{E}_{\mathcal{T}_m}\bigg\{\mathbb{E}_{(\xb_m, \yb_m)}\left[\dfrac{\partial C(f_{\thetab^t}(\xb_m), \yb_m)}{\partial f_{\thetab^t}(\xb_m)} \dfrac{\partial f_{\thetab^t}(\xb_m)}{\partial \thetab^t}\right]\bigg\}\dfrac{\partial f_{\thetab^t}}{\partial \thetab^t} ^\intercal \nonumber\\
    & = -\mathbb{E}_{\mathcal{T}_m}\bigg\{\mathbb{E}_{(\xb_m, \yb_m)}\left[\dfrac{\partial C(f_{\thetab^t}(\xb_m), \yb_m)}{\partial f_{\thetab^t}(\xb_m)} \dfrac{\partial f_{\thetab^t}(\xb)}{\partial \thetab^t} \dfrac{\partial f_{\thetab^t}}{\partial \thetab^t} ^\intercal \right]\bigg\} \nonumber\\
    &=-\mathbb{E}_{\mathcal{T}_m}\bigg\{\mathbb{E}_{(\xb_m, \yb_m)}\left[\dfrac{\partial C(f_{\thetab^t}(\xb_m), \yb_m)}{\partial f_{\thetab^t}(\xb_m)} \Thetab^t(\xb_m, \cdot) \right]\bigg\},
\end{align}
where $\Thetab^t$ is the Neural Tangent Kernel of neural network $f_{\thetab^t}$ \citep{jacot2018neural}.

If $\mathcal{H}^t$ is the Reproducing Kernel Hilbert Space induced by a kernel $\Thetab^t$ and  $V_{\xb_m}: \mathcal{H}\rightarrow R $ is the evaluation functional at $\xb_m$, which is defined as
\[
V_{\xb_{m}}(f) = f(\xb_m), 
\]
then for an arbitrary function $g$ and a small perturbation $\epsilon$, we have 
\begin{align*}
        \langle \nabla_{f} V_{\xb_{m}}(f),  g\rangle &= \lim _{\epsilon \rightarrow 0} \dfrac{V_{\xb_{m}}(f + \epsilon g)-V_{\xb_{m}}(f)}{\epsilon}\\
        \langle \nabla_{f} V_{\xb_{m}}(f),  g\rangle &= \lim _{\epsilon \rightarrow 0} \dfrac{f(\xb_{m}) + \epsilon g(\xb_{m}) - f(\xb_{m})}{\epsilon}\\
        \langle \nabla_{f} V_{\xb_{m}}(f),  g\rangle &=  g(\xb_{m}) \\
        \langle \nabla_{f} V_{\xb_{m}}(f),  g\rangle &= \langle \Thetab^t (\xb_{m}, \cdot),  g\rangle\\
        \nabla_{f} V_{\xb_{m}}(f) &= \Thetab^t (\xb_{m}, \cdot)\\
        \nabla_{f} f(\xb_m) &= \Thetab^t (\xb_{m}, \cdot).
\end{align*}
With an initial function $f^0 = f_{\thetab^0} \in \mathcal{H}$, the gradient flow of $\mathcal{E}$ in $\mathcal{H}$ is 
\[
    \dfrac{\textup{d} f^t}{\textup{d} t} = -\nabla_{f^t} \mathcal{E}(f^t).
\]
We have 
\begin{align}\label{eq:proof_equa_1}
    \dfrac{\textup{d} f^t}{\textup{d} t} 
    &= -\mathbb{E}_{\mathcal{T}_m}\{\mathbb{E}_{(\xb_m, \yb_m)}\left[ \nabla_{f^t} C(f^t(\xb_m), \yb_m)\right]\} \nonumber \\
    &= -\mathbb{E}_{\mathcal{T}_m}\bigg\{\mathbb{E}_{(\xb_m, \yb_m)}\left[\dfrac{\partial C(f^t(\xb_m), \yb_m)}{\partial f^t(\xb_m)} \nabla_{f^t}f^t(\xb_m)\right]\bigg\} \nonumber \\
    &= -\mathbb{E}_{\mathcal{T}_m}\bigg\{\mathbb{E}_{(\xb_m, \yb_m)}\left[\dfrac{\partial C(f^t(\xb_m), \yb_m)}{\partial f^t(\xb_m)}\Thetab^t(\xb_m, \cdot) \right]\bigg\}. 
\end{align}
We can complete the proof by comparing \eqref{eq:proof_equa_2} and \eqref{eq:proof_equa_1}.
\end{proof}

\section{Proof of Theorem \ref{thm:equal_gradient}}
\textbf{Theorem \ref{thm:equal_gradient}} 
\textit{
    If $f_\thetab$ is a neural network with parameter $\thetab$ and $\mathcal{H}$ is the Reproducing Kernel Hilbert Space (RKHS) induced by $\Thetab$, where $\Thetab$ is the Neural Tangent Kernel (NTK) of $f_\thetab$,  then 
    \[
    \mathcal{M}_1 = \widetilde{\mathcal{E}}(\alpha, f_\thetab), \text{ and } \beta_0 = \alpha \Vert \nabla_{\thetab} \mathcal{L}_m(f_\thetab)\Vert ^2 = \alpha \Vert \nabla _{f_\thetab}\mathcal{L}_m(f_\thetab) \Vert^2_{\mathcal{H}}.
    \]
    }
\begin{proof}
Without loss of generality, we rewrite $\mathcal{L}_m(f_\thetab) = \mathbb{E}_{\xb_{m}, \yb_{m}}\left[ C(f_\thetab(\xb_m), \yb_m)\right]$. 

In regression task, we have 
\[C(f_\thetab(\xb_m), \yb_m) = \dfrac{1}{2}\big\Vert f_\thetab(\xb_m)- \yb_m \big\Vert^2\].
In classification task, we have 
\[C(f_\thetab(\xb_m), \yb_m) =  \yb_m \text{log}(f_\thetab(\xb_m))^\intercal,\] 
where $\text{log}$ is element-wise logarithm operation.
\begin{align*}
    &\Vert \nabla_{\thetab} \mathcal{L}_m(f_\thetab)\Vert ^2\\
    &= \nabla_{\thetab} \mathcal{L}_m(f_\thetab) \nabla_{\thetab} \mathcal{L}_m(f_\thetab)^\intercal \\
    &= \nabla_{\thetab} \mathbb{E}_{\xb_{m}, \yb_{m}}\left[ C(f_\thetab(\xb_m), \yb_m)\right]\nabla_{\thetab} \mathbb{E}_{\xb_{m}, \yb_{m}}\left[ C(f_\thetab(\xb_m), \yb_m)\right]^\intercal \\
    & = \mathbb{E}_{\xb_{m}, \yb_{m}}\left[ \dfrac{\partial C(f_\thetab(\xb_m), \yb_m)}{\partial f_{\thetab}(\xb_m)} \dfrac{\partial f_{\thetab}(\xb_m)}{\partial \thetab}\right] \mathbb{E}_{\xb_{m}, \yb_{m}}\left[\dfrac{\partial f_{\thetab}(\xb_m)}{\partial \thetab}^\intercal \dfrac{\partial C(f_\thetab(\xb_m), \yb_m)}{\partial f_{\thetab}(\xb_m)}^\intercal \right]\\
    & = \mathbb{E}_{\xb_{m}, \yb_{m}} \bigg\{\mathbb{E}_{\xb'_{m}, \yb'_{m}} \left [  \dfrac{\partial C(f_\thetab(\xb_m), \yb_m)}{\partial f_{\thetab}(\xb_m)} \dfrac{\partial f_{\thetab}(\xb_m)}{\partial \thetab} \dfrac{\partial f_{\thetab}(\xb'_m)}{\partial \thetab}^\intercal \dfrac{\partial C(f_\thetab(\xb'_m), \yb'_m)}{\partial f_{\thetab}(\xb'_m)}^\intercal\right]\bigg\} \\
    & = \mathbb{E}_{\xb_{m}, \yb_{m}} \bigg\{\mathbb{E}_{\xb'_{m}, \yb'_{m}} \left [  \dfrac{\partial C(f_\thetab(\xb_m), \yb_m)}{\partial f_{\thetab}(\xb_m)}\Thetab (\xb_m, \xb'_m)\dfrac{\partial C(f_\thetab(\xb'_m), \yb'_m)}{\partial f_{\thetab}(\xb'_m)}^\intercal\right]\bigg\} \\
    &= \bigg\langle \mathbb{E}_{\xb_{m}, \yb_{m}}\left [  \dfrac{\partial C(f_\thetab(\xb_m), \yb_m)}{\partial f_{\thetab}(\xb_m)}\Thetab (\xb_m, \cdot)\right],  \mathbb{E}_{\xb'_{m}, \yb'_{m}} \left[\dfrac{\partial C(f_\thetab(\xb'_m), \yb'_m)}{\partial f_{\thetab}(\xb'_m)} \Thetab (\xb'_m, \cdot)\right]\bigg\rangle_{\mathcal{H}}\\
    &=\bigg\langle \mathbb{E}_{\xb_{m}, \yb_{m}}\left [  \dfrac{\partial C(f_\thetab(\xb_m), \yb_m)}{\partial f_{\thetab}(\xb_m)}\nabla_{f_\thetab} f_{\thetab}(\xb_m)\right],  \mathbb{E}_{\xb'_{m}, \yb'_{m}} \left[\dfrac{\partial C(f_\thetab(\xb'_m), \yb'_m)}{\partial f_{\thetab}(\xb'_m)} \nabla_{f_\thetab} f_{\thetab}(\xb'_m)\right]\bigg\rangle_{\mathcal{H}}\\
    &=\big\langle \nabla _{f_\thetab}\mathcal{L}_m(f_\thetab) ,  \nabla _{f_\thetab}\mathcal{L}_m(f_\thetab) \big\rangle_{\mathcal{H}}\\
    &=\Vert \nabla _{f_\thetab}\mathcal{L}_m(f_\thetab) \Vert^2_{\mathcal{H}},
\end{align*}
where $\langle \cdot, \cdot\rangle_{\mathcal{H}}$ is the inner product in Reproducing Kernel Hilbert Space (RKHS) $\mathcal{H}$. In the above equations, we use the definition of Neural Tangent Kernel (NTK), the property of inner product in RKHS, the definition of evaluation functional and its gradient in RKHS. 

Recall that 
\[
    \widetilde{\mathcal{E}}(\alpha, f_\thetab) = \mathbb{E}_{\mathcal{T}_m}\left[\mathcal{L}_m(f_\thetab) - \alpha \Vert \nabla _{f_\thetab}\mathcal{L}_m(f_\thetab) \Vert^2_{\mathcal{H}} \right]
\]
and
\[
    \mathcal{M}_k = \mathbb{E}_{\mathcal{T}_m}\left[\mathcal{L}_m(f_\thetab) - \sum_{i=0}^{k-1} \beta_i \right],
\]
where $\beta_i = \alpha \nabla_{\thetab_i} \mathcal{L}_m(f_{\thetab_i}) \nabla_{\thetab}\mathcal{L}_m(f_\thetab)^\intercal$ and $\thetab_0 = \thetab, \thetab_{i+1} = \thetab_{i} - \alpha \nabla_{\thetab_i} \mathcal{L}(f_{\thetab_i}, \mathcal{D}_m^{tr})$.
The result is straightforward now.

\end{proof}

\section{Proof of Theorem \ref{thm:delta_bound_nn}}
\new{The proof techniques we use are similar to some previous works such as \citep{arora19, pmlr-v97-allen-zhu19a}. We summaries some of the differences. Different from previous works that typically assume a neural network is Gaussian initialized, we do not have such an assumption as we are trying to learn a good meta-initialization in the meta-learning setting. Previous works try to investigate the behavior of models during training, while we focus on revealing the connection between different meta-learning algorithms. Previous work focuses on single-task regression/classification problems, while we focus on meta-learning problem. }

\textbf{Theorem \ref{thm:delta_bound_nn}} 
\textit{
  Let $f_{\thetab}$ be a fully-connected neural network with $L$ hidden layers and ReLU activation function,  $s_1, ..., s_{L+1}$ be the spectral norm of the weight matrices,  $s=\max_h s_h$, and $\alpha$ be the learning rate of gradient descent. If $\alpha \leq O(qr)$ with $q=\min(1/(Ls^L), L^{-1/(L+1)})$ and $r = \min(s^{-L}, s)$, then the following holds
    \[
    \vert \widetilde{\mathcal{E}}(k\alpha, f_\thetab) - \mathcal{M}_k \vert \leq O\Big(\dfrac{1}{L}\Big).
    \]
 %   Let $f_{\thetab}$ be a fully-connected neural network with $L$ hidden layers and ReLU activation function, and $\xb$ be a data sample. 
  %  Let $s_1, ..., s_{L+1}$ be the spectral norm of weight matrices, $s=\max_h s_h$. Let $\alpha$ be the learning rate of gradient descent.
  %  If $\alpha \leq O(qr)$, where $q=\min(1/(Ls^L), L^{-1/(L+1)}), r = \min(s^{-L}, s)$. Then     \[
  %  \vert \mathcal{M}_k - \widetilde{\mathcal{E}}(k\alpha, f_\thetab) \vert \leq O\Big(\dfrac{1}{L}\Big).
  %  \]
    }
\begin{proof}
We first prove the case of $k=2$, i.e. applying a two-step gradient descent adaptation in MAML. 

We need to prove the following theorem first.

\begin{theorem}\label{thm:perturbation}
    Let $f_{\thetab}$ be a fully-connected neural network with $L$ hidden layers, and $\xb$ be a data sample. Represent the neural network by $f_{\thetab}(\xb)= \sigma(\sigma(...\sigma(\xb W^1)...W^{L-1})W^L)W^{L+1}$, where $W^1, ..., W^{L+1}$ denote the weight matrices, and  $\sigma$ is the ReLU activation function. Let $s_1, ..., s_{L+1}$ be the spectral norm of weight matrices, and $s=\max_h s_h$. Let $\alpha$ be the learning rate of gradient descent, and $\f_{\tilde{\thetab}}(\xb)$  be the resulting value after one step of gradient descent, and $\Vert \cdot \Vert_\mathcal{F}$ be the Frobenius norm.
    
    % If $s \geq 1$ and $\Vert \triangle W^i\Vert_2 \leq O(s^{-L}/L)$ for all i, then:
    % If $s \geq 1$ and $\alpha \leq O(s^{-2L}/L)$ for all i, then:
    % \[
    % \Big \Vert \dfrac{\partial \f_{\tilde{\thetab}}(\xb)}{\partial \tilde{\thetab}} - \dfrac{\partial f_{\thetab}(\xb)}{\partial \thetab}\Big \Vert \leq O(\dfrac{1}{s}). 
    % \]
    % \indent If $s < 1$ and $\Vert \triangle W^i\Vert_2 \leq O(q)$ for all i, where $q=\min(1/(Ls^L), L^{-1/(L+1)})$, then:
    \indent If $\alpha \leq O(qs^{-L})$, where $q=\min(1/(Ls^L), L^{-1/(L+1)})$, then 
    \[
    \Big \Vert \dfrac{\partial \f_{\tilde{\thetab}}(\xb)}{\partial \tilde{\thetab}} - \dfrac{\partial f_{\thetab}(\xb)}{\partial \thetab}\Big \Vert_{\mathcal{F}} \leq O(\dfrac{1}{s\sqrt{L+1}}). 
    \]
\end{theorem}
\begin{remark}
Theorem \ref{thm:perturbation} states that for a neural network with $L$ hidden layers, if the learning rate of gradient descent is bounded, then the norm of derivative w.r.t all the parameters will not change too much, although there are $O(Lm^2)$ parameters, where $m$ denotes the maximum width of hidden layers. We use row vector instead of column vector for consistency, while it does not affect our results.
\end{remark}

For simplicity, we will write $g^h(\xb)$ as $g^h$. The bias terms in the neural network are introduced by adding an additional coordinate thus omitted in Theorem \ref{thm:perturbation}. Without loss of generality, we can assume $\Vert \xb \Vert \leq 1$, which can be done by data normalization in pre-processing.

Let $g^h(\xb)=\sigma(\sigma(...\sigma(\xb W^1)...W^{h-1})W^h)$ be the activation at $h^{th}$ hidden layer and $g^0(\xb)= \xb, g^{L+1} = f_{\thetab}(\xb)$. Define diagonal matrices $D^h$, where $D^h_{(i,i)} = \mathbf{1}\{g^{h-1}W^h \geq 0\}$ and
\[
    b^h=\left\{
                \begin{array}{ll}
                 \mathbf{I}_{d_y},\quad  &\text{if $h=L+1$} \\
                 b^{h+1} (W^{h+1})^\intercal D^h,\quad &\text{otherwise}\\
                \end{array}
              \right.
\]
where $\mathbf{I}_{d_y}$ is a $d_y \times d_y$ identity matrix.
We first prove the following Lemma.

\begin{lemma}\label{lemma:activation_perturbation}
Given a neural network as stated in Theorem \ref{thm:perturbation}, let $\Vert \cdot \Vert_2$ denote the spectral norm, $\triangle W^h = \tilde{W}^h - W^h$ denote some perturbation on weight matrices, $\tilde{g}^h(\xb)$  denote the resulting value after perturbation, and $\triangle g^h(\xb) = \tilde{g}^h(\xb) - g^h(\xb)$. 
If $s \geq 1$ and $\Vert \triangle W^h\Vert_2 \leq O(s^{-L}/L)$ for all h, then 
    \[
    \Vert \triangle g^h\Vert \leq O(\dfrac{1}{Ls^{L-h+1}});
    \]
\indent If $s < 1$ and $\Vert \triangle W^h\Vert_2 \leq O(q)$ for all h, where $q=\min(1/(Ls^L), L^{-1/(L+1)})$ and $r=\max(q,s)$, then 
    \[
    \Vert \triangle g^h\Vert \leq O(r^{h-1}q)=\left\{
                \begin{array}{ll}
                 O(\dfrac{1}{Ls^{L-h+1}}),\quad  &\text{if $1/(Ls^L) \leq L^{-1/(L+1)}$} \\
                 O(L^{-h/(L+1)}),\quad &\text{if $1/(Ls^L) > L^{-1/(L+1)}$}.\\
                \end{array}
              \right.
    \]
\end{lemma}
\begin{proof}
Proof of Lemma \ref{lemma:activation_perturbation} is based on induction. 

We first prove the case of $s\geq 1$. Note that $g^0 = \xb$, thus $\triangle g^0 = 0 \leq O(\dfrac{1}{Ls^{L-0+1}})$ always holds.

For $\triangle g^1$, we have 
\begin{align*}
    \Vert \triangle g^1 \Vert &= \Vert \sigma(\xb\tilde{W}^1) - \sigma(\xb W^1) \Vert \\
    &\leq \Vert \xb\tilde{W}^1  -  \xb W^1  \Vert, \quad \text{due to the property of ReLU activation}\\
    &\leq \Vert \xb \Vert \Vert \triangle W^1\Vert_2 \\
    &\leq O(\dfrac{1}{Ls^{L}}). 
\end{align*}
Thus,  the hypothesis holds for $\triangle g^1$.

Now, assume that the hypothesis holds for $\triangle g^h$, then we have 
\begin{align*}
    \Vert \triangle g^{h+1} \Vert &= \Vert \sigma(\tilde{g}^{h}\tilde{W}^{h+1} ) - \sigma(g^h W^{h+1}) \Vert \\
    &\leq \Vert \tilde{g}^{h}\tilde{W}^{h+1} - g^h W^{h+1} \Vert, \quad \text{due to the property of ReLU activation}\\
    &\leq \Vert \tilde{g}^{h}W^{h+1} + \tilde{g}^{h}\triangle W^{h+1}  - g^h W^{h+1} \Vert \\
    &\leq \Vert \triangle g^h \Vert \Vert W^{h+1}  \Vert_2 + \Vert \tilde{g}^{h} \Vert \Vert \triangle W^{h+1} \Vert_2\\
    &\leq O(s)\Vert \triangle g^h \Vert + \Vert g^{h} + \triangle g^h \Vert \Vert \triangle W^{h+1} \Vert_2\\
    &\leq O(s)\Vert \triangle g^h \Vert + O(s^h) \Vert \triangle W^{h+1} \Vert_2 + \Vert \triangle g^h \Vert \Vert \triangle W^{h+1} \Vert_2\\
    &\leq O(s) O(\dfrac{1}{Ls^{L-h+1}}) + O(s^h) O(\dfrac{1}{Ls^L}) + O(\dfrac{1}{Ls^{L-h+1}}) O(\dfrac{1}{Ls^L})\\
    &\leq O(\dfrac{1}{Ls^{L-h}}).
\end{align*}
The last three inequalities come from the fact that $g^h=\sigma(\sigma(...\sigma(\xb W^1)...W^{h-1})W^h) \leq O(s^h)$ and $s\geq 1$. Thus, we have proved the Lemma in the case $s \geq 1$.

Now, we prove the first part of the case of $s<1$, i.e. $\Vert \triangle g^h\Vert \leq O(r^{h-1}q)$. Because $\triangle g^0 = 0$, thus the hypothesis for $\triangle g^0$ always holds.

For $\triangle g^1$, we have 
\begin{align*}
    \Vert \triangle g^1 \Vert &= \Vert \sigma(\xb\tilde{W}^1) - \sigma(\xb W^1) \Vert \\
    &\leq \Vert \xb\tilde{W}^1  -  \xb W^1  \Vert\\
    &\leq \Vert \xb \Vert \Vert \triangle W^1\Vert_2 \\
    &\leq O(q). 
\end{align*}
Thus, the hypothesis holds for $\triangle g^1$.

Now, we assume that the hypothesis holds for $\triangle g^h$. Then, we have 
\begin{align*}
    \Vert \triangle g^{h+1} \Vert &= \Vert \sigma(\tilde{g}^{h}\tilde{W}^{h+1} ) - \sigma(g^h W^{h+1}) \Vert \\
    &\leq \Vert \tilde{g}^{h}\tilde{W}^{h+1} - g^h W^{h+1} \Vert \\
    &\leq \Vert \tilde{g}^{h}W^{h+1} + \tilde{g}^{h}\triangle W^{h+1}  - g^h W^{h+1} \Vert \\
    &\leq \Vert \triangle g^h \Vert \Vert W^{h+1}  \Vert_2 + \Vert \tilde{g}^{h} \Vert \Vert \triangle W^{h+1} \Vert_2\\
    &\leq O(s)\Vert \triangle g^h \Vert + \Vert g^{h} + \triangle g^h \Vert \Vert \triangle W^{h+1} \Vert_2\\
    &\leq O(s) O(r^{h-1}q) + O(s^h) q + q O(r^{h-1}q)\\
    &\leq O(r^{h}q). 
\end{align*}
The last inequality comes from the fact that $r=\max(q,s)$ and $s^h <s < 1$. 

Next we consider the second part of the case of $s<1$. 

If $1/(Ls^L) \leq L^{-1/(L+1)}$, we know that $q=1/(Ls^L)$ and 
\begin{align*}
    1/(Ls^L) &\leq L^{-1/(L+1)} \\
    L^{1/(L+1)}&\leq Ls^L \\
    L^{-L/(L+1)}&\leq s^L\\
    L^{-1} &\leq s^{L+1}\\
    L^{-1}s^{-L} &\leq s, 
\end{align*}
which means $q\leq s$, thus $r=s$. Then, we have 
\[
\Vert \triangle g^h \Vert = O(r^{h-1}q) = O(s^{h-1}q)=O(s^{h-1}L^{-1}s^{-L})=O(\dfrac{1}{Ls^{L-h+1}}).
\]
If $1/(Ls^L) > L^{-1/(L+1)}$, we know that $q=L^{-1/(L+1)}$ and $q > s$; then, $r=q$ and
\[
\Vert \triangle g^h \Vert = O(r^{h-1}q) = O(q^{h-1}q)=O(q^h)=O(L^{-h/(L+1)}).
\]
Thus, we can conclude that Lemma \ref{lemma:activation_perturbation} also holds for the case of $s<1$, which completes the proof.
\end{proof}

We now prove a similar Lemma for $\triangle b^h$. 
\begin{lemma}\label{lemma:back_perturbation}
Given a neural network as stated in Theorem \ref{thm:perturbation}, let $\Vert \cdot \Vert_2$ denote the spectral norm, $\triangle W^h = \tilde{W}^h - W^h$ denote some perturbation on weight matrices, $\tilde{b}^h$  denote the resulting value after perturbation, and $\triangle b^h = \tilde{b}^h - b^h$. 

If $s \geq 1$ and $\Vert \triangle W^h\Vert_2 \leq O(s^{-L}/L)$ for all h, then 
    \[
    \Vert \triangle b^h \Vert \leq  O(\dfrac{1}{Ls^h});
    \]
\indent If $s < 1$ and $\Vert \triangle W^h\Vert_2 \leq O(q)$ for all h, where $q=\min(1/(Ls^L), L^{-1/(L+1)})$, then 
    \begin{align*}
    \Vert \triangle b^h\Vert \leq \left\{
                \begin{array}{ll}
                 O(L^{-1}s^{-h}),\quad  &\text{if $1/(Ls^L) \leq L^{-1/(L+1)}$} \\
                 O(L^{(h-L-1)/(L+1)}),\quad &\text{if $1/(Ls^L) > L^{-1/(L+1)}$}.\\
                \end{array}
              \right.
    \end{align*}
\end{lemma}
\begin{proof}
Recall that
\[
    b^h=\left\{
                \begin{array}{ll}
                 \mathbf{I}_{d_y},\quad  &\text{if $h=L+1$} \\
                 b^{h+1} (W^{h+1})^\intercal D^h,\quad &\text{otherwise}\\
                \end{array}
              \right.
\]
where $\mathbf{I}_{d_y}$ is a $d_y \times d_y$ identity matrix and $D^h_{(i,i)} = \mathbf{1}\{g^{h-1}W^h \geq 0\}$.
It is easy to see that $\Vert b^h\Vert \leq O(s^{L-h+1})$, because $\Vert D^h\Vert_2 \leq 1$ and $\Vert W^h\Vert_2 \leq s$.

We first prove the case of $s\geq 1$. We know that $\triangle b^{L+1} = 0 \leq O(s^{-L-1}/L)$ always holds.

For $h\leq L$, we can re-write $b^h$ as
\[
    b^h = \mathbf{I}_{d_y}(W^{L+1})^\intercal D^L(W^{L})^\intercal D^{L-1}...(W^{h+1})^\intercal D^h.
\]
Then, we have
\begin{align}\label{eq:perturbation_back_composite}
        b^h (g^h)^\intercal = \mathbf{I}_{d_y}(W^{L+1})^\intercal D^L(W^{L})^\intercal D^{L-1}...(W^{h+1})^\intercal D^h (g^h)^\intercal.
\end{align}
Because of the fact that
\[
f_{\thetab} = g^{L+1} = \xb W^1 D^1 W^2 D^2...D^L W^{L+1} = g^h W^{h+1}D^{h+1}...D^L W^{L+1}
\]
and $g^h = g^h D^h$, $D^h = (D^h)^\intercal$. We can re-write \eqref{eq:perturbation_back_composite} as
\[
    b^h (g^h)^\intercal = f_{\thetab}^\intercal.
\]
Thus, 
\[
 \Vert \tilde{b}^h (\tilde{g}^h)^\intercal - b^h (g^h)^\intercal \Vert = \Vert \f_{\tilde{\thetab}} - f_{\thetab} \Vert = \Vert \triangle g^{L+1}\Vert \leq O(\dfrac{1}{L}) 
\]
by Lemma \ref{lemma:activation_perturbation}. Consequently, we have
\begin{align*}
     \Vert \tilde{b}^h (\tilde{g}^h)^\intercal - b^h (g^h)^\intercal \Vert &=  \Vert \triangle b^h (g^h)^\intercal + \triangle b^h \triangle (g^h)^\intercal  + \tilde{b}^h  \triangle (g^h)^\intercal \Vert \leq O(\dfrac{1}{L}). 
\end{align*}
Since $\Vert g^h \Vert \leq O(s^h)$,  we know that 
\[
\Vert \triangle b^h\Vert \leq O(\dfrac{1}{Ls^h}), \quad  \Vert \triangle b^h\Vert \leq O(s^{L-h+1})
\]
always hold. Since $L\geq 1, s\geq 1$,  we simply have $\Vert \triangle b^h\Vert \leq O(\dfrac{1}{Ls^h})$.

Now, we prove the case of $s<1$. Similarly, we have
\[
 \Vert \tilde{b}^h (\tilde{g}^h)^\intercal - b^h (g^h)^\intercal \Vert = \Vert \f_{\tilde{\thetab}} - f_{\thetab} \Vert = \Vert \triangle g^{L+1}\Vert  \leq O(\dfrac{1}{L}). 
\]
Similarly, we must have 
\[
\Vert \triangle b^h\Vert \leq O(\dfrac{1}{Ls^h}), \quad  \Vert \triangle b^h\Vert \leq O(\dfrac{1}{Lr^{h-1}q}), 
\]
where $q=\min(1/(Ls^L), L^{-1/(L+1)})$ and $r=\max(q,s)$ by Lemma \ref{lemma:activation_perturbation}. 

If $1/(Ls^L) \leq L^{-1/(L+1)}$, then $s^{L+1} \geq 1/L$. We thus have
\[
O(\dfrac{1}{Lr^{h-1}q}) = O(\dfrac{Ls^{L-h+1}}{L}) = O(\dfrac{s^{L+1}}{s^h})\geq O(\dfrac{1}{Ls^h}). 
\]
Hence, we get $\Vert \triangle b^h\Vert \leq O(\dfrac{1}{Ls^h})$.

If $1/(Ls^L) > L^{-1/(L+1)}$, then $s^{L+1} < 1/L$. We have
\[
O(\dfrac{1}{Lr^{h-1}q}) = O(L^{-1}\cdot L^{h/(L+1)}) \leq O(L^{-1}\cdot s^{-h}) = O(\dfrac{1}{Ls^h}).
\]
Thus, we get $\Vert \triangle b^h\Vert \leq O(L^{(h-L-1)/(L+1)})$.
\end{proof}
\begin{lemma}\label{lemma:gradient_weight_matrices_perturbation}
Given a neural network as stated in Theorem \ref{thm:perturbation}, let $\Vert \cdot \Vert_{\mathcal{F}} $ be the Frobenius norm, $W^1, ..., W^{L+1}$ be the weight matrices in the neural network, $\triangle W^h = \tilde{W}^h - W^h$ be the perturbation on weight matrices, $\thetab^h$ be the parameter vector containing all the elements in $W^h$,  $\triangle \thetab^h = \tilde{\thetab}^h - \thetab^h$ be the perturbation on parameter vectors, and $\f_{\tilde{\thetab}}(\xb)$  be the resulting value after perturbation.

If $s \geq 1$ and $\Vert \triangle W^h\Vert_2 \leq O(s^{-L}/L)$ for all h, for any weight matrices the following holds
    \[
    \Big\Vert \dfrac{\partial \f_{\tilde{\thetab}}(\xb)}{\partial \tilde{\thetab}^h} - \dfrac{\partial f_{\thetab}(\xb)}{\partial \thetab^h} \Big\Vert_{\mathcal{F}} \leq  O(\dfrac{1}{sL});
    \]
\indent If $s < 1$ and $\Vert \triangle W^h\Vert_2 \leq O(q)$ for all h, where $q=\min(1/(Ls^L), L^{-1/(L+1)})$, for any weight matrices the following holds
    \[
    \Big\Vert  \dfrac{\partial \f_{\tilde{\thetab}}(\xb)}{\partial \tilde{\thetab}^h} - \dfrac{\partial f_{\thetab}(\xb)}{\partial \thetab^h}   \Big\Vert_{\mathcal{F}} \leq  O(\dfrac{1}{sL}).
    \]
\end{lemma}
\begin{proof}
We first prove the case of $d_y = 1$, i.e. the output of neural network is 1-dimensional.

In this case, we know that 
\[
    \Big\Vert \dfrac{\partial \f_{\tilde{\thetab}}(\xb)}{\partial \tilde{\thetab}^h} - \dfrac{\partial f_{\thetab}(\xb)}{\partial \thetab^h} \Big\Vert_{\mathcal{F}} =     \Big\Vert \dfrac{\partial \f_{\tilde{\thetab}}(\xb)}{\partial \tilde{W}^h} - \dfrac{\partial f_{\thetab}(\xb)}{\partial W^h} \Big\Vert_{\mathcal{F}} = \Big\Vert \triangle \dfrac{\partial f_{\thetab}(\xb)}{\partial W^h}\Big\Vert_{\mathcal{F}}  
\]
and the derivative to $W^h$ is
\[
\dfrac{\partial f_{\thetab}(\xb)}{\partial W^h} = (b^h)^\intercal g^{h-1}.
\]

Then, we have
\begin{align*}
    \Big\Vert \triangle \dfrac{\partial f_{\thetab}(\xb)}{\partial W^h}\Big\Vert_{\mathcal{F}} &= \Vert (\tilde{b}^h)^\intercal \tilde{g}^{h-1} - (b^h)^\intercal g^{h-1} \Vert_{\mathcal{F}} \\
    &= \Vert  (\tilde{b}^h)^\intercal g^{h-1} - (b^h)^\intercal g^{h-1} + (\tilde{b}^h)^\intercal \triangle g^{h-1} \Vert_{\mathcal{F}}\\
    &\leq \Vert (\triangle b^h)^\intercal g^{h-1} \Vert_{\mathcal{F}} + \Vert (b^h + \triangle b^h)^\intercal \triangle g^{h-1} \Vert_{\mathcal{F}}.
\end{align*}

Recall the fact that $g^h \leq O(s^h)$ and $ b^h \leq O(s^{L+1-h})$.

When $s \geq 1$, from Lemma \ref{lemma:activation_perturbation} and Lemma \ref{lemma:back_perturbation} we know that 
\[
    \Vert \triangle g^h\Vert \leq O(\dfrac{1}{Ls^{L-h+1}}), \quad \Vert \triangle b^h \Vert \leq  O(\dfrac{1}{Ls^h}).
\]
Then, we have
\begin{align*}
    \Big\Vert \triangle \dfrac{\partial f_{\thetab}(\xb)}{\partial W^h}\Big\Vert_{\mathcal{F}} &\leq O(s^{h-1}) O(\dfrac{1}{Ls^h}) + O(s^{L+1-h}) O(\dfrac{1}{Ls^{L-h+2}}) 
    + O(\dfrac{1}{Ls^{L-h+2}}) O(\dfrac{1}{Ls^h})\\ 
    &\leq O(\dfrac{1}{sL}).
\end{align*}

When $s < 1$, from Lemma \ref{lemma:activation_perturbation} and Lemma \ref{lemma:back_perturbation} we know that 
\[
    \Vert \triangle g^h\Vert \leq \left\{
                \begin{array}{ll}
                 O(\dfrac{1}{Ls^{L-h+1}}),\quad  &\text{if $1/(Ls^L) \leq L^{-1/(L+1)}$} \\
                 O(L^{-h/(L+1)}),\quad &\text{if $1/(Ls^L) > L^{-1/(L+1)}$}\\
                \end{array}
              \right.
\]
and 
\[
    \Vert \triangle b^h\Vert \leq \left\{
                \begin{array}{ll}
                 O(L^{-1}s^{-h}),\quad  &\text{if $1/(Ls^L) \leq L^{-1/(L+1)}$} \\
                 O(L^{(h-L-1)/(L+1)}),\quad &\text{if $1/(Ls^L) > L^{-1/(L+1)}$}.\\
                \end{array}
              \right.
\]

If $1/(Ls^L) \leq L^{-1/(L+1)}$, we have
\[
    \Big\Vert \triangle \dfrac{\partial f_{\thetab}(\xb)}{\partial W^h}\Big\Vert_{\mathcal{F}} \leq O(s^{h-1}) O(\dfrac{1}{Ls^h}) + O(s^{L-h+1}) O(\dfrac{1}{Ls^{L-h+2}}) + O(\dfrac{1}{Ls^{L-h+2}})O(\dfrac{1}{Ls^h}).
\]

Since $1/(Ls^L) \leq L^{-1/(L+1)}$ implies  $L^{-1} \leq s^{L+1}$ (from proof of Lemma \ref{lemma:activation_perturbation}), we have 
\[ 
\dfrac{1}{Ls^{h}} \leq s^{L-h+1}.
\]

Then we can conclude that
\[
    \Big\Vert \triangle \dfrac{\partial f_{\thetab}(\xb)}{\partial W^h}\Big\Vert_{\mathcal{F}} \leq O(\dfrac{1}{sL}).
\]

If $1/(Ls^L) > L^{-1/(L+1)}$, we have
\begin{align*}
    \Big\Vert \triangle \dfrac{\partial f_{\thetab}(\xb)}{\partial W^h}\Big\Vert_{\mathcal{F}} &\leq O(s^{h-1}) O(L^{(h-L-1)/(L+1)}) + O(s^{L+1-h}) O(L^{-(h-1)/(L+1)}) \\ 
    &+ O(L^{-(h-1)/(L+1)})O(L^{(h-L-1)/(L+1)}).
\end{align*}

Since $1/(Ls^L) > L^{-1/(L+1)}$ implies $L^{-1} > s^{L+1}$ (from proof of Lemma \ref{lemma:activation_perturbation}), we have
\[ 
L^{(h-L-1)/(L+1)} > s^{L-h+1}, \quad \dfrac{1}{L^{(h-1)/(L+1)}} > s^{h-1}. 
\]

Then we have
\[
    \Big\Vert \triangle \dfrac{\partial f_{\thetab}(\xb)}{\partial W^h}\Big\Vert_{\mathcal{F}} \leq O(\dfrac{1}{L}) \leq O(\dfrac{1}{sL}), \text{ because }s<1.
\]

We have proved the Lemma for the case of $d_y=1$.

For the case of $d_y > 1$, we know that
\[
    \Big\Vert \dfrac{\partial \f_{\tilde{\thetab}}(\xb)}{\partial \tilde{\thetab}^h} - \dfrac{\partial f_{\thetab}(\xb)}{\partial \thetab^h} \Big\Vert^2_{\mathcal{F}} =  \sum_{i=1}^{d_y} \Big\Vert \dfrac{\partial \f_{\tilde{\thetab}, i}(\xb)}{\partial \tilde{\thetab}^h} - \dfrac{\partial f_{\thetab, i}(\xb)}{\partial \thetab^h} \Big\Vert^2_{\mathcal{F}}\leq O(\dfrac{d_y}{s^2L^2}), 
\]
where $f_{\thetab, i}(\xb)$ is the $i^{th}$ dimension of $f_{\thetab}(\xb)$. The last inequality directly comes from the 1-dimensional case.

Since $d_y$ is a constant, we ignore it.  Then, we have
\[
    \Big\Vert \dfrac{\partial \f_{\tilde{\thetab}}(\xb)}{\partial \tilde{\thetab}^h} - \dfrac{\partial f_{\thetab}(\xb)}{\partial \thetab^h} \Big\Vert_{\mathcal{F}} \leq O(\dfrac{1}{sL}),
\]
which completes the proof.
\end{proof}

Now we can prove Theorem \ref{thm:perturbation}, if $\tilde{W}^h$ is obtained by one step gradient descent starting from $W^h$, $\tilde{\thetab}$ is obtained by one step gradient descent starting from $\thetab$, and learning rate is $\alpha$. Then, for any weight matrix we have
\begin{align*}
    \Vert \triangle W^h\Vert_2 &= \Vert \alpha \nabla_{W^h} \mathcal{L}(\thetab)\Vert_2 \\
    & \leq \Vert \alpha \nabla_{W^h} \mathcal{L}(\thetab)\Vert_{\mathcal{F}} \\
    & = \Vert \alpha \nabla_{\thetab^h} \mathcal{L}(\thetab)\Vert_{\mathcal{F}} \\
    & = \alpha \Big\Vert  \dfrac{\sum_{i=1}^n}{n}\left[f_{\thetab}(\xb_i) - \yb_i\right] \dfrac{\partial f_{\thetab}(\xb_i)}{\partial \thetab^h}^\intercal \Big\Vert_{\mathcal{F}} \\
    & \leq \dfrac{\alpha \sum_{i=1}^n}{n} c_i \left[ \sum_j^{d_y} \Big\Vert \dfrac{\partial f_{\thetab, j}(\xb_i)}{\partial W^h}\Big\Vert^2_{\mathcal{F}}\right]^{1/2}\\
    % & \leq \dfrac{\alpha \sum_{i=1}^n}{n} c_i \Vert b^h\Vert\Vert g^{h-1}\Vert \\
    & \leq \dfrac{\alpha \sum_{i=1}^n}{n} c_i \sqrt{d_y} O(s^{L-h+1}) O(s^{h-1}) \\
    & \leq \alpha O(s^L),
\end{align*}
where $c_i = \Vert f_{\thetab}(\xb_i) - \yb_i \Vert$ are some constants. 

If $\alpha \leq O(s^{-2L}/L)$ when $s \geq 1$, then for any weight matrix we have 
\[
\Vert \triangle W^h\Vert_2 \leq \alpha O(s^L) \leq O(s^{-L}/L). 
\]

If $\alpha \leq O(qs^{-L})$ where $q=\min(1/(Ls^L), L^{-1/(L+1)})$ when $s < 1$, then for any weight matrix we have 
\[
\Vert \triangle W^h\Vert_2 \leq \alpha O(s^L) \leq O(q).
\]

By Lemma \ref{lemma:gradient_weight_matrices_perturbation}, we can conclude that
\[
    \Big\Vert \dfrac{\partial \f_{\tilde{\thetab}}(\xb)}{\partial \tilde{W}^h} - \dfrac{\partial f_{\thetab}(\xb)}{\partial W^h} \Big\Vert_{\mathcal{F}} \leq  O(\dfrac{1}{sL}).
\]
Then, we have
\begin{align*}
    \Big \Vert \dfrac{\partial \f_{\tilde{\thetab}}(\xb)}{\partial \tilde{\thetab}} - \dfrac{\partial f_{\thetab}(\xb)}{\partial \thetab} \Big \Vert_{\mathcal{F}} &= \left[\sum_{h=1}^{L+1} \Big\Vert \dfrac{\partial \f_{\tilde{\thetab}}(\xb)}{\partial \tilde{\thetab}^h} - \dfrac{\partial f_{\thetab}(\xb)}{\partial \thetab^h}\Big\Vert^2_{\mathcal{F}}\right]^{1/2}\\
    &\leq O(\dfrac{1}{s\sqrt{L+1}}).
\end{align*}
When $s \geq 1$, we know that
\[
s^{-L} \leq 1 \leq L^{L/(L+1)}.
\]
Then, we have 
\[
\dfrac{1}{Ls^L} \leq \dfrac{1}{L}\leq L^{-1/(L+1)}. 
\]
Thus, we know $1/(Ls^L) = \min(1/(Ls^L), L^{-1/(L+1)})$ when $s \geq 1$. 

For the case of $s\geq 1$, we can rewrite $\alpha \leq O(s^{-2L}/L) = O(qs^{-L})$, where $q=\min(1/(Ls^L), L^{-1/(L+1)})$, which completes the proof of Theorem \ref{thm:perturbation}. 

Now, we prove Theorem \ref{thm:delta_bound_nn} with $k=2$, i.e. two-step gradient descent adaptation. We know that 
\[
\beta_1 = \alpha \nabla_{\tilde{\thetab}} \mathcal{L}_m(f_{\tilde{\thetab}}) \nabla_{\thetab}\mathcal{L}_m(f_\thetab)^\intercal, \Vert \nabla _{f_\thetab}\mathcal{L}_m(f_\thetab) \Vert^2_{\mathcal{H}} = \Vert \nabla_{\thetab} \mathcal{L}_m(f_\thetab)\Vert ^2.
\]
Thus,  we have 
\begin{align*}
    &\big\vert \beta_1 - \alpha \Vert \nabla _{f_\thetab}\mathcal{L}_m(f_\thetab) \Vert^2_{\mathcal{H}} \big\vert\\
    =& \big\vert \alpha \nabla_{\tilde{\thetab}} \mathcal{L}_m(f_{\tilde{\thetab}}) \nabla_{\thetab}\mathcal{L}_m(f_\thetab)^\intercal - \alpha \nabla_{\thetab} \mathcal{L}_m(f_{\thetab}) \nabla_{\thetab}\mathcal{L}_m(f_\thetab)^\intercal\big\vert \\
    =&\alpha  \Vert \nabla_{\tilde{\thetab}} \mathcal{L}_m(f_{\tilde{\thetab}})  - \nabla_{\thetab} \mathcal{L}_m(f_{\thetab}) \Vert \Vert \nabla_{\thetab}\mathcal{L}_m(f_\thetab) \Vert\\
    =&\alpha \Big\Vert \mathbb{E}_{(\xb_m, \yb_m)}\Big\{\left[ f_{\tilde{\thetab}}(\xb_m) - \yb_m\right] \dfrac{\partial f_{\tilde{\thetab}}(\xb_m)}{\partial \tilde{\thetab}}^\intercal - \left[ f_{\thetab}(\xb_m) - \yb_m\right] \dfrac{\partial f_{\thetab}(\xb_m)}{\partial \thetab}^\intercal\Big\} \Big\Vert  \big\Vert \nabla_{\thetab}\mathcal{L}_m(f_\thetab) \big\Vert \\
    =&\alpha \Big\Vert \mathbb{E}_{(\xb_m, \yb_m)}\Big\{\left[ f_{\thetab}(\xb_m) - \yb_m + \triangle f_{\thetab}(\xb_m)\right] \left[\dfrac{\partial f_{\tilde{\thetab}}(\xb_m)}{\partial \tilde{\thetab}} + \triangle \dfrac{\partial f_{\thetab}(\xb_m)}{\partial \thetab} \right]^\intercal \\
    &- \left[ f_{\thetab}(\xb_m) - \yb_m\right] \dfrac{\partial f_{\thetab}(\xb_m)}{\partial \thetab}^\intercal\Big\} \Big\Vert  \big\Vert \nabla_{\thetab}\mathcal{L}_m(f_\thetab) \big\Vert \\
    =& \alpha \Big\Vert \mathbb{E}_{(\xb_m, \yb_m)}\Big\{\triangle f_{\thetab}(\xb_m) \left[\dfrac{\partial f_{\tilde{\thetab}}(\xb_m)}{\partial \tilde{\thetab}}  + \triangle \dfrac{\partial f_{\thetab}(\xb_m)}{\partial \thetab} \right]^\intercal \\
    & + \left[ f_{\thetab}(\xb_m) - \yb_m\right]  \triangle \dfrac{\partial f_{\thetab}(\xb_m)}{\partial \thetab} ^\intercal\Big\} \Big\Vert  \big\Vert \nabla_{\thetab}\mathcal{L}_m(f_\thetab) \big\Vert \\
    \leq & \alpha \left[O(\dfrac{1}{L})O(s^L\sqrt{L}) + O(\dfrac{1}{L})O(\dfrac{1}{s\sqrt{L}}) + O(\dfrac{1}{s\sqrt{L}})\right] \big\Vert \nabla_{\thetab}\mathcal{L}_m(f_\thetab) \big\Vert\\
    \leq & \alpha \left[O(\dfrac{s^L}{\sqrt{L}}) + O(\dfrac{1}{s\sqrt{L}})\right] \big\Vert \nabla_{\thetab}\mathcal{L}_m(f_\thetab) \big\Vert, \text{ because $L\geq 1$}\\
    \leq & \left[O(\dfrac{qrs^L}{\sqrt{L}}) + O(\dfrac{qr}{s\sqrt{L}})\right] \big\Vert \nabla_{\thetab}\mathcal{L}_m(f_\thetab) \big\Vert, \text{ where $q=\min(1/(Ls^L), L^{-1/(L+1)}), r = \min(s^{-L}, s)$}\\
    \leq & O(\dfrac{q}{\sqrt{L}}) \big\Vert \nabla_{\thetab}\mathcal{L}_m(f_\thetab) \big\Vert.
\end{align*}

In the case of $d_y=1$, we have
\[
\Big \Vert \dfrac{\partial f_{\thetab}(\xb)}{\partial W^h} \Big \Vert_{\mathcal{F}} = (b^h)^\intercal g^{h-1} \leq O(s^L), 
\]
which has already been shown in the proof of Lemma \ref{lemma:gradient_weight_matrices_perturbation}. Then,  we have 
\begin{align*}
    \big\Vert \nabla_{\thetab}\mathcal{L}_m(f_\thetab) \big\Vert = O\bigg(\sqrt{ \sum_{h=1}^{L+1} \Big \Vert \dfrac{\partial f_{\thetab}(\xb)}{\partial \thetab^h} \Big \Vert^2}\bigg)= O\bigg(\sqrt{ \sum_{h=1}^{L+1} \Big \Vert \dfrac{\partial f_{\thetab}(\xb)}{\partial W^h} \Big \Vert^2_{\mathcal{F}}}\bigg) \leq O( s^L \sqrt{L+1}).
\end{align*}
In the case of $d_y \geq 1$, the bound is simply scaled by a constant of $\sqrt{d_y}$.

Thus we have
\[
\big\vert \beta_1 - \alpha \Vert \nabla _{f_\thetab}\mathcal{L}_m(f_\thetab) \Vert^2_{\mathcal{H}} \big\vert \leq  O(\dfrac{q}{\sqrt{L}}) \big\Vert \nabla_{\thetab}\mathcal{L}_m(f_\thetab) \big\Vert \leq O(qs^L)\leq O(\dfrac{1}{L})
\]
because $q=\min(1/(Ls^L), L^{-1/(L+1)})$, which completes the proof for the case of $k=2$.

For the case of $k>2$, we only need to make sure that the bound on learning rate always holds. Fortunately, since $k$ is a finite constant, according to what we have already showed in the proof of previous lemmas, every step of gradient descent will not change the spectral norm of the weight matrix too much: $\Vert \triangle W^h\Vert_2 \leq O(s^{-L}/L)$ for all h if $s\geq 1$, and $\Vert \triangle W^h\Vert_2 \leq O(q)$ for all h if $s<1$, where $q=\min(1/(Ls^L), L^{-1/(L+1)})$. Thus, we may assume that  the bound on learning rate always holds during the adaptation. Using triangle inequality to generalize the results from $k=2$ to $k>2$, i.e. for all $1\leq i \leq k-1$, we have 
\[
 \big\vert \beta_i - \alpha \Vert \nabla _{f_\thetab}\mathcal{L}_m(f_\thetab) \Vert^2_{\mathcal{H}} \big\vert \leq O\Big(\dfrac{1}{L}\Big).
\]
Recall that 
\[
    \widetilde{\mathcal{E}}(\alpha, f_\thetab) = \mathbb{E}_{\mathcal{T}_m}\left[\mathcal{L}_m(f_\thetab) - \alpha \Vert \nabla _{f_\thetab}\mathcal{L}_m(f_\thetab) \Vert^2_{\mathcal{H}} \right]
\]
and
\[
    \mathcal{M}_k = \mathbb{E}_{\mathcal{T}_m}\left[\mathcal{L}_m(f_\thetab) - \sum_{i=0}^{k-1} \beta_i \right], 
\]
where $\beta_i = \alpha \nabla_{\thetab_i} \mathcal{L}_m(f_{\thetab_i}) \nabla_{\thetab}\mathcal{L}_m(f_\thetab)^\intercal$ and $\thetab_0 = \thetab, \thetab_{i+1} = \thetab_{i} - \alpha \nabla_{\thetab_i} \mathcal{L}(f_{\thetab_i}, \mathcal{D}_m^{tr})$.
The result is straightforward now.
\end{proof}

\section{Proof of Theorem \ref{thm:delta_bound_cnn}}
\textbf{Theorem \ref{thm:delta_bound_cnn}} 
\textit{
 Let $f_{\thetab}$ be a convolutional neural network with $L-l$ convolutional layers and $l$ fully-connected layers and with ReLU activation function, and $d_x$ be the input dimension. 
    Denote  by $W^h$   the parameter \textbf{vector} of the convolutional layer for $h\leq L-l$, and 
    %$W^{h}$ with $L-l+1<h\leq L+1$ be 
    the weight \textbf{matrices} of the fully connected layers for $L-l+1<h\leq L+1$. $\Vert \cdot \Vert_2$ means both the spectral norm of a matrix and the Euclidean norm of a vector.  
    Define 
    \[
    s_h = \left\{
                \begin{array}{ll}
                 \sqrt{d_x}\Vert W^h \Vert_2,\quad  &\text{if $h=1,...,L-l$} \\
                 \Vert W^h \Vert_2,\quad &\text{if $L-l+1< h \leq L+1$}\\
                \end{array}
              \right.
    \]
    and let $s=\max_h s_h$ and $\alpha$ be the learning rate of gradient descent. If $\alpha \leq O(qr)$ with $q=\min(1/(Ls^L), L^{-1/(L+1)})$ and $r = \min(s^{-L}, s)$, the following holds
    \[
    \vert \mathcal{M}_k - \widetilde{\mathcal{E}}(k\alpha, f_\thetab) \vert \leq O\Big(\dfrac{1}{L}\Big).
    \]
    }
  % Let $f_{\thetab}$ be a convolutional neural network with $L-l$ convolutional layers and $l$ fully-connected layers with ReLU activation function, $\xb$ be a data sample. 
  %  Let $W^h$ with $h\leq L-l$ be the parameter %\textbf{vector} of the convolutional layer, $W^{h}$ with $L-l+1<h\leq L+1$ be the weight \textbf{matrices} of the fully connected layers. $\Vert \cdot \Vert_2$ denotes both the spectral norm of matrices and Euclidean norm of vectors for convenience. 
   % Define 
%    \[
 %   s_h = \left\{
  %              \begin{array}{ll}
   %              \sqrt{d_x}\Vert W^h \Vert_2,\quad  %&\text{if $h=1,...,L-l$} \\
    %             \Vert W^h \Vert_2,\quad &\text{if $L-l+1< % h \leq L+1$}\\
     %           \end{array}
     %         \right.
    %\]
    %and $s=\max_h s_h$. Let $\alpha$ be the learning rate of gradient descent. 
    %If $\alpha \leq O(qr)$, where $q=\min(1/(Ls^L), %L^{-1/(L+1)}), r = \min(s^{-L}, s)$. Then     
    %\[
    %\vert \mathcal{M}_k - \widetilde{\mathcal{E}}(k\alpha, %f_\thetab) \vert \leq O\Big(\dfrac{1}{L}\Big).
    %\]
    
\begin{proof}
We prove Theorem \ref{thm:delta_bound_cnn} by first transforming the convolutional neural network into an equivalent fully connected neural network and then applying Theorem \ref{thm:delta_bound_nn}.

First of all, we assume that there are $c_h$ channels in $h^{th}$ convolutional layer's output $g^h(\xb), \text{ where } h=0,...,L-l$. For fully-connected layers, define $c_{L-l}=... = c_{L+1} =1$. We may represent the dimensionality of input data by $\xb \in R^{d_x c_0}$. Instead of using matrices, we represent the output of every convolutional layer by a $d_xc_h$ length vector $g^h=\left[g^h_{1},g^h_{2},...,g^h_{d_x} \right]$, where every $g^h_{i} = \left[g^h_{i, 1},g^h_{i, 2},...,g^h_{i, c_h}\right]$ is a $c_h$ length vector contains value of different channels at the same position.

We assume that for every element $g^h_{i,j}$ of $g^h_{i}$, its value is completely determined by elements of set $Q^{h-1}_i$, where $Q^{h-1}_i$ contains $kc_{h-1}$ elements with fixed positions in $g^{h-1}$ for a given $i$. In other words, every element of the output of a convolutional layer is determined by some elements with fixed positions from output of the previous layer. This is exactly how convolutional layer works in deep learning.

If we use $g^{h-1}_{Q^{h-1}_i}$ to represent the concatenation of $g^{h-1}_{a,b} \in Q^{h-1}_i$, then $g^{h-1}_{Q^{h-1}_i}$ is a $kc_{h-1}$ length vector, where $k$ is the kernel size. Then we have
\[
g^h_{i} = \sigma(g^{h-1}_{Q^{h-1}_i} U^h_{i})
\]
where $U^h_{i, j} \in R^{kc_{h-1} \times c_h}$ is a $kc_{h-1} \times c_h$ matrix.

For notation simplicity, one can define a matrix $U^h \in R^{d_xc_{h-1} \times d_xc_{h}}$, where every column of $U^h$ only has $kc_{h-1}$ non-zero elements, and it satisfies
\[
g^h = \sigma(g^{h-1} U^h)
\]
By the property of convolutional layer, we know the following facts:
\begin{itemize}
    \item One can represent $U^h$ by $U^h = \left[ V_1^h, V_2^h,...,V_{d_x}^h\right]$ where $V_{i}^h \in R^{d_xc_{h-1} \times c_h}$ is sub-matrix of $U^h$;
    \item Every $V_{i}^h$ contains the same set of elements as $W^h$, while these elements are located at different positions;
    \item Every $V_{i}^h$ can be obtained by any other $V_{j}^h$ by swapping rows;
\end{itemize}
Let's define $U^{L-l} = W^{L-l},...,U^{L+1} = W^{L+1}$ for the fully-connected layer and output layer. Then we can represent the neural network just as in Theorem \ref{thm:delta_bound_nn} by $f_{\thetab}(\xb)= \sigma(\sigma(...\sigma(\xb U^1)...U^{L-1})U^L)U^{L+1}$, and $\xb \in R^{d_xc_0}$.

Now let $t_h$ be the spectral norm of $U^h$, and $t = \max_h t_h$. By Theorem \ref{thm:delta_bound_nn}, we know that we want $\alpha \leq O(qr)$, where $q=\min(1/(Ls^L), L^{-1/(L+1)}), r = \min(s^{-L}, s)$.

Because every $V_{i}^h$ contains the same set of elements, we know that every $V_i^h$ has the same Frobenius norm. Because every $V_{i}^h$ can be obtained by any other $V_{j}^h$ by swapping rows, we know that every $V_i^h$ has the same rank.

We know that
\begin{align*}
   \dfrac{1}{\sqrt{r}}\Vert V^h_1 \Vert_{\mathcal{F}} \leq \Vert V^h_1 \Vert_2 \leq \Vert U^h \Vert_2 \leq  \Vert U^h \Vert_{\mathcal{F}} = \sqrt{d_x}\Vert V^h_1 \Vert_{\mathcal{F}} = \sqrt{d_x}\Vert W^h \Vert_2
\end{align*}
where $\Vert \cdot \Vert_{\mathcal{F}}$ denotes Frobenius norm, $r$ denotes the rank of $V_1^h$. The last equality holds because matrix $V^h_1 $  and vector $W^h$ have the same set of elements. 

Let's define
    \[
    s_h = \left\{
                \begin{array}{ll}
                 \sqrt{d_x}\Vert W^h \Vert_2,\quad  &\text{if $h=1,...,L-l$} \\
                 \Vert W^h \Vert_2,\quad &\text{if $L-l+1< h \leq L+1$}\\
                \end{array}
              \right.
    \]
and $s = \max_h s_h$.

From above we know that $t_h = \Theta(s_h)$, because $s_h/\sqrt{d_x r} \leq t_h \leq s_h$. So we also have $t = \Theta(s)$. Then the conclusion is straightforward.
\end{proof}

\section{Revision of Theorem \ref{thm:delta_bound_nn} and Theorem \ref{thm:delta_bound_cnn} in classification case}

We now show how to obtain similar results of Theorem \ref{thm:delta_bound_nn} and Theorem \ref{thm:delta_bound_cnn} in classification problem, where cross-entropy loss is used instead of squared loss. We need two more restrictions in the classification case:
\begin{enumerate}
    \item There exist matrix $A$ and $B$ such that $g^L A \leq \text{softmax}(g^LW^{L+1})\leq g^L B$ for all data points, where $\text{softmax}$ is the softmax operation at the last layer.
    \item For any data point $\xb$ whose belongs to $c^{th}$ class, there exists a constant $\epsilon>0$ such that $f_{\thetab,c}(\xb) \geq \epsilon$, i.e. the output of neural network has a lower bound on the true class position.
\end{enumerate}
The proof is actually similar to the proof in regression case. We briefly talk about the differences here.

Firstly, in the classification case, softmax function is used at the last layer. By the first restriction, we can get rid of softmax function by introducing new matrices, which further leads to bound of the learning rate as in regression case.  

Secondly, if the loss function is the cross-entropy loss, we have:
\[
\nabla_{\thetab} \mathcal{L}_{m}(f_{\thetab}) = \mathbb{E}_{(\xb_m, \yb_m)}\left [ \dfrac{1}{f_{\thetab, c_m}(\xb_m)}\dfrac{\partial f_{\thetab, c_m}(\xb_m)}{\partial \thetab}\right]
\]
where $c_m$ denotes the class of $\xb_m$, e.g. if $\xb_m$ belongs to the third class, then $c_m=3$. $f_{\thetab, c_m}(\xb_m)$ denotes the $c^{th}_m$ dimensional element of $f_{\thetab}(\xb_m)$. We want a lower bound of $f_{\thetab,c}(\xb)$ exists, so that the gradient $\nabla_{\thetab} \mathcal{L}_{m}(f_{\thetab}) $ can be further bounded.

Then we can prove similar theorems just follow the steps in regression case.

\section{Proof of Theorem \ref{thm:energy_functional_relation}}
\textbf{Theorem \ref{thm:energy_functional_relation}} 
\textit{
    Let $f_{\thetab}$ be a neural network with $L$ hidden layers, with each layer being either fully-connected or convolutional. Assume that $\|\mathcal{L}\|_{\infty} < \infty$. Then, $error(T)  = \vert \widetilde{\mathcal{E}}(T, f_{\thetab})  - \overline{\mathcal{E}}(T, f_{\thetab}) \vert$ is a non-decreasing function of $T$. Furthermore, for arbitrary $T>0$ we have:
    \[
     error(T) \leq O\big( T^{2L+3}\big).
    \]
    }
    
\begin{proof}
Recall that $\overline{\mathcal{E}}(t, f_{\thetab})$ is defined based on $f^t_{m, \thetab}$, which is the resulting function whose parameters evolve according to the gradient flow $\dfrac{\textup{d}\thetab^t_m}{\textup{d}t} = -\nabla_{\thetab^t_m} \mathcal{L}(f_{m, \thetab}^t, \mathcal{D}_m^{tr})$. 

We actually have the following 
%hold 
\citep{santambrogio2016}:
\[
\Vert \triangle \thetab \Vert = \Vert \thetab^0 - \thetab^t \Vert \leq O(\sqrt{t}).
\]
For simplicity and clearness, we use $\triangle$ to denote the change of any vectors and matrices.
Thus, we know that
\[
\Vert\triangle W^h \Vert_2 \leq \Vert\triangle W^h \Vert_{\mathcal{F}}  \leq \Vert \triangle \thetab \Vert  \leq O(\sqrt{t}).
\]
Just like the proofs of Lemma \ref{lemma:activation_perturbation}, Lemma \ref{lemma:back_perturbation} and Lemma \ref{lemma:gradient_weight_matrices_perturbation}, we show that
\[
\Vert \triangle g^h \Vert \leq O(t^{h/2}), \Vert \triangle b^h \Vert \leq O(t^{(L-h+1)/2}), \Big\Vert \triangle \dfrac{\partial f_\thetab(\xb)}{\partial \thetab } \Big \Vert_{\mathcal{F}} \leq O(t^{(L+1)/2}\sqrt{L+1} )
\]
by mathematical inductions; we skip the details here. Note that different from some previous theorem, here we focus on time t, and thus hide the effect of the spectral norms by treating them as constants.

Then, we have
\begin{align*}
     &\Big\Vert \triangle \Big(\nabla_{\thetab} \mathcal{L}_m(f_{\thetab})  \Big)\Big\Vert\\
    =& \Vert \nabla_{\thetab^t} \mathcal{L}_m(f_{m,\thetab}^t)  -  \nabla_{\thetab} \mathcal{L}_m(f_{\thetab}) \Vert \\
    =&\Big\Vert \mathbb{E}_{(\xb_m, \yb_m)}\Big\{\left[ f_{m, \thetab}^t(\xb_m) - \yb_m\right] \dfrac{\partial  f_{m, \thetab}^t(\xb_m)}{\partial {\thetab^t}}^\intercal - \left[ f_{\thetab}(\xb_m) - \yb_m\right] \dfrac{\partial f_{\thetab}(\xb_m)}{\partial \thetab}^\intercal\Big\} \Big\Vert \\
    =& \Big\Vert \mathbb{E}_{(\xb_m, \yb_m)}\Big\{\triangle f_{\thetab}(\xb_m) \left[\dfrac{\partial f_{m, \thetab}^t(\xb_m)(\xb_m)}{\partial \thetab^t}  + \triangle \dfrac{\partial f_{\thetab}(\xb_m)}{\partial \thetab} \right]^\intercal  + \left[ f_{\thetab}(\xb_m) - \yb_m\right]  \triangle \dfrac{\partial f_{\thetab}(\xb_m)}{\partial \thetab} ^\intercal\Big\} \Big\Vert\\
    \leq &  O(t^{L+1}\sqrt{L+1} ).
\end{align*}

Recall that:
\begin{align*}
    \overline{\mathcal{E}}(T, f_{\thetab}) &= \mathbb{E}_{\mathcal{T}_m}\left[\mathcal{L}_m(f_{m, \thetab}^T) \right]\\
    &=\mathbb{E}_{\mathcal{T}_m}\left[\mathcal{L}_m(f_{\thetab}) + \int_0 ^T \nabla_t \mathcal{L}_m(f_{m, \thetab}^t) \text{d} t\right] \\
    &= \mathbb{E}_{\mathcal{T}_m}\left[\mathcal{L}_m(f_{\thetab}) + \int_0 ^T \dfrac{\text{d}\thetab^t}{\text{d}t } \nabla_{\thetab^t}\mathcal{L}_m(f_{m, \thetab}^t)  \text{d} t\right] \\
    & = \mathbb{E}_{\mathcal{T}_m}\left[\mathcal{L}_m(f_{\thetab}) - \int_0 ^T \Big\Vert \nabla_{\thetab^t}\mathcal{L}_m(f_{m, \thetab}^t)  \Big\Vert ^2  \text{d} t\right] 
\end{align*}
and
\[
    \widetilde{\mathcal{E}}(T, f_\thetab) = \mathbb{E}_{\mathcal{T}_m}\left[\mathcal{L}_m(f_\thetab) - T \Vert \nabla _{f_\thetab}\mathcal{L}_m(f_\thetab) \Vert^2_{\mathcal{H}} \right] = \mathbb{E}_{\mathcal{T}_m}\left[\mathcal{L}_m(f_\thetab) - T \Vert \nabla _{\thetab}\mathcal{L}_m(f_\thetab) \Vert^2 \right].
\]

Because
\begin{align*}
    &\widetilde{\mathcal{E}}(T, f_{\thetab})  - \overline{\mathcal{E}}(T, f_{\thetab})\\
    =&  \int_0 ^T \Big\Vert \nabla_{\thetab^t}\mathcal{L}_m(f_{m, \thetab}^t)  \Big\Vert ^2 \text{d} t -T \Vert \nabla _{\thetab}\mathcal{L}_m(f_\thetab) \Vert^2  \\
    =&\int_0 ^T  \Big \Vert \nabla _{\thetab}\mathcal{L}_m(f_\thetab)  + \triangle \Big( \nabla _{\thetab}\mathcal{L}_m(f_\thetab) \Big)\Big\Vert^2 \text{d} t  -T \Vert \nabla _{\thetab}\mathcal{L}_m(f_\thetab) \Vert^2\\
    =& \int_0 ^T  2 \nabla _{\thetab}\mathcal{L}_m(f_\thetab)\triangle \Big( \nabla _{\thetab}\mathcal{L}_m(f_\thetab) \Big)^\intercal  + \Big \Vert\triangle \Big( \nabla _{\thetab}\mathcal{L}_m(f_\thetab) \Big)\Big\Vert^2 \text{d} t, 
\end{align*}
we have
\[
error(T) = \vert \widetilde{\mathcal{E}}(T, f_{\thetab})  - \overline{\mathcal{E}}(T, f_{\thetab}) \vert \leq O\Big(\dfrac{L+1}{2L+3} T^{2L+3}\Big) =  O\big(T^{2L+3}\big)
\]
by simple calculation.

On the other hand, observe that
\[
  \bar{ \mathcal{E} }( T, f_{\thetab} )
  = \mathbb{E}_{\mathcal{T}_m } \left[ \mathcal{L}_m( f_{\thetab} ) - \int_0^T \Big \Vert\nabla_{\thetab^t}
  \mathcal{L}_m( f_{m, \thetab}^t)\Big \Vert^2_{\mathcal{H}} dt \right],   
\]
\[
  \tilde{ \mathcal{E}}(T, f_{\thetab}) = \mathbb{E}_{ \mathcal{T}_m}[
  \mathcal{L}_m( f_{\thetab} ) - T \Big \Vert \nabla_{\thetab } \mathcal{L}_m(
  f_{\thetab} )  \Big \Vert_{ \mathcal{H} }^2  ].
\]
We let 
\[ G(\tau) = \int_0^{\tau} \Big \Vert  \nabla_{\thetab^t}
\mathcal{L}_m( f_{m, \thetab}^t) \Big \Vert^2 dt,  \]
and assume that $\nabla_{\thetab^t}
  \mathcal{L}_m( f_{m, \theta}^t)$ is continuous at $t=0$.  Then, we have $ G'( \tau) = \|  \nabla_{\thetab^t}
\mathcal{L}_m( f_{\thetab}^t) \|^2 $. 
\begin{align*}
\Big \Vert \overline{ \mathcal{E} }( T,f_{\thetab} ) - \tilde{ \mathcal{E}}(T,
 f_{\thetab}) \Big \Vert  =& \Big \Vert \mathbb{E}_{\mathcal{T}_m} \left[ \int_0^T \Big \Vert \nabla_{\thetab^t}
  \mathcal{L}_m( f_{m, \thetab}^t) \Big \Vert^2 dt -  T\Big \Vert \nabla_{\thetab} \mathcal{L}_m(
  f_{\thetab} )  \Big \Vert_{ \mathcal{H} }^2
 \right]  \Big \Vert \\
 =& \Big \Vert \mathbb{E}_{\mathcal{T}_m} (G(T) - T\cdot G'(0))  \Big \Vert,
\end{align*}
where $TG'(0) = G(0) + TG'(0)$ (note that $G(0) = 0$)  is a first order approximation to $G(T)$ at $\tau = 0$. When $T=1$, $G(T)-TG'(0)$ can be taken as a local truncation error (i.e., the error that occurs in one step of a numerical approximation). When $T$ increases, the difference is no better than the global truncation error (in $T$ steps):
\begin{align*} 
\Big \Vert G(T) - \sum_{i=0}^T (i-(i-1))G'(i) \Big \Vert 
= & \Big \Vert \sum_{i=0}^T \int_i^{i+1} \Big \Vert \nabla_{\thetab^t}
  \mathcal{L}_m( f_{m, \thetab}^t)\Big \Vert^2 - \Big \Vert \nabla_{\thetab^i}
  \mathcal{L}_m( f_{m, \thetab}^{t=i})\Big \Vert^2 dt \Big \Vert \\
\approx & \Big \Vert \sum_{i=0}^T \int_i^{i+1}  2 \cdot \triangle_{t}^i 
  \mathcal{L}_m( f_{m, \thetab}) \cdot \nabla \mathcal{L}_m( f_{m, \thetab}^t) dt\Big \Vert,  
\end{align*}
where $\triangle_{t}^i 
  \mathcal{L}_m( f_{m, \thetab}) = \nabla_{\thetab^t}
  \mathcal{L}_m( f_{m, \thetab}^t) -  \nabla_{\thetab^i}
  \mathcal{L}_m( f_{m, \thetab}^t) $ as shown previously
, $i$ is the $i$-th time step, and $G'(i)$ is the gradient of $G$ at time step $i$.
Now we can see that  $\Big \Vert \overline{ \mathcal{E} }( T,f_{\thetab} ) - \tilde{ \mathcal{E}}(T,
 f_{\thetab}) \Big \Vert$ highly relates to the difference between $\nabla_{\thetab^t}
  \mathcal{L}_m( f_{m, \thetab}^t)$ at different time steps (i.e. $\triangle_{t}^i 
  \mathcal{L}_m( f_{m, \thetab})$), $\nabla_{\thetab^t} \mathcal{L}_m(f_{m,\thetab}^t)$ and $T$ . The first two terms  relate to how flat or sharp the hyperplane of $\mathcal{L}_m(f_{m, \thetab})$ is near $t=0$. We can wrap it as a constant $C_0(\mathcal{L},t=0)$. Then, the error is at least $C_0(\mathcal{L},t=0)\cdot O(T)$. For the hyperplane smooth enough, we can further get a first order approximation of $\triangle_{t}^i
  \mathcal{L}_m( f_{m, \thetab}) $ and yield $C(\mathcal{L},t=0)O(T^2)$, where $C(\mathcal{L},t=0)$ can be analogized as the second order derivative of $\mathcal{L}$.
\end{proof}  
  
\section{Some Experimental Details}

% \subsection{Some Experimental Details of Meta-RKHS-\ROM{1}}
% In the implementation of Meta-RKHS-\ROM{1}, $\alpha$ was set to be 0 at the beginning, and it will be a value selected from $\left[0.00001, 0.1\right]$ after half of the iterations have been finished (e.g. for a training procedure with 100,000 iterations, $\alpha$ is set to be 0 for the first 50,000 iterations, and is non-zero for the last 50,000 iterations). This is because we want to ensure that at the beginning the model is driven towards the direction that is close to all optima, after the model is near this solution, we can start searching for a point near this solution with large functional gradient norm. 

\subsection{Implementation of Classification for Meta-RKHS-\ROM{2}}
As we mentioned earlier, our proposed energy functional with closed form adaptation can not be directly applied to classification problem. We handle this challenge following \cite{arora19}. For a $d_y$ class classification problem, every data $\xb$ is associated with a $R^{d_y}$ one-hot vector $\yb$ as its label. For $C$ classes classification problem, its encoding is $C$ dimensional vector and we use $-1/C$ and $(C-1)/C$ as its correct and incorrect entries encoding. In the prediction, $Y^{tr}$ is replaced by the encoding of training data. $f_{\thetab}(\xb)$ is replaced by $f_{\thetab}(\xb)^\intercal [1, ..., 1] \in R^{n \times d_y}$ for dimension consistency. During the testing time, we compute the encoding of the test data point, and choose the position with largest value as its predicted class.

\section{Extra Experimental Results}
\subsection{Comparison with RBF kernel}
One interesting question is, without introducing extra model components or networks, what will the results of other kernel be? We provide the results of using RBF (Gaussian) kernel here: $42.1 \pm 1.9$ (5-way 1-shot) and $54.9 \pm 1.1$ (5-way 5-shot)  on Mini-ImageNet, $32.4 \pm 2.0$ (5-way 1-shot) and $38.2 \pm 0.9$ (5-way 5-shot) on FC-100, which are worse than the NTK based Meta-RKHS-\ROM{2}, showing the superiority of using NTK.

\subsection{More Results on out-of-distribution Generalization}
We provide some more results on out-of-distribution generalization experiments here. From the results we can find that the proposed methods is more robust and can generalize to different datasets better.
% \begin{table}[htb!]
%     \caption{Meta testing on different out-of-distribution datasets with model trained on Mini-ImageNet.}
%     % \label{fig:OOD-Mini-ImageNet}
%     \begin{center}
%     \begin{sc}
%     \begin{adjustbox}{scale=0.9}
%         \begin{tabular}{lrrrrrrrrrr}
%             \toprule
%         & \multicolumn{2}{c}{5 way 1 shot}& &\multicolumn{2}{c}{5 way 5 shot}&\\
%     %   \cmidrule{2-3}  \cmidrule{5-7}\\
%             Algorithm & CUB & VGG Flower & & CUB & VGG Flower &\\
%             \midrule
%          MAML &$34.23\pm 1.52 \%$ &$52.98\pm 1.76 \%$&& $52.36 \pm 0.94 \%$&  $67.52 \pm 1.30 \%$ &\\
%          FOMAML& $35.32\pm 1.69 \%$ & $53.86\pm 1.64 \%$ && $52.02\pm 0.71 \%$ & $68.83\pm 1.16 \%$\\
%          Reptile& $35.61\pm 1.38 \%$ & $53.57\pm 1.58 \%$ &&$51.93\pm 0.89 \%$&$71.62\pm 1.25 \%$& \\
%          iMAML& $40.55 \pm 1.61 \%$ & $54.97 \pm 1.80 \%$ & &$49.31 \pm 1.03 \%$ &$64.67 \pm 1.41 \%$\\
%          Meta-RKHS-\ROM{1} &$37.85\pm 1.26 \%$&  $54.79\pm 1.61 \%$ & &$54.19\pm 0.73 \%$ & $72.76\pm 1.08 \%$\\
%          Meta-RKHS-\ROM{2}&$\mathbf{45.36\pm 0.87 \%}$&$\mathbf{60.80\pm 1.02 \%}$ && $\mathbf{65.21 \pm 0.64 \%}$& $\mathbf{78.25 \pm 0.49 \%}$ & \\
%         \bottomrule
%         \end{tabular}
%     \end{adjustbox}
%     \end{sc}
%     \end{center}
% \end{table}

\begin{wraptable}{R}{0.9\linewidth}
    \caption{Meta testing on different out-of-distribution datasets with model trained on FC-100.}
    \label{fig:OODFC-100}
    \begin{center}
        \begin{sc}
        \begin{adjustbox}{scale=0.9}
        \begin{tabular}{lrrrrrrrrrr}
        \toprule
        & \multicolumn{2}{c}{5 way 1 shot}& &\multicolumn{2}{c}{5 way 5 shot}&\\
    %   \cmidrule{2-3}  \cmidrule{5-7}\\
        Algorithm & CUB & VGG Flower & & CUB & VGG Flower & \\
        \midrule
         MAML &$31.58 \pm 1.89 \%$ &  $50.82 \pm 1.94 \%$ && $41.72 \pm 1.29 \%$ &$65.19\pm 1.36 \%$&\\
         FOMAML& $32.34 \pm 1.57 \%$ & $49.90 \pm 1.78 \%$& &  $41.96 \pm 1.53\%$& $66.87 \pm 1.45\%$\\
         Reptile&$33.56 \pm 1.40 \%$ & $46.77 \pm 1.81 \%$ & & $42.79 \pm 1.38 \%$&$67.97 \pm 0.71 \%$ &\\
         iMAML&$32.49 \pm 1.52 \%$ &$49.96 \pm 1.98 \%$ & &$38.92 \pm 1.62 \%$ &$59.80 \pm 1.82 \%$&\\
         Bayesian TAML(SOTA) &$31.82\pm0.49\%$&$49.58\pm0.55\%$ & &$43.97\pm 0.57\%$& $67.36\pm 0.53\%$&\\
         Meta-RKHS-\ROM{1} &$34.12 \pm 1.34 \%$ & $48.81 \pm 1.89 \%$ & & $43.31 \pm 1.43\%$ &$69.02 \pm 0.62 \%$ &\\
         Meta-RKHS-\ROM{2}&$\mathbf{36.35 \pm 1.07 \%}$ & $\mathbf{59.75 \pm 1.23 \%}$ & &$\mathbf{49.92 \pm 0.68} \%$ & $\mathbf{76.32 \pm 0.58 \%}$ &&\\
         \bottomrule
        \end{tabular}
    \end{adjustbox}
    \end{sc}
    \end{center}
\end{wraptable}

\clearpage
\subsection{More Results on Adversarial Attack}
We now show some more extra results on adversarial attack in the following figures. Consistent to the results in main text, we can find  that our proposed methods are more robust to adversarial attacks.

\begin{figure}[htb!]
     \centering
     \begin{minipage}{0.45\textwidth}
         \centering
         \includegraphics[width=\textwidth]{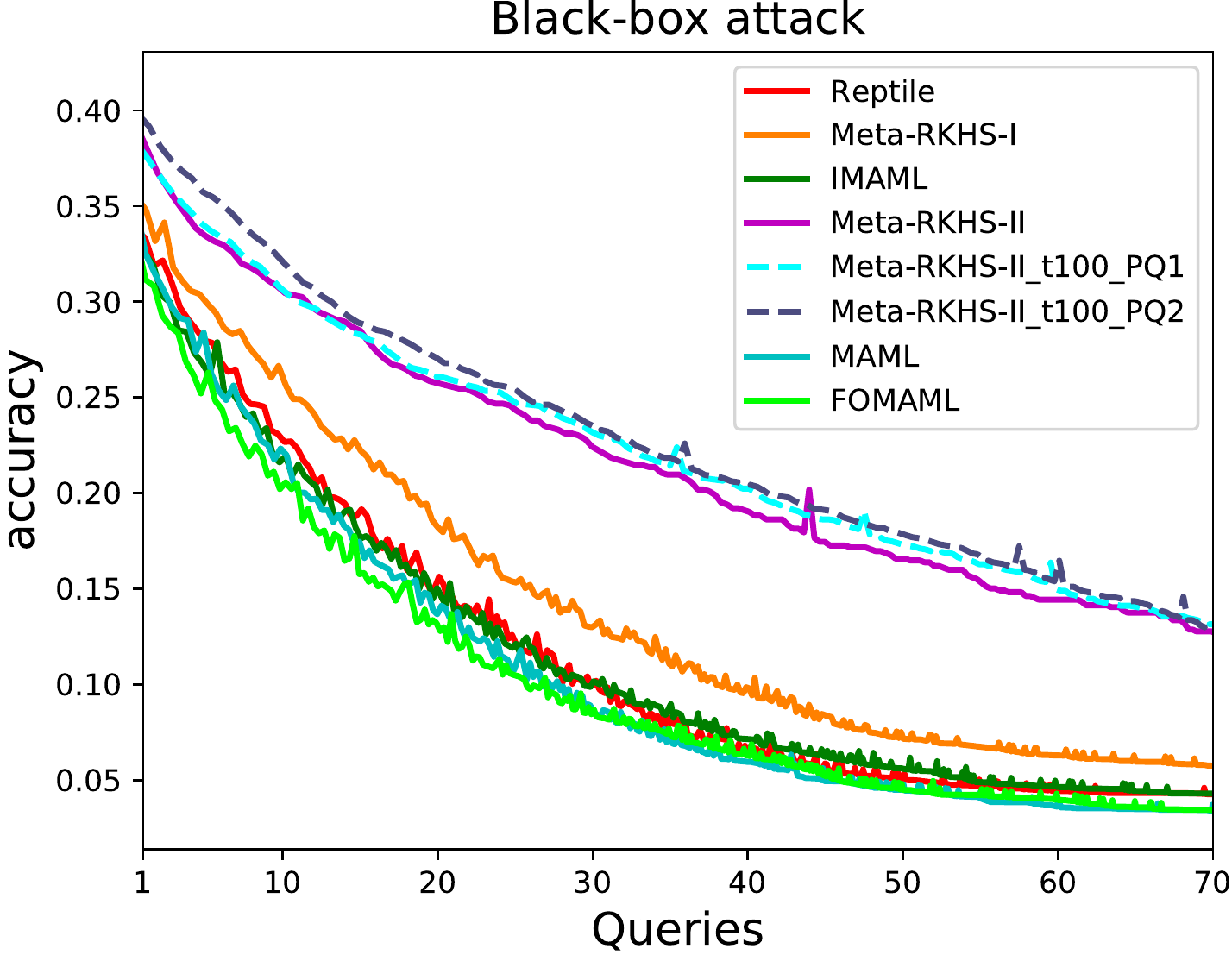}
     \end{minipage}
     \hfill
     \begin{minipage}{0.54\textwidth}
         \centering
         \includegraphics[width=\textwidth]{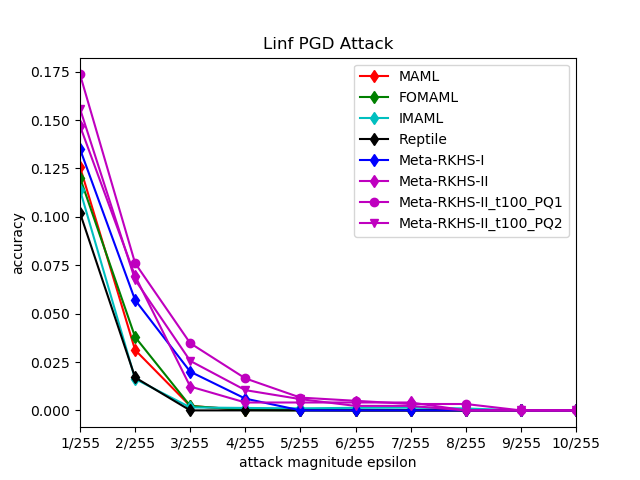}
     \end{minipage}
     \caption{FC-100 5-way 5-shot Black-box attacks (left) and 5-way 1-shot PGD $\ell_{\infty}$ norm attack (right).}
     \label{blackbox:cifar100}
\end{figure}

\subsection{Impact of Gradient Norm in Meta-RKHS-\ROM{1}}
 In this experiment, we compare between our proposed Meta-RKHS-\ROM{1} and Reptile. We evaluate the trained models with different adaptation steps in testing-time. The comparison is shown in Figure \ref{fig:reptile_vs_mgfl}. As we can see, our Meta-RKHS-\ROM{1} always gets better results than Reptile, which supports our idea that the learned function should be close to task-specific optimal and have large functional gradient norm. These two conditions together lead to the ability of fast adaptation.
\begin{figure}[ht!]
    \centering
        \subfigure[Mini-ImageNet, 5 Way 1 Shots]{\includegraphics[width=0.22\linewidth]{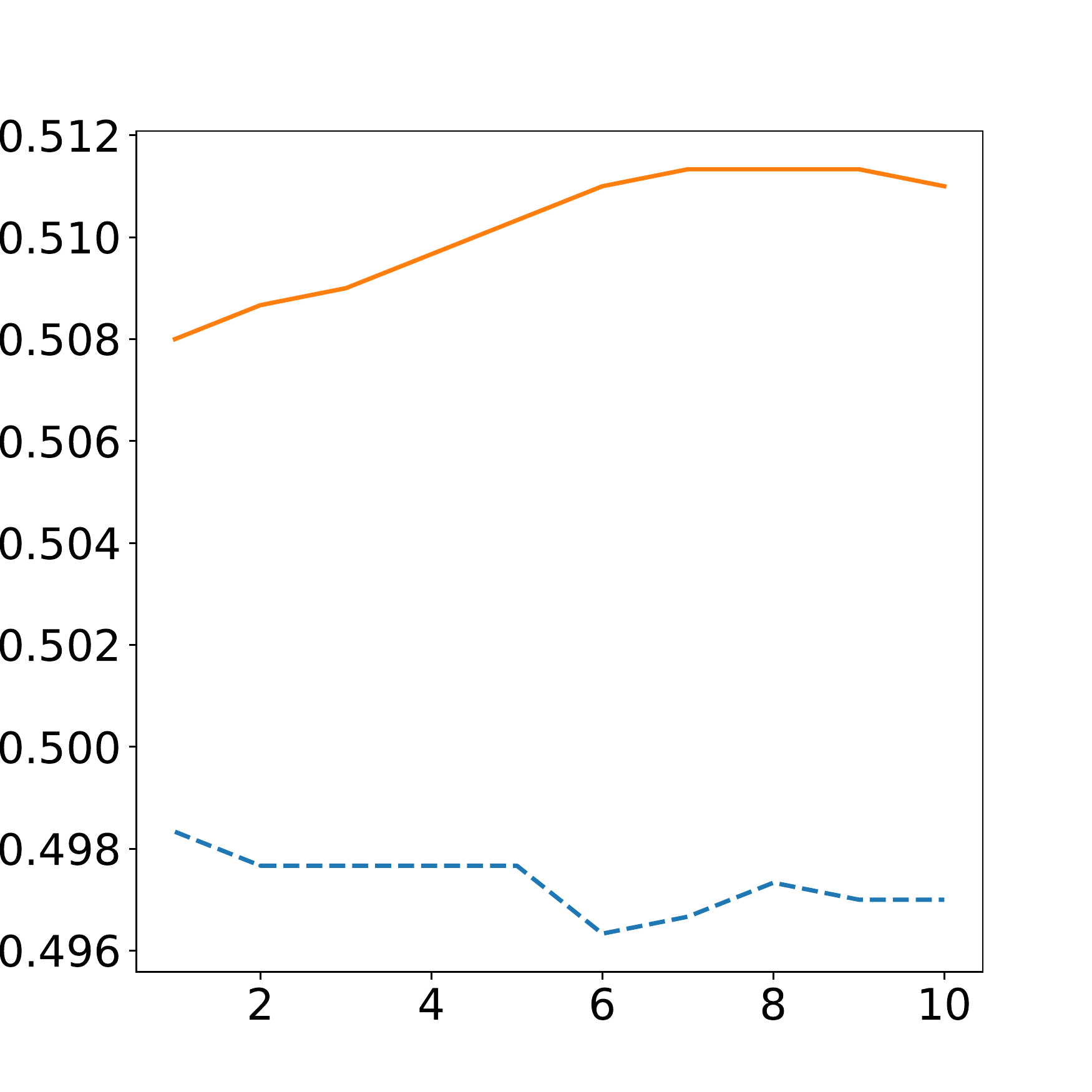}}
        \subfigure[Mini-ImageNet, 5 Way 5 Shots]{\includegraphics[width=0.22\linewidth]{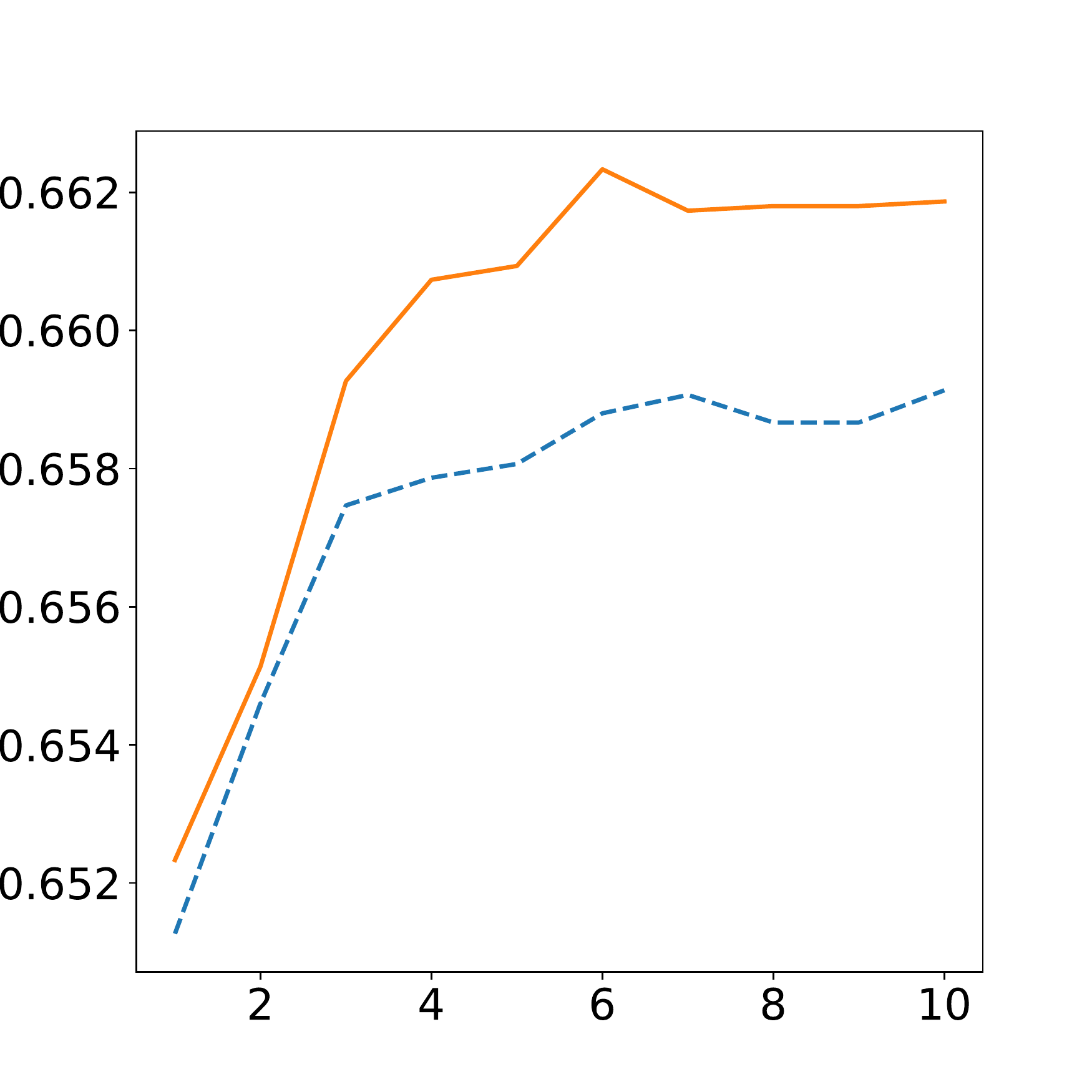}}
        \subfigure[FC-100, 5 Way 1 Shots]{\includegraphics[width=0.22\linewidth]{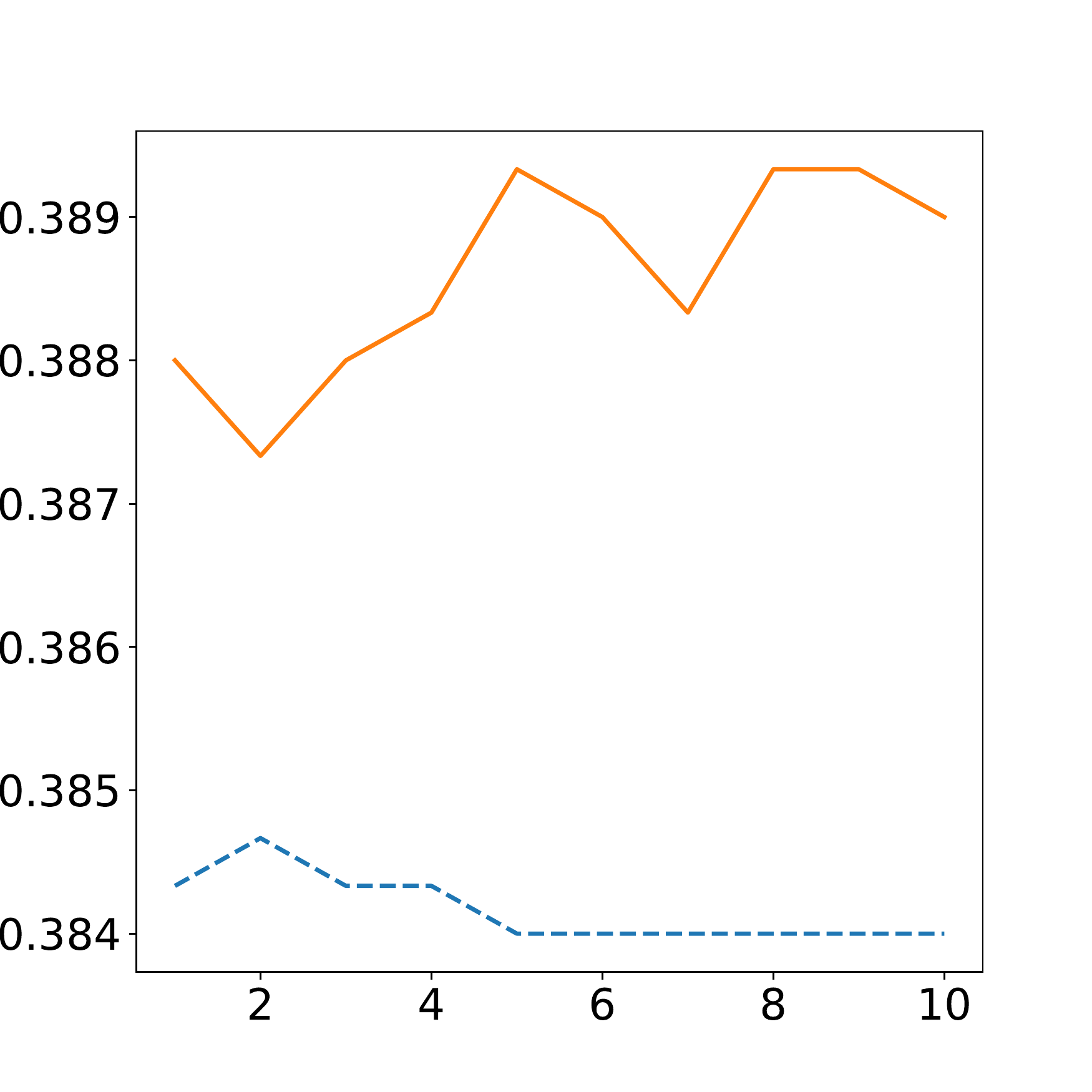}}
        \subfigure[FC-100, 5 Way 5 Shots]{\includegraphics[width=0.22\linewidth]{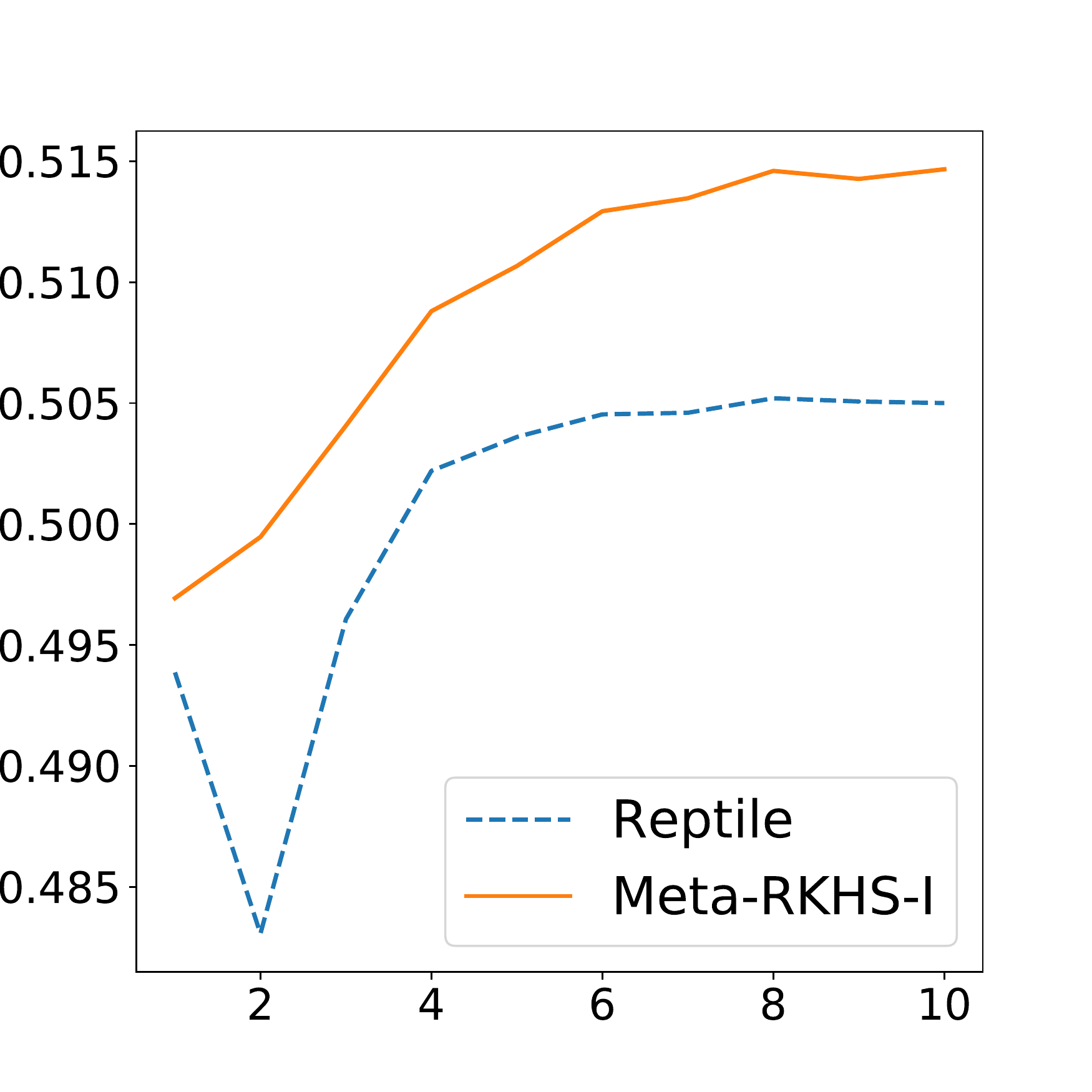}}
    \caption{Reptile (dashed) vs. Meta-RKHS-\ROM{1} (solid) with different testing adaptation steps (x-axis).} 
    \label{fig:reptile_vs_mgfl}
\end{figure}

\subsection{Impact of network architecture for different meta-learning models}

In this section, we compare different meta-learning models with feature channels of 100 and 200 of the CNN network structure with 4 or 5 CNN layers respectively.

\begin{wraptable}{R}{0.9\linewidth}
    \caption{Few-shot classification results on Mini-ImageNet with different number of feature channels of 4 convolution layers.}
    \vspace{-0.15in}
    \label{tab:ablation_architecture}
    \begin{center}
    \begin{sc}
    \begin{adjustbox}{scale=0.9}
        \begin{tabular}{lrrrrrr}
            \toprule
            &\multicolumn{2}{c}{100} & \multicolumn{2}{c}{200}\\
            Algorithm& 5 Way 1 Shot & 5 Way 5 Shots&5 Way 1 Shot & 5 Way 5 Shots \\
            \midrule
             MAML& $49.50 \pm 1.58 \%$ & $64.31 \pm 1.07 \%$ & $48.91 \pm 1.69 \%$ & $63.96 \pm 0.82 \%$&\\
             FOMAML& $48.69 \pm 1.62 \%$ & $63.73 \pm 0.76 \%$& $48.55 \pm 1.86 \%$ & $63.18 \pm 0.96 \%$&\\
             iMAML& $49.30 \pm 1.94 \%$ & $62.89 \pm 0.95 \% $& $48.23 \pm 1.58 \%$ & $62.25 \pm 0.83 \%$&\\
            %  Ours &$48.23 \pm 1.96 \%$ & $65.35 \pm 1.02 \%$ & $38.92 \pm 1.65 \%$ & $50.22 \pm 1.82 \%$\\
            %  Ours-GR &$48.56 \pm 1.89 \%$ & & $\mathbf{39.23 \pm 1.72 \%}$ \\
             Reptile & $50.20 \pm 1.69 \%$ & $64.12 \pm 0.92 \%$ & $48.72 \pm 1.97 \%$ & $63.67 \pm 0.79 \%$&\\
             \midrule
             Meta-RKHS-\ROM{1}& $51.23 \pm 1.79 \%$ & $66.69 \pm 0.73 \%$ & $\mathbf{51.54 \pm 1.64 \%}$ & $\mathbf{65.92 \pm 0.92\%}$&\\
             Meta-RKHS-\ROM{2}& $\mathbf{51.37 \pm 2.31\%}$& $\mathbf{66.97 \pm 0.98 \%}$& $50.96 \pm 2.15\%$& $65.21\pm 0.87\%$\\
             \bottomrule
        \end{tabular}
    \end{adjustbox}
    \end{sc}
    \end{center}
    \vspace{-0.1in}
\end{wraptable}

\begin{wraptable}{R}{0.9\linewidth}
    \caption{Few-shot classification results on Mini-ImageNet with different number of feature channels of 5 convolution layers.}
    \vspace{-0.15in}
    \label{tab:ablation_architecture}
    \begin{center}
    \begin{sc}
    \begin{adjustbox}{scale=0.9}
        \begin{tabular}{lrrrrrr}
            \toprule
            &\multicolumn{2}{c}{100} & \multicolumn{2}{c}{200}\\
            Algorithm& 5 Way 1 Shot & 5 Way 5 Shots&5 Way 1 Shot & 5 Way 5 Shots \\
            \midrule
             MAML& $49.87 \pm 1.65 \%$ & $65.78 \pm 1.18 \%$ & $48.62 \pm 1.82 \%$ & $63.25 \pm 0.75 \%$&\\
             FOMAML& $48.93 \pm 1.71 \%$ & $64.37 \pm 0.80 \%$& $48.27 \pm 1.74 \%$ & $62.95 \pm 0.83 \%$&\\
             iMAML& $48.03 \pm 1.76 \%$ & $62.15 \pm 0.83\% $& $47.52 \pm 1.73 \%$ & $61.77 \pm 0.89 \%$&\\
            %  Ours &$48.23 \pm 1.96 \%$ & $65.35 \pm 1.02 \%$ & $38.92 \pm 1.65 \%$ & $50.22 \pm 1.82 \%$\\
            %  Ours-GR &$48.56 \pm 1.89 \%$ & & $\mathbf{39.23 \pm 1.72 \%}$ \\
             Reptile & $50.62 \pm 1.83 \%$ & $64.53 \pm 0.97 \%$ & $49.33 \pm 1.89 \%$ & $63.26 \pm 0.70 \%$&\\
             \midrule
             Meta-RKHS-\ROM{1}& $\mathbf{52.45 \pm 1.88 \%}$ & $66.07 \pm 0.69 \%$ & $\mathbf{51.37 \pm 1.92 \%}$ & $\mathbf{65.39 \pm 0.98\%}$&\\
             Meta-RKHS-\ROM{2}& $50.92 \pm 2.16\%$& $\mathbf{66.45 \pm 0.91 \%}$& $50.43 \pm 2.42\%$& $64.17\pm 1.06\%$\\
             \bottomrule
        \end{tabular}
    \end{adjustbox}
    \end{sc}
    \end{center}
    \vspace{-0.1in}
\end{wraptable}

\end{document}